\documentclass[letterpaper, 11 pt]{article} 
\usepackage{amsmath,amsfonts,amssymb,color}
\usepackage{mathtools}
\usepackage{dsfont}
\usepackage[dvipsnames]{xcolor}
\definecolor{mygreen}{rgb}{0.0, 0.5, 0.0}
\definecolor{winered}{rgb}{0.8,0,0}
\definecolor{myblue}{rgb}{0,0,0.8}
\usepackage{tikz}
\usepackage{amsthm}
\usepackage{empheq}
\usepackage{microtype}
\usepackage{graphicx}
\usepackage{subfigure}
\usepackage{booktabs} 
\usepackage{bibentry}
\usepackage{algpseudocode}
\usepackage{algorithm}
\usepackage{colortbl}
\usepackage{subcaption}
\usepackage{graphicx}
\usepackage{bbding,pifont}

\newtheorem{theorem}{Theorem}
\newtheorem{lemma}{Lemma}
\newtheorem{proposition}{Proposition}
\newtheorem{corollary}{Corollary}

\newtheorem{assumption}{Assumption}

\newcommand{\mc}{\mathcal}

\allowdisplaybreaks[4]
\newcommand{\E}{\mathbb{E}}

\renewcommand{\P}{\mathbb{P}}

\newcommand{\R}{\mathbb{R}}

\newcommand{\cA}{\mathcal{A}}

\newcommand{\cF}{\mathcal{F}}

\newcommand{\cM}{\mathcal{M}}
\newcommand{\cN}{\mathcal{N}}
\newcommand{\cO}{\mathcal{O}}
\newcommand{\cP}{\mathcal{P}}

\newcommand{\cS}{\mathcal{S}}

\newcommand{\cX}{\mathcal{X}}

\newcommand{\iprod}[2]{\left\langle #1,\, #2\right\rangle}
\newcommand{\bracket}[1]{\left[ #1\right]}
\newcommand{\norm}[1]{\left\| #1\right\|_2}
\newcommand{\bigo}[1]{\cO\left( #1\right)}
\newcommand{\paren}[1]{\left( #1\right)}
\usepackage{algorithm}
\usepackage{algpseudocode}
\usepackage{epsfig}
\usepackage{array}
\usepackage{cite}
\usepackage{multirow}
\usepackage{epstopdf}
\usepackage{tikz}
\usepackage{relsize}
\usetikzlibrary{shapes,arrows}
\usepackage[hidelinks]{hyperref}
\hypersetup{
    colorlinks=true,
    linkcolor={winered},
    citecolor={myblue}
}
\usepackage{cite}
\usepackage[margin=1in]{geometry}
\title{Achieving Tighter Finite-Time Rates for Heterogeneous Federated Stochastic Approximation under Markovian Sampling}
\author{Feng Zhu, Robert W. Heath Jr., and Aritra Mitra
\thanks{F. Zhu and A. Mitra are with the Dept. of Electrical and Computer Engineering, North Carolina State University. Email: {\tt \{fzhu5, amitra2\}@ncsu.edu}. Robert W. Heath Jr. is with the Dept. of Electrical and Computer Engineering at the University of California, San Diego, USA. Email: {\tt rwheathjr@ucsd.edu}. This material is based upon work supported in part by the National Science Foundation under Grant No. NSF-CCF-2225555.}}
\date{}
\begin{document}
\maketitle
\thispagestyle{empty}
\pagestyle{empty}
\begin{abstract}
    Motivated by collaborative reinforcement learning (RL) and optimization with time-correlated data, we study a generic federated stochastic approximation problem involving $M$ agents, where each agent is characterized by an agent-specific (potentially nonlinear) local operator. The goal is for the agents to communicate intermittently via a server to find the root of the average of the agents' local operators. The generality of our setting stems from allowing for (i) Markovian data at each agent and (ii) heterogeneity in the roots of the agents' local operators. The limited recent work that has accounted for both these features in a federated setting fails to guarantee convergence to the desired point or to show any benefit of collaboration; furthermore, they rely on projection steps in their algorithms to guarantee bounded iterates. Our work overcomes each of these limitations. We develop a novel algorithm titled \texttt{FedHSA}, and prove that it guarantees convergence to the correct point, while enjoying an $M$-fold linear speedup in sample-complexity due to collaboration. To our knowledge, \emph{this is the first finite-time result of its kind}, and establishing it (without relying on a projection step) entails a fairly intricate argument that accounts for the interplay between complex temporal correlations due to Markovian sampling, multiple local steps to save communication, and the drift-effects induced by heterogeneous local operators. Our results have implications for a broad class of heterogeneous federated RL problems (e.g., policy evaluation and control) with function approximation, where the agents' Markov decision processes can differ in their probability transition kernels and reward functions.  
\end{abstract}

\section{Introduction}
\label{sec:Intro}
In the classical stochastic approximation (SA) formulation~\cite{robbins1951stochastic, borkar, borkarode, meyn2023}, the {goal} is to solve for a parameter  $\theta^\star\in\R^d$ such that $\bar G(\theta^\star)=0$, where $\bar G:\R^d\rightarrow\R^d$ is a potentially nonlinear operator satisfying $\bar G(\cdot)=\E_{ o\sim \mu}\bracket{G(\cdot, o)}$. Here, $G:\R^d\times\cX\rightarrow\R^d$ is the noisy version of the true operator $\bar{G}$, and $o$ is an observation random variable drawn with some \emph{unknown} distribution $\mu$ from a sample space $\mc X$. The agent (learner) tasked with finding $\theta^\star$ only has access to the operator $\bar G$ via noisy samples $\{G(\cdot, o_t)\}$, 
where $\{o_t\}$ is a stochastic observation process that is typically assumed to converge in distribution to $\mu$~\cite{borkar}. Interestingly, the SA formulation described above captures a large class of problems arising in
stochastic optimization, control, and reinforcement learning
(RL). For instance, popular iterative algorithms like stochastic gradient descent (SGD) in optimization, and temporal
difference (TD) learning, Q-learning in RL turn out to be
specific instances of SA.

Federated learning (FL) has also made extensive use of SA. The general idea of FL is to leverage data collected from multiple agents to train accurate statistical models for downstream prediction~\cite{konevcny, bonawitz, mcmahan}. SA has primarily been used in FL to solve stochastic optimization problems that arise in the context of empirical risk minimization (ERM), under assumptions like offline data collection and i.i.d. (independent and identically distributed) data at each agent. {In contrast, when the data arrives sequentially and exhibits strong temporal correlations - as is the case in multi-agent/federated RL - very little is understood about non-asymptotic/finite-time performance}. This is particularly the case when the agents’ data-generating processes are potentially non-identical. \emph{In our paper, we develop convergence and sample-complexity results for a general class of federated SA methods that significantly improve upon existing bounds, and provide a deeper understanding of the interplay between data-heterogeneity and temporal correlations in FL}.

\textbf{Our Federated SA Setup.} We aim to develop a unified framework that can accommodate a broad class of distributed SA problems. To that end, we consider a setup involving $M$ agents that can exchange information via a central server, as is typical in an FL setting. Each agent $i\in[M]$ has its own true and noisy operators $\bar G_i:\R^d\rightarrow\R^d$ and $G_i: \R^d\times\mc{X}_i\rightarrow\R^d$, respectively, where $\bar G_i(\cdot)=\E_{ o\sim \mu_i}\bracket{G_i(\cdot, o)}$. Here, $\mc X_i$ is the sample space of agent $i$, and the observation random variable $o$ is drawn from an \emph{agent-specific} unknown distribution $\mu_i$. The information available locally at an agent $i$ comprises a sequence of noisy samples $\{G_i(\cdot, o_{i,t})\}.$ To capture temporal correlations, we assume that the the observation sequence $\{o_{i,t}\}$ is generated from an ergodic Markov chain $\mc{M}_i$ with stationary distribution $\mu_i$. Given this setup, the collective objective is to find the root $\theta^\star\in\R^d$ of the ``average'' operator $\bar G$; this objective can be succinctly stated as
\begin{equation}
\boxed{
    \text{Find}\  \theta^\star \, \text{s.t.}\ \bar G(\theta^\star)=0, \text{where}\  \bar G:=\frac{1}{M}\sum_{i=1}^M \bar G_i.} \label{prob::average_operator}
\end{equation}

Note that to achieve the above objective, communication via the server is \emph{necessary}. In a typical FL scenario, communication takes place over low-bandwidth channels, and can hence be costly and slow. To mitigate the communication bottleneck, we will adhere to the standard \emph{intermittent} communication protocol in FL, where an exchange of information with the server takes places only once in every $H > 1$ time-steps, where $H$ is some predefined synchronization time-period. To summarize, there are three main features in our formulation: (1) \textbf{Heterogeneity:} The local operators $\{\bar{G}_i\}$ at the agents can be non-identical and, as such, have different roots; (2) \textbf{Markovian Sampling:} The agents' observation sequences are generated in a Markovian manner; and (3) \textbf{Sparse Communication:} Agents need to perform local computations in isolation between communication rounds. Additionally, complying with the FL paradigm, the agents are required to solve the problem in~\eqref{prob::average_operator} without ever exchanging their raw observations. 

\subsection{Motivation and Related Work}
When one considers the challenging setting involving all of the above features, there are significant gaps in the understanding of the problem in Eq.~\eqref{prob::average_operator}. We now explain these gaps by considering several motivating applications of~\eqref{prob::average_operator}. Along the way, we also review relevant literature. 

\begin{enumerate}
\item\textbf{Motivation 1: Federated Optimization in Dynamic Environments.}

\hspace{1em} $\circ$ \textbf{I.I.D. Data.} When $\bar{G}_i$ is the gradient of a local loss function at agent $i$, and $G_i$ a noisy stochastic version thereof, the problem in Eq.~\eqref{prob::average_operator} reduces to the standard FL setting that has been extensively studied under the assumption that the observation process $\{o_{i,t}\}$ at each agent $i$ is generated in an i.i.d. manner from $\mu_i$~\cite{mcmahan}. Under this i.i.d. data regime, a large body of work has studied the heterogeneous federated optimization problem~\cite{li, khaled2, scaffold, fedlin, gorbunov, mishchenko2022proxskip}.

\hspace{1em} $\circ$ \textbf{Markovian Data.} Our formulation substantially generalizes the standard FL setting by allowing the observation process at each agent $i$ to be generated from an ergodic Markov chain that converges in distribution to $\mu_i$. As explained in detail in~\cite{sun2018markov} and~\cite{duchi}, there is ample reason to consider such Markovian data streams. Moreover, temporally correlated Markovian data arises naturally in several real-world applications: dynamic environments such as non-stationary wireless systems where the channel statistics drift over time following Gauss-Markovian processes, linear dynamical systems in control with input noise, and distributed sensor networks with correlated measurements. With these reasons in mind, considerable recent attention has been given to the study of Markovian gradient descent in the single-agent setting~\cite{sun2018markov, doanMkv, even2023stochastic, beznosikov2024first}. 
\\ \newline
\textcolor{winered}{Research Gap:} Counterparts of the results in the above papers remain open in FL, i.e., \emph{we are unaware of any finite-time results for federated optimization under Markov data.}

\item \textbf{Motivation 2: Heterogeneous Fixed Point Problems.} Our work is particularly inspired by that of \cite{malinovsky}, where the general problem of finding the fixed point of an average of local operators is considered, exactly like we do in~\eqref{prob::average_operator}. As detailed in~\cite{malinovsky}, in addition to optimization, heterogeneous fixed-point problems find applications in many other settings such as alternating minimization methods, nonlinear inverse problems, variational inequalities, and Nash equilibria computation. 
\\ \newline
\textcolor{winered}{Research Gap:} The main gap, however, is that~\cite{malinovsky} considers a ``noiseless'' setting, where each agent can access their true local operator directly, i.e., there is no aspect of i.i.d. or Markovian sampling in their work.

\item \textbf{Motivation 3: Collaborative Multi-Agent RL.} In RL, an agent interacts sequentially with an environment modeled as a Markov decision process (MDP). At each time-step, the agent plays an action according to some policy, observes a reward, and transitions to a new state. The goal is for the agent to play a sequence of actions that maximizes some long-term cumulative utility referred to as the \emph{value-function}, without a priori knowledge of the reward functions and the probability transition kernels of the MDP. It turns out that the problem of estimating the value-function corresponding to a given policy (policy evaluation), and that of finding the optimal policy (control), can be both cast as instances of SA~\cite{meyn2023}. In this context, the sample-complexities of SA-based RL algorithms like TD learning and Q-learning have been studied recently in~\cite{bhandari_finite, srikant, khamaruTD, Waiwright, Qu, li2024q, mitra2024simple}. This line of work has collectively revealed that for contemporary RL problems with large state and action-spaces, several samples are typically needed to achieve a desired performance accuracy. A natural way to overcome the sample-complexity barrier is via parallel data collection from multiple environments. This has led to a new paradigm called \emph{federated reinforcement learning} (FRL)~\cite{qiFRL}, where the idea is to use information from multiple environments (MDPs) to learn a policy that performs ``well'' on average in all MDPs. Compared to FL, the theoretical aspects of FRL remain poorly understood. To explain the gaps in this context, we broadly categorize SA problems in FRL into two main groups as follows; for a quick comparative summary, see Table~\ref{tab::comparison}.

\hspace{1em} $\circ$ \textbf{Homogeneous Environments.} In this case, agents share identical MDPs and local operators. The works~\cite{doan} and~\cite{liuMARL} both examine federated TD algorithms under a restrictive i.i.d. sampling assumption. While~\cite{shen2023} analyzes federated actor-critic algorithms under both i.i.d. and Markov data, they establish a linear speedup in sample-complexity (w.r.t. the number of agents) only under i.i.d. data. To our knowledge, the first paper to establish a linear speedup in sample-complexity under Markovian sampling for contractive SA problems was~\cite{khodadadian}, followed by~\cite{woo2023blessing} for tabular Q-learning. Follow-up works have focused on the minimum amount of communication needed to achieve linear speedups~\cite{tian2024one, salgia2024sample} in sample-complexity, and also the effect of imperfect communication channels~\cite{dal, mitraTDEF, beikmohammadi2024, dal2025finite}.

\hspace{1em} $\circ$ \textbf{Heterogeneous Environments.} In practice, it is unreasonable to expect that the agents' environments are \emph{exactly the same.} Thus, it makes more sense to consider a setting where agents interact with potentially distinct MDPs, \emph{leading to distinct local operators}. This setting has received much less attention. The convergence of federated Q-learning algorithms is analyzed in~\cite{jinFRL}, but no linear speedup is established. The results in \cite{wang2023TMLR} for federated TD learning, and \cite{zhang2024finite} for federated \texttt{SARSA}, do exhibit a linear speedup in the noise variance term, but their bounds also feature an additive heterogeneity-induced bias term that grows with the discrepancy in the agents' environments; such a bias term also appears in~\cite{jinFRL}, and its presence precludes possible gains from collaboration. 
\\ \newline
\textcolor{winered}{Research Gap:} To sum up, in the context of general heterogeneous stochastic approximation problems in FRL under Markov noise, it remains an open problem to establish a linear collaborative speedup that is not affected by any heterogeneity-induced bias term.
\end{enumerate}

\begin{table*}[t!]
\caption{Comparison of finite-time analysis for federated RL papers that study stochastic approximation problems.}
\label{tab::comparison}
\resizebox{\textwidth}{!}{%
\begin{tabular}{lllllll}
\hline
Work                   & Heterogeneity  & No heterogeneity bias & Linear speedup                   & Markovian sampling\\ \hline
\cite{doan}         & \XSolidBrush & - & \XSolidBrush   & \XSolidBrush \\
\cite{liuMARL, shen2023}      & \XSolidBrush & - & \CheckmarkBold & \XSolidBrush \\
\cite{khodadadian, woo2023blessing, salgia2024sample, tian2024one}  & \XSolidBrush   & -                  & \CheckmarkBold                   & \CheckmarkBold      \\
\cite{jinFRL}       & \CheckmarkBold & \XSolidBrush     & \XSolidBrush                     & \XSolidBrush        \\
\cite{wang2023TMLR, zhang2024finite}    & \CheckmarkBold & \XSolidBrush     & \CheckmarkBold                   & \CheckmarkBold       \\
\rowcolor{gray}
This work & \CheckmarkBold & \CheckmarkBold     & \CheckmarkBold                   & \CheckmarkBold     
\end{tabular}%
}
\end{table*}

Now that we have elaborated on the motivation for our work and the research gaps in the literature, let us turn our attention to the main questions investigated in this paper. 

\textbf{Key Questions.} Having motivated the rationale behind studying the problem in Eq.~\eqref{prob::average_operator}, we can now state more precisely the key questions of interest to us. To do so, we start by presenting the convergence rates of SA in the single-agent case as a benchmark. Accordingly, consider the SA scheme of~\cite{robbins1951stochastic}: 
\begin{equation}
    \theta^{(t+1)} = \theta^{(t)}+\alpha_t G(\theta^{(t)},o_t), \label{eqn::stochastic_approx}
\end{equation}
for $t=0,...,T-1$. Here, $T$ denotes the total number of iterations, $\theta^{(t)}\in\R^d$ is the parameter estimate at time-step $t$, and $\{\alpha_t\}$ is the step-size sequence. Asymptotic convergence results for the SA rule in~\eqref{eqn::stochastic_approx} were derived in the seminal works~\cite{tsitsiklisroy} and~\cite{borkarode}. Finite-time rates under Markovian observations were more recently established in~\cite{bhandari_finite, srikant, chen2022finite,mitra2024simple}, where it was shown that under certain mild technical assumptions (which will be detailed later in Section~\ref{sec::problem_formulation}), running $T$ iterations of the rule in Eq.~\eqref{eqn::stochastic_approx} with the same step-size $\alpha_t=\alpha$ leads to the following error bound: 
\begin{equation}
    d_T\leq \underbrace{C_1\exp{\paren{-\alpha C_2 T}}}_{\text{bias}}+\underbrace{\alpha C_3 \sigma^2}_{\text{variance}}, \label{eqn::mse_bound}
\end{equation}
where $d_t:=\E\bracket{\norm{ \theta^{(t)}-\theta^\star}^2}$ denotes the MSE at time-step $t$, $C_1, C_2, C_3$ are problem-specific constants, and $\sigma^2$ characterizes the variance of the noise model. The MSE bound consists of two components: (i) a bias term that decays exponentially fast and (ii) a variance term that captures the effect of noise. Thus, the rule in \eqref{eqn::stochastic_approx} ensures linear convergence to a ball of radius $O(\alpha \sigma^2)$ centered around $\theta^\star$. By choosing an $\alpha$ on the order of $O(\log(T)/T)$, one can then achieve exact convergence to $\theta^\star$. Now consider the federated SA setup with {heterogeneous local operators, Markovian data, and intermittent communication}. We ask:
\\\newline 
\emph{Is it possible to converge exactly (in the mean-square sense) to the root $\theta^\star$ of the average operator $\bar{G}$? Moreover, can one achieve a linear $M$-fold reduction in the variance term, capturing the benefit of collaboration?}
\\ \newline
\textbf{Contributions.} As far as we are aware, no prior work has been able to establish both \emph{exact convergence} and \emph{linear sample-complexity speedups} in FL under Markovian sampling and heterogeneity of local operators. We close this significant gap via the following contributions. 


$\bullet$ \textbf{Motivation for New Algorithm.} In Section~\ref{sec:Motivation}, we study the performance of existing FRL algorithms~\cite{jinFRL, khodadadian, wang2023TMLR, zhang2024finite}, where each agent performs multiple local parameter updates using just its own operator. Proposition~\ref{prob::average_operator} reveals that in a heterogeneous setting, such algorithms fail to match the single-agent SA convergence rate in Eq.~\eqref{eqn::mse_bound}. In particular, when run with a constant non-diminishing step size $\alpha$, such algorithms do not converge exactly to the root $\theta^\star$, even in the absence of noise. 

$\bullet$ \textbf{Novel Algorithm.} Motivated by the findings from Section~\ref{sec:Motivation}, we develop a new local SA procedure called \texttt{FedHSA} in Section~\ref{sec:algo}. The core idea in \texttt{FedHSA} is to modify the local update rule to ensure convergence to the correct point $\theta^\star$. We emphasize that while drift-mitigation techniques to combat heterogeneity have been studied before in federated optimization~\cite{scaffold, gorbunov, fedlin}, \texttt{FedHSA} \emph{is not limited to optimization}. Instead, \texttt{FedHSA} applies much more broadly to general nonlinear SA problems, including those in RL. 

$\bullet$ \textbf{Matching Centralized Rates and Linear Speedup.} In Theorem~\ref{theorem::markov_mse} of Section~\ref{sec: results}, we prove that \texttt{FedHSA} matches the centralized rate in~\eqref{eqn::mse_bound}, and, unlike the results in~\cite{jinFRL, wang2023TMLR, zhang2024finite}, there is \emph{no heterogeneity-induced bias in our final bound}. Furthermore, with a linearly decaying step-size, we prove that \texttt{FedHSA} achieves an optimal sample-complexity bound of $\tilde{\mathcal{O}}(1/(MHT))$ after $T$ communication rounds, with $H$ local steps in each round (Corollary~\ref{coro::markov}). \textbf{This result is significant because it is the first to establish a collaborative $M$-fold linear speedup for heterogeneous federated SA problems under Markovian sampling, with no additional bias term to negate the benefit of collaboration.} 

$\bullet$ \textbf{Analysis Technique.} Even in the single-agent SA setting, a finite-time analysis under Markovian data is quite non-trivial. Our FL setting is further complicated by complex statistical correlations that arise from combining data generated by distinct Markov chains, drift effects due to heterogeneous local operators, and multiple local steps. The only other recent papers~\cite{wang2023TMLR, zhang2024finite} that consider similar settings rely on a projection step in the algorithm to ensure uniform boundedness of iterates. Furthermore, they build on a ``virtual MDP framework'' for their analysis. Our proof neither requires a projection step nor a virtual MDP. In particular, the lack of a projection step entails a much more involved analysis. We elaborate on the main technical challenges  in Section~\ref{sec::Challenges}.  

Although our work is primarily theoretical, we corroborate our main theoretical findings via numerical simulations, covering both optimization and RL, in Section~\ref{sec:sims}. 
\newpage
\textbf{More Related Work.} Since the focus of this paper is on federated stochastic approximation, we have only reviewed SA-based approaches in FRL. We note that policy gradient-based approaches for FRL have also recently been explored in~\cite{xie, fan2021fault, lan2023improved, wang2024momentum, zhu2024towards}. In particular, the approaches developed in~\cite{wang2024momentum} and~\cite{zhu2024towards} manage to achieve collaborative speedups without incurring additive bias terms due to heterogeneity. That said, even in the single-agent setting, the dynamics of policy gradient methods and SA schemes in RL are considerably different, and entail separate treatments. For instance, the aspect of Markovian sampling does not show up at all in policy-gradient methods. 

At the time of preparing this paper, we became aware of a very recent piece of work~\cite{mangold2024} that looks at federated \emph{linear} stochastic approximation (LSA). In this work, the authors provide a detailed and refined analysis of local update rules - of the form of \texttt{FedAvg} - under both i.i.d. and Markov sampling. Furthermore, they develop a bias-corrected algorithm for federated LSA, and show that it simultaneously achieves a linear speedup and no heterogeneity-induced bias. Although the flavor of the results in \cite{mangold2024} is similar to that of our work, there are significant differences in scope, algorithms, and proof techniques that we outline below. First, the results in~\cite{mangold2024} are limited to linear stochastic approximation. In contrast, we consider general nonlinear operators throughout this work. An 
 inspection of the refined analysis in~\cite{mangold2024} seems to suggest that the recurrence relations exploited in~\cite{mangold2024} rely heavily on the linearity of the underlying operator. As such, it is unclear whether the ideas in~\cite{mangold2024} can be easily extended to accommodate the much broader class of operators we consider in this work. Second, the results under Markov sampling in~\cite{mangold2024} are provided only for the vanilla federated LSA schemes, not for the bias-corrected one. Thus, even in light of the contributions made in~\cite{mangold2024}, our work is the first to establish a linear speedup result without heterogeneity bias for general (nonlinear) contractive SA under Markov sampling. Third, the authors in~\cite{mangold2024} use the ``blocking technique" to handle temporal correlations under Markov sampling. This requires modifying the original algorithm so that it only operates on a sub-sampled data sequence, where the sub-sampling gap is informed by the mixing time of the underlying Markov chain. In contrast, our proof technique is more direct, does not go through the blocking apparatus, and, as such, requires no modification to the algorithm for Markov data. 

\vspace{-3mm}
\section{Setting and Motivation}\label{sec::problem_formulation}
\vspace{-2mm}
\subsection{Setting}
\vspace{-2mm}
We consider a multi-agent heterogeneous SA problem involving $M$ agents, where each agent $i \in [M]$ has its own local true operator $\bar{G}_i$. Since the true operators are generally hard to evaluate exactly, each agent $i$ can access $\bar{G}_i$ only through a sequence of noisy samples $\{G_i(\cdot, o_{i,t})\}.$ The observation $o_{i,t}$ made at time-step $t$ by agent $i$ is sequentially sampled from an underlying \emph{agent-specific} time-homogeneous Markov chain $\cM_i$ with stationary distribution $\mu_i$. We further have $\bar G_i(\cdot)=\E_{ o\sim \mu_i}\bracket{G_i(\cdot, o)}$. We consider the case where the agents' Markov chains $\{\mc{M}_i\}$ share a common \emph{finite} state space $\cS$, but have potentially different probability transition matrices. The collaborative goal is to solve the root-finding problem described in~\eqref{prob::average_operator} within a federated framework, where the agents communicate intermittently via a central server only once in every $H$ time-steps, while keeping their raw observation sequences $\{o_{i,t}\}$ private. 

\textbf{Working Assumptions.} We now make certain standard assumptions on the agents' operators and stochastic observation processes for our subsequent analysis. 

\begin{assumption}[Lipschitzness]\label{ass::smoothness}
The local noisy operator $G_i$ for each agent $i\in[M]$ is $L$-Lipschitz, i.e., there exists a constant $L\geq 1$ such that given any observation $o$, for all $\theta_1, \theta_2\in \R^d$, we have 
\begin{equation}
    \norm{G_i(\theta_1, o) - G_i(\theta_2, o)}\leq L\norm{\theta_1 - \theta_2}.  
\end{equation}
Furthermore, there exists $\sigma_i \geq 1$ for each $i\in[M]$ such that for any given $\theta, o$, the following holds
\begin{equation}
    \norm{G_i(\theta, o)}\leq L(\norm{\theta}+\sigma_i). \label{eqn::uniform_bound}
\end{equation}
\end{assumption}

\begin{assumption}[1-point strong monotonicity]\label{ass::strong_monotonicity}
The average true operator $\bar G$ is 1-point strongly monotone w.r.t. $\theta^\star$, i.e., there exists some constant $\mu\in(0,1]$ such that for any $\theta\in\R^d$, we have 
\begin{equation}
    \iprod{\theta-\theta^\star}{\bar G(\theta) - \bar G(\theta^\star)}\leq -\mu\norm{\theta-\theta^\star}^2.
\end{equation}
\end{assumption}

In the context of optimization, Assumption~\ref{ass::smoothness} corresponds to a smoothness assumption typical in the analysis of Markovian gradient descent~\cite{doanMkv}, and Assumption~\ref{ass::strong_monotonicity} corresponds to strong-convexity. As for RL, Assumptions~\ref{ass::smoothness} and~\ref{ass::strong_monotonicity} both hold for TD learning with linear function approximation (LFA)~\cite{bhandari_finite, srikant}, and certain variants of Q-learning with LFA~\cite{chen2022finite, zengTAC}, where $\bar G_i$ and $G_i$ correspond to the non-noisy and noisy versions, respectively, of the TD/Q-learning update rules. Assumptions~\ref{ass::smoothness} and~\ref{ass::strong_monotonicity} suffice to guarantee the MSE bound in Eq.~\eqref{eqn::mse_bound} in the centralized case, i.e., when $M=1$. Since our objective is to obtain such a bound for the multi-agent heterogeneous setting, it is natural for us to work under the same assumptions. 

The next assumption on the agents' Markov chains helps control the effect of temporal correlations in the data, and appears in almost all finite-time analysis papers for both single-agent~\cite{bhandari_finite, srikant, chen2022finite, adibi2024stochastic} and multi-agent RL~\cite{khodadadian, zengTAC, wang2023TMLR}. 

\begin{assumption} \label{ass::irreducible}
    For each agent $i\in[M]$, the state space of the underlying Markov chain $\cM_i$ is finite, and the Markov chain $\cM_i$ is aperiodic and irreducible. 
\end{assumption}

A key implication of Assumption~\ref{ass::irreducible} is that it implies geometric mixing of each of the agents' Markov chains~\cite{levin2017markov}. More precisely, for each $i \in [M]$, there exists some $c_i \geq 1$ and some $\rho_i \in (0,1)$, such that the following is true for any state $s \in \mc{S}$:
\begin{equation}\label{eqn::tv}
    d_{TV}(\P(o_{i,t}=\cdot\ |o_{i,0}=s), \mu_i)\leq c_i {\rho_i^t},\ \forall t > 0,
\end{equation}
where $d_{TV}(P, Q)$ is the total variation distance between two probability measures $P$ and $Q$, $o_{i,t}$ denotes the state of agent $i$ at the $t$-th time-step, and $\mu_i$ is the stationary distribution of $\cM_i$. In simple words, Assumption~\ref{ass::irreducible} implies that each agent $i$'s Markov chain converges to its stationary distribution $\mu_i$ exponentially fast. 

We now define the concept of a \emph{mixing time} that plays a key role in our analysis. Intuitively, the mixing time of a Markov chain measures how long it takes for the chain to ``forget" its starting state and become close to its stationary distribution. For any $\varepsilon > 0$, we define the mixing time of the Markov chain $\mc{M}_i$ (at precision level $\varepsilon$) as $\tau_i(\varepsilon):= \min\{ t \in \mathbb{N}_0: c_i \rho^t_i \leq \varepsilon\}.$ We define $\tau(\varepsilon):= \max_{i \in [M]} \tau_i(\varepsilon)$ as the mixing time corresponding to the slowest-mixing Markov chain.

Our final assumption concerns statistical independence between the data across different agents. Such an assumption is needed to establish ``linear speedups'' in performance, and has appeared in prior FRL work~\cite{khodadadian, woo2023blessing,  wang2023TMLR, zhang2024finite}. 

\begin{assumption}
\label{ass::independence}
For every pair of agents $i \neq j \in [M]$, the observation processes $\{o_{i,t}\}$ and $\{o_{j,t}\}$ are statistically independent. 
\end{assumption}

With the above assumptions in place, we are in a position to describe and analyze our proposed algorithm. Before doing so, however, it is natural to ask: \emph{What is the need for a new federated SA algorithm?} We provide a concrete answer in the next section. 
\vspace{-2mm}
\subsection{Motivation for a new Federated SA Algorithm}
\vspace{-2mm}
\label{sec:Motivation}
To explain the motivation for developing our algorithm, it suffices to focus on the class of linear stochastic approximation (LSA) problems, where for each agent $i \in [M]$, $\bar{G}_i(\theta) = \bar{A}_i \theta - \bar{b}_i$, and $G_i(\theta, o_{i,t}) = A_i(o_{i,t}) \theta - b_i(o_{i,t})$, i.e., the operators are affine in the parameter $\theta$. Furthermore, to isolate the effect of heterogeneity, we will consider a simplified ``noiseless'' setting where each agent $i$ can directly access its true operator $\bar{G}_i(\theta).$ Our goal is to formally establish that even for this simplified scenario, if one employs the existing algorithms in~\cite{jinFRL, khodadadian, wang2023TMLR, zhang2024finite}, then it might be impossible to match the convergence rates in the single-agent setting. To see this, we first note that the algorithms in these papers operate in rounds $t=0, 1,\ldots, T-1,$ where within each round $t$, each agent $i$ performs $H \geq 1$ local model-update steps of the following form:
\begin{equation}
\theta_{i,\ell+1}^{(t)} =  \theta_{i,\ell}^{(t)} + \eta \bar{G}_i(\theta_{i,\ell}^{(t)}), \ell = 0, 1, \ldots, H-1,
\label{eqn:localSA}
\end{equation}
where $\theta_{i,\ell}^{(t)}$ is agent $i$'s parameter estimate in local step $\ell$ of communication round $t$, and $\eta \in (0,1)$ is a step-size. For each $i \in [M]$, $\theta^{(t)}_{i,0}$ is initialized from a common global parameter $\bar{\theta}^{(t)}$. At the end of the $t$-th round, each agent $i$ transmits $\theta^{(t)}_{i, H}$ to the server; the server then broadcasts the next global parameter $\bar{\theta}^{(t+1)} = (1/M) \sum_{i\in [M]} \theta^{(t)}_{i, H}$ to all the agents. We will analyze local SA rules of the form in Eq.~\eqref{eqn:localSA} under the standard assumptions made to derive finite-time rates for linear SA in the single-agent case~\cite{srikant}: for all $i \in [M]$, (i) (Lipschitzness) the 2-norms of $A_i$ and $b_i$ are bounded, and (ii) (strong-monotonicity) all the eigenvalues of $A_i$ have strictly negative real parts, i.e., $A_i$ is Hurwitz. When $M=1$, i.e., there is only one agent, and $H=1$, i.e., communication occurs every time-step, the rule in Eq.~\eqref{eqn:localSA} ensures exponentially fast convergence to the root of the underlying operator~\cite{srikant}. Our next result reveals that this is no longer the case when $M >1$ and $H >1.$ To state this result, we define $\bar{A} := (1/M) \sum_{i \in [M]} \bar{A}_i$, $\bar{b} := (1/M) \sum_{i \in [M]} \bar{b}_i$, and $\bar{A}':=(1/M) \sum_{i \in [M]} \bar{A}^2_i$. 

\begin{proposition} \label{prop:hetbias} Suppose $M >1$ and $H=2$. Consider the local SA update rule in Eq.~\eqref{eqn:localSA}, and suppose the step-size $\eta$ is chosen such that the matrix $(I+2\eta \bar{A} + \eta^2 \bar{A}')$ is Schur-stable. Then, we have:
$
\lim_{t \to \infty} e_t = \eta v$, where $v = (1/M) (2 \bar{A} + \eta \bar{A}')^{-1}  \sum_{i \in [M]} \bar{A}^2_i (\theta^\star_i - \theta^\star), $   $e_t = \bar{\theta}^{(t)} - \theta^\star, \theta^\star_i = \bar{A}^{-1}_i \bar{b}_i$ is the root of the local operator $\bar{G}_i$, and $\theta^\star = \bar{A}^{-1} \bar{b}$ is the root of the global operator $\bar{G}.$ 
\end{proposition}

The main takeaway from Proposition~\ref{prop:hetbias} is that even with just $2$ local steps (i.e., $H=2$) and no noise, in the limit, there is a \emph{non-vanishing error} $\eta v$ that depends on how much each local root $\theta^\star_i$ differs from the global root $\theta^\star$. To eliminate this error, the step-size $\eta$ must be diminished with time, i.e., one cannot afford to use a constant step-size like in the single-agent case. Note, however, if a diminishing step-size sequence is used, one would not be able to achieve exponentially fast convergence. Thus, \emph{there is a gap between the bounds in the single-agent case and those achievable with algorithms of the form in Eq.~\eqref{eqn:localSA} in the heterogeneous federated SA setting}. We now proceed to  develop \texttt{FedHSA} that will not only close this gap, but also achieve an optimal $M$-fold sample-complexity reduction due to collaboration. 

\begin{algorithm}[t!]
\caption{\texttt{FedHSA}} 
\label{algo:FedHSA}
\begin{algorithmic}[1]
\State \textbf{Input:} Local step-size $\eta$, global step-size $\alpha_g$, initial parameter $\bar\theta^{(0)}$, initial noisy operator $G(\Bar{\theta}^{(0)})$. 
\For {$t=0,\ldots,T-1$} 
\For {$i=1,\ldots, M$} 
\State \hspace{-8mm} Agent $i$ initializes its local parameter $\theta_{i,0}^{(t)}=\bar\theta^{(t)}.$
\For {$\ell=0,\ldots, H-1$}
\State \hspace{-2mm} Agent $i$ observes $o_{i,\ell}^{(t)}$ generated from its Markov chain $\mc{M}_i$, and updates $\theta^{(t)}_{i, \ell}$ as per~\eqref{eqn::FedHSA_update}. 
\EndFor
\State \hspace{-8mm} Agent $i$ transmits $\Delta_{i,H}^{(t)}=\theta_{i,H}^{(t)}-\bar{\theta}^{(t)}$ to server. 
\EndFor
\State \hspace{-2mm} Server broadcasts $\bar\theta^{(t+1)}$ computed as in~\eqref{eqn::simple_average}.
\For {$i=1,\ldots, M$}
\State \hspace{-8mm} Agent $i$ transmits $G_i(\Bar{\theta}^{(t+1)},o_{i,0}^{(t+1)})$ to server.
\EndFor
\State Server broadcasts average operator $G(\Bar{\theta}^{(t+1)})$.
\EndFor
\end{algorithmic}
\end{algorithm}
\vspace{-3mm}
\section{Proposed Algorithm: \texttt{FedHSA}}
\label{sec:algo}
In this section, we will develop our proposed algorithm titled \texttt{Federated Heterogeneous Stochastic Approximation (FedHSA)}, designed carefully to account for heterogeneous local operators and intermittent communication. We now elaborate on the steps of \texttt{FedHSA}, outlined in Algorithm~\ref{algo:FedHSA}. \texttt{FedHSA} adheres to the standard intermittent communication model in FL, where communication takes place in rounds $t=0,1,\ldots, T-1$. At the beginning of each round $t$, a central server broadcasts the global parameter $\bar\theta^{(t)}$ to all the agents, who then perform $H$ steps of local updates; we will describe the local update process shortly. We denote the local parameter of agent $i$ at the $\ell$-th local step of the $t$-th communication round as $\theta_{i,\ell}^{(t)}$, with ${\theta_{i,0}^{(t)}}$ initialized from $\bar\theta^{(t)}$. In each local step $\ell$ of round $t$, agent $i$ interacts with its own environment and observes $o_{i,\ell}^{(t)}$ from its Markov chain $\mc{M}_i$. Using this observation, agent $i$ computes the noisy operator $G_i(\theta_{i,\ell}^{(t)}, o_{i,\ell}^{(t)})$. Note here that we define each local step as a time-step, and thus $o_{i,\ell}^{(t)}$ can be equivalently denoted as $o_{i,tH+\ell}$. 

\textbf{The Core Idea.} The core technique involves the local update rule at each agent. As revealed in Section~\ref{sec:Motivation}, if each agent makes local updates by simply taking steps along its own operator, then it can be impossible to converge to $\theta^\star$, while maintaining the same convergence rates as in the centralized setting. However, this is precisely what is done in the existing FRL literature~\cite{jinFRL, khodadadian, woo2023blessing, wang2023TMLR, zhang2024finite}, where the local update rule takes the form 
\begin{equation}
\theta_{i,\ell+1}^{(t)} =  \theta_{i,\ell}^{(t)} + \eta G_i(\theta_{i,\ell}^{(t)},o_{i,\ell}^{(t)}), 
\label{eqn::local_update_fedavg}
\end{equation}
with  $\eta>0$ being the local step-size. When each agent $i$ follows the update rule in~\eqref{eqn::local_update_fedavg} for several local steps, it tends to naturally drift towards the root $\theta^\star_i$ of its own local operator $\bar{G}_i.$ As such, the reason why update rules of the form in~\eqref{eqn::local_update_fedavg} fail to achieve the desired MSE bound in~\eqref{eqn::mse_bound} can be attributed to the following simple observation: \emph{in the heterogeneous setting, the root $\theta^\star$ of the global operator $\bar{G}$ may not coincide with the average $(1/M) \sum_{i \in [M]} \theta^\star_i$ of the roots of the agents' local operators.}  

 We now develop a drift-mitigation technique that overcomes this issue. To start with, we observe that if each agent had the luxury of talking to the server at every time-step, the ideal update rule of the global parameter would be $\bar \theta^{(t+1)} = \bar \theta^{(t)} + \alpha_g\eta G(\bar\theta^{(t)})$, where $G(\bar\theta^{(t)}):=(1/M)\sum_{i\in[M]}G_i(\bar\theta^{(t)},o_{i, t})$. Under the intermittent communication model, however, this is not feasible since an agent $i$ does not have access to the information from the other agents in $[M]\setminus\{i\}$ during each local step. Accordingly, our algorithm exploits the \emph{memory of the global operator $G(\bar \theta^{(t)})$ from the beginning of communication round $t$ to guide the local updates of each agent during round $t$.} To be concrete, at each local step $\ell$ of round $t$, agent $i$ adds the correction term $G(\bar\theta^{(t)})-G_i(\bar\theta^{(t)}, o_{i,0}^{(t)})$ to its local update direction $G_i(\theta_{i,\ell}^{(t)}, o_{i,\ell}^{(t)})$ to account for drift-effects, leading to the update rule for \texttt{FedHSA}:
\begin{equation}
\boxed{
    \theta_{i,\ell+1}^{(t)} =  \theta_{i,\ell}^{(t)} + \eta \paren{G_i(\theta_{i,\ell}^{(t)},o_{i,\ell}^{(t)})+G(\bar\theta^{(t)})-G_i(\bar\theta^{(t)}, o_{i,0}^{(t)})}.}\label{eqn::FedHSA_update}
\end{equation}

To gain further intuition about the above rule, suppose for a moment that every agent can access the noiseless versions of their local operators. In this case, the noiseless version of  \texttt{FedHSA} would take the form: $ \theta_{i,\ell+1}^{(t)} =  \theta_{i,\ell}^{(t)} + \eta \paren{\bar{G}_i(\theta_{i,\ell}^{(t)})+\bar{G}(\bar\theta^{(t)})-\bar{G}_i(\bar\theta^{(t)})}.$ Now suppose the global parameter $\bar{\theta}^{(t)}$ is $\theta^\star$. Since $\theta^{(t)}_{i,0} = \bar{\theta}^{(t)}$ and $\bar{G}(\theta^\star)=0$ by definition, observe that all subsequent iterates of the agents remain at $\theta^\star$. Said differently, if one initializes \texttt{FedHSA} at $\theta^\star$, the iterates never evolve any further, exactly as desired, i.e., the root $\theta^\star$ of the operator $\bar{G}$ is a \emph{stable equilibrium point} of \texttt{FedHSA}. After $H$ local steps, each agent $i$ transmits their local parameter change $\Delta_{i,H}^{(t)}:=\theta_{i,H}^{(t)}-\bar \theta^{(t)}$ to the central server, and the global parameter $\bar \theta(t)$ is updated as follows with global step-size $\alpha_g$:
\begin{equation}
    \bar\theta^{(t+1)} = \bar\theta^{(t)} + \frac{\alpha_g}{M}\sum_{i\in [M]}\Delta_{i,H}^{(t)}. \label{eqn::simple_average}
\end{equation}
\vspace{-5mm}
\section{Main Results and Discussion}
\label{sec: results}
As a warm-up to our main convergence result for \texttt{FedHSA}, we first consider a simpler setting where the observation $o_{i,\ell}^{(t)}$ made by each agent $i\in[M]$ at local iteration $\ell$ and round $t$ is drawn i.i.d. from the stationary distribution $\mu_i$ of its underlying Markov chain $\cM_i$. With $d_t:=\E\bracket{\norm{ \bar\theta^{(t)}-\theta^\star}^2}$, we have the following result for this setting. 

\begin{theorem}\label{thm::iid_mse} 
Suppose Assumptions~\ref{ass::smoothness} to~\ref{ass::independence} hold, and consider the i.i.d. sampling model described above. Define $\alpha=H\eta\alpha_g$ as the effective stepsize, and $\sigma := \max\{\{\sigma_i\}_{i\in[M]}, \norm{\theta^\star}, 1\}$. Then, there exists a universal constant $C$, such that with $\alpha_g=1$ and $\eta\leq\mu/(2CL^2H)$, \texttt{FedHSA} guarantees the following $\forall T\geq 0$:
\begin{equation}
\begin{aligned}
    d_T&\leq \exp\paren{-\frac{\mu}{2}\alpha T}\norm{\bar\theta^{(0)}-\theta^\star}^2+\bigo{\frac{\alpha L^2}{\mu MH}+\frac{\alpha^2L^4}{\mu^2}}\sigma^2. 
\nonumber
\end{aligned}
\end{equation} 
\end{theorem}

Next, we present our main convergence result under Markovian sampling. 

\begin{theorem}[\textbf{Main Result}]\label{theorem::markov_mse}
Suppose Assumptions~\ref{ass::smoothness} to~\ref{ass::independence} hold. Define $\bar\tau=\tau(\alpha^2)$ and $\rho = \max_{i\in[M]}\rho_i$. Then, there exists a universal constant $C'\geq 1$, such that by selecting $\alpha_g=1, \eta\leq \mu/(C'\bar \tau L^2H)$, the following holds for \texttt{FedHSA} for any $T\geq 2\bar\tau$:
\begin{equation}
\begin{aligned}
    d_T\leq \exp\paren{-\frac{\mu}{4}\alpha T}\bigo{d_0+\sigma^2}+\bigo{\frac{\bar\tau \alpha L^2}{\mu MH(1-\rho)}+\frac{\alpha^2L^4}{\mu^2}}\sigma^2. 
\end{aligned}
\label{eqn::markov}
\end{equation}
\end{theorem}
The next result is an immediate corollary of Theorem~\ref{theorem::markov_mse}. 

\begin{corollary}[\textbf{Linear Speedup}]\label{coro::markov}
Suppose all the conditions in Theorem~\ref{theorem::markov_mse} hold. Then, by choosing $\eta=4\log(MHT)/(\mu HT)$, and $T\geq(L^2/\mu^2)\max\{4C'\bar\tau \log(MHT), MH(1-\rho)/\bar\tau\}$, \texttt{FedHSA} guarantees the following for any $T\geq 2\bar\tau$:
\begin{equation}
    d_T\leq\Tilde{\mathcal{O}}\paren{\paren{d_0+\frac{\bar\tau L^2\sigma^2}{\mu^2(1-\rho)}}\frac{1}{MHT}}.
\end{equation}
\end{corollary}

We provide detailed convergence proofs of Theorems~\ref{thm::iid_mse} and~\ref{theorem::markov_mse} in Appendices~\ref{app:iidproof} and~\ref{app:Markovproof}, respectively. We will provide a proof sketch for Theorem~\ref{theorem::markov_mse} shortly in Section~\ref{sec::Challenges}. Before doing so, several comments are in order.

\textbf{Discussion.} Comparing our bounds for the i.i.d. (Theorem~\ref{thm::iid_mse}) and Markov settings (Theorem~\ref{theorem::markov_mse}), we note that the only difference comes from the fact that the noise variance term $\sigma^2$ in the Markov case gets inflated by an additional factor $\bar{\tau}/(1-\rho)$ capturing the rate at which the slowest mixing Markov chain approaches its stationary distribution. Such an inflation by the mixing time is typical for problems with Markovian data~\cite{nagaraj}. 

$\bullet$ \emph{Matching Centralized Rates.} Theorem~\ref{theorem::markov_mse} reveals that \texttt{FedHSA} guarantees exponentially fast convergence to a ball around $\theta^\star$. In particular, comparing \eqref{eqn::markov} with~\eqref{eqn::mse_bound}, we conclude that \texttt{FedHSA} recovers the known finite-time bounds for single-agent SA in~\cite{bhandari_finite, srikant, chenQ, chen2022finite}. 

$\bullet$ \emph{Linear Speedup Effect.} From Eq.~\eqref{eqn::markov}, notice that the radius of the ball of convergence around $\theta^\star$ is the sum of two terms: an $\mathcal{O}(\alpha \bar{\tau} \sigma^2/(MH))$ term that gets scaled down by the number of agents $M$, and a higher-order $\bigo{\alpha^2 \sigma^2}$ term that can be made much smaller relative to the first term by making $\alpha$ sufficiently small, i.e., the dominant noise term exhibits a ``variance-reduction'' effect. To further highlight this effect, Corollary~\ref{coro::markov} reveals that with a decaying step-size, the sample-complexity of \texttt{FedHSA} is $\tilde{\mathcal{O}}(\sigma^2/(MHT))$; \emph{this is essentially the best rate one can hope for} since after $H$ local steps in $T$ communication rounds, the total number of data samples collected across $M$ agents is precisely $MHT$. The $M$-fold reduction in sample-complexity makes it explicit that \emph{even in a heterogeneous federated SA setting with time-correlated data, one can achieve linear speedups in performance by collaborating using our proposed algorithm \texttt{FedHSA.}} This is the first result of its kind and significantly generalizes similar bounds in the homogeneous setting~\cite{khodadadian, woo2023blessing}. 


$\bullet$ \emph{No Heterogeneity Bias.} In Proposition~\ref{prop:hetbias}, we saw that if one employs the algorithms in~\cite{jinFRL, wang2023TMLR, zhang2024finite} in the heterogeneous setting, then there is a heterogeneity-induced bias term in the final bound that can potentially negate the benefits of collaboration. \texttt{FedHSA} effectively eliminates such a bias term \emph{without making any assumptions whatsoever on the level of heterogeneity}. To see this, it suffices to note from Eq.~\eqref{eqn::markov} that in the noiseless case when $\sigma^2=0$, \texttt{FedHSA} guarantees exponentially fast convergence to $\theta^\star$, as opposed to a ball of radius $\mathcal{O}(\eta)$ around $\theta^\star$ like in Proposition~\ref{prop:hetbias}. 
\begin{figure}[t!]
\centering
\includegraphics[width=0.6\textwidth]{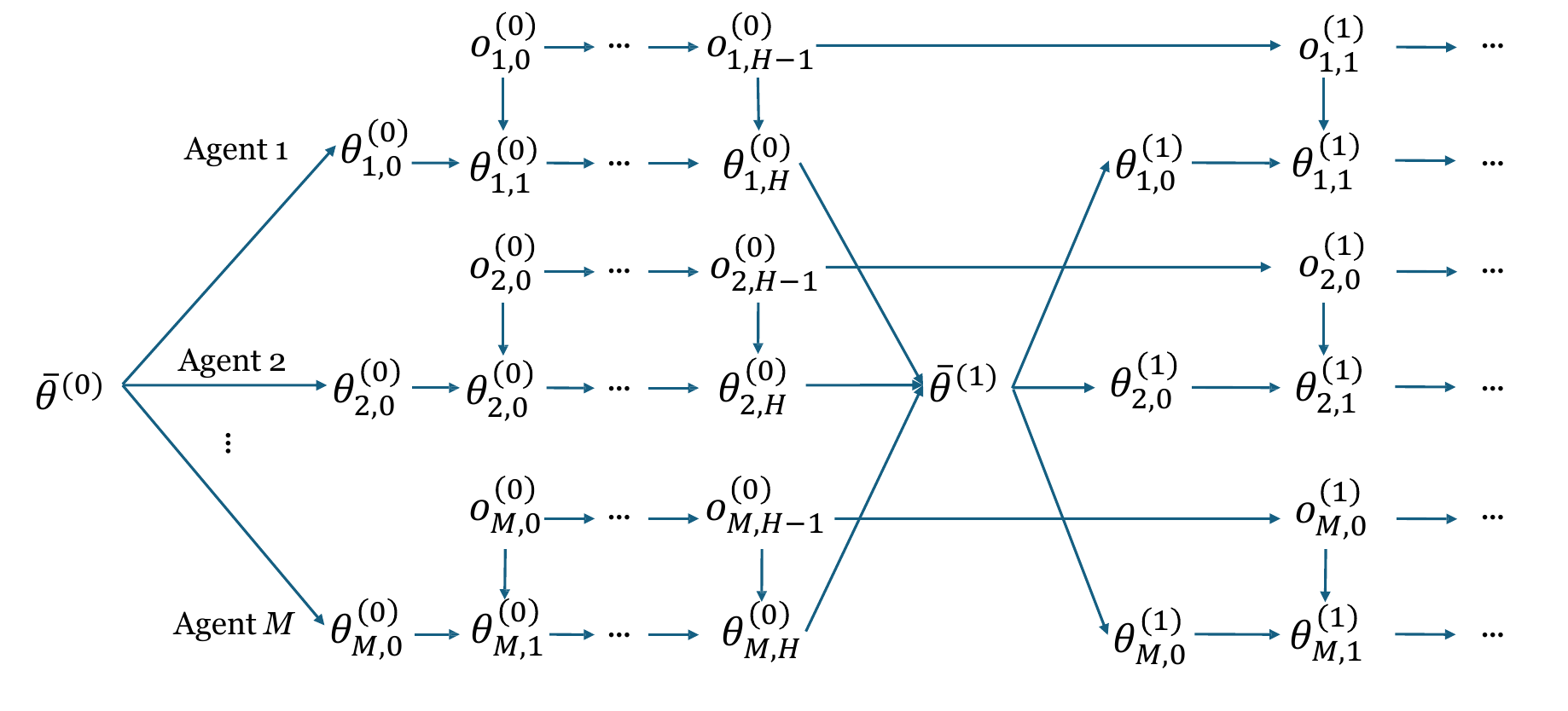}
\caption{Illustration of the different statistical correlations in our heterogeneous federated SA setting.}
\label{fig::interaction}
\vspace{-5mm}
\end{figure}
\vspace{-2mm}
\section{Challenges and Proof Sketch}
\label{sec::Challenges}
Arriving at our main result in Theorem~\ref{theorem::markov_mse} is highly non-trivial, and requires overcoming several technical challenges arising from the interplay between complex statistical correlations, drift effects due to heterogeneity, and multiple local update steps. We elaborate on these \textbf{challenges} as follows. 

$\bullet$ As illustrated in Figure~\ref{fig::interaction}, our analysis has to account for two types of data correlations: (i) temporal correlations in the data for any given agent, and (ii) correlations induced by exchanging data across agents. For each agent $i$, the observation $o_{i,\ell}^{(t)}$ is statistically dependent on all prior observations for itself since they are all part of a \emph{single Markovian trajectory.} Such an issue does not arise in the standard FL setting where one assumes i.i.d. data over time. Additionally, the local parameter $\theta_{i,\ell}^{(t)}$ is jointly influenced by all the parameters of all agents up to the beginning of round $t$, due to communication. Finally, all such parameters inherit randomness from prior Markovian observations. In short, \emph{combining information generated by heterogeneous Markov chains creates complex spatial and temporal correlations.}

$\bullet$ Since the local operators across agents are heterogeneous and can have distinct roots, local updates create the need to control a heterogeneity-induced drift effect. Such an issue does not arise in~\cite{khodadadian} and~\cite{woo2023blessing}, where all local operators share a common root. 

$\bullet$ To tame the effect of heterogeneity, the recent papers of~\cite{wang2023TMLR} and~\cite{zhang2024finite} in FRL assume a projection step in the algorithm to ensure uniform boundedness of iterates. This considerably simplifies parts of their analysis. In contrast, we do not assume any projection step, making it much harder to contend with the biases introduced by Markovian sampling and heterogeneity. 

\textbf{Proof Sketch.} We now provide a high-level technical overview of our analysis; the details are deferred to Appendix~\ref{app:Markovproof}. Our first step is to establish a one-step recursion that captures the progress made by \texttt{FedHSA} in each communication round (Lemma~\ref{lem::onestep_markov} in Appendix~\ref{app:Markovproof}). Up to a higher order term in $\alpha$, the R.H.S. of this recursion comprises four terms: a ``good'' term that leads to a contraction in the mean-square error, a ``noise variance'' term that gets scaled down by $M$, a ``drift'' term due to heterogeneity, and a bias term due to Markov sampling. The challenging part of this result is showing the variance-reduction effect under Markov sampling; to do so, we carefully exploit the geometric mixing properties of the agents' Markov chains and Assumption~\ref{ass::independence}. 

Next, to control the drift effect due to heterogeneity, we show in Lemma~\ref{lem::drift} (in Appendix~\ref{app:basic}) that if $\eta \leq 1/(LH)$, then the following is true \emph{deterministically}:
$$ \norm{\theta_{i,\ell}^{(t)}-\bar \theta^{(t)}}^2\leq \cO(\eta^2 L^2H^2)\left(\norm{\bar\theta^{(t)}-\theta^\star}^2+\sigma^2\right). $$
To build intuition, notice that when there is no noise, i.e., $\sigma = 0$, if $\bar\theta^{(t)} = \theta^\star$, meaning the iterate at the start of the round is at the desired value, there would be no client-drift at all, precisely as desired.
The most challenging part of our analysis pertains to controlling a Markovian bias term:
$$T_{bias}=:\E\bracket{\iprod{\bar\theta^{(t)}-\theta^\star}{\frac{2\alpha}{MH}\sum_{i=1}^M\sum_{\ell=0}^{H-1} \paren{G_{i}(\theta_{i,\ell}^{(t)})-\bar G_{i}(\theta_{i,\ell}^{(t)})}}},$$ 
that arises due to temporal correlations in data. Such a term vanishes in the standard FL setting where one assumes i.i.d. data. In the heterogeneous FRL setting in~\cite{wang2023TMLR} and~\cite{zhang2024finite}, such a term is simplified by assuming a projection step. Since we do not assume a projection step in \texttt{FedHSA}, we cannot benefit from such simplifications. Nonetheless, we establish the following key result. 

\textbf{Claim (Informal).} \textbf{(Markovian Bias Control).} \textit{Under the conditions of Theorem~\ref{theorem::markov_mse}, we have:} 
\begin{equation}
\begin{aligned}
  T_{bias} &\leq \paren{\frac{\alpha\mu}{2}+\bigo{\bar\tau L^2\alpha^2+\frac{L^4\alpha^3}{\mu}}}\E\bracket{\norm{\bar\theta^{(t)}-\theta^\star}^2}\\
    &\quad+\bigo{\frac{  L^4 \sigma^2\alpha^3}{\mu}+\frac{\bar\tau L^2\sigma^2\alpha^2}{MH(1-\rho)}}. 
\nonumber
\end{aligned}
\end{equation}

A formal version of this claim appears as Lemma~\ref{lem::markov_bias} in Appendix~\ref{app:Markovproof}. For the Markovian sampling result to closely resemble the i.i.d. case, we need to crucially ensure that (i) the iterate-dependent term in the bias can be dominated by the contractive ``good'' term, and (ii) the noise terms are either $\mathcal{O}(\alpha^3)$, or $\mathcal{O}(\alpha^2/M)$, i.e., the noise terms need to be either higher-order in $\alpha$ or exhibit an inverse scaling with $M$, to preserve the linear speedup effect. In the single-agent case~\cite{bhandari_finite, srikant}, one need not worry about linear speedups, and, as such, the corresponding analysis is much less involved. In summary, to arrive at our desired bounds, we need to significantly depart from known analysis techniques for both single- and multi-agent SA under Markov sampling.

\section{Experimental Results}\label{sec:sims}
In this section, we present numerical results for three heterogeneous federated SA tasks subject to Markovian noise, which provide empirical support for our theoretical framework. In our experiments, we aim to convey \textbf{two key messages}: (i) Our \texttt{FedHSA} algorithm eliminates the heterogeneity bias and (ii) \texttt{FedHSA} achieves linear speedup w.r.t. the number of agents. The experimental setups are described in detail in the sequel.
\subsection{Federated Quadratic Loss Minimization Problem}\label{exp::FQLM}
Consider the classical FL framework involving $M$ agents: 
\begin{equation} 
\min_{\theta\in\R^d} f(\theta) = \frac{1}{M}\sum_{i=1}^M f_i(\theta). \label{problem::quad} 
\end{equation} Here, 
\begin{equation} 
f_i(\theta)=\frac{1}{2}\theta^TA_i\theta - b_i^T\theta + c_i,\quad\quad\forall i\in[M] 
\end{equation} 
where $A_i \in \R^{d \times d}$ is a positive definite matrix, $b_i \in \R^d$ is a $d$-dimensional vector, and $c_i \in \R$ is a scalar for $i=1,\cdots, M$. Since $\{f_i\}_{i=1}^M$ are quadratic functions with positive definite $A_i$'s, the gradients are given by $\nabla f_i(\theta) = A_i\theta - b_i$, for $i = 1, \cdots, M$.

Problem~\eqref{problem::quad} is a good fit for our federated SA framework. Specifically, the true local operator $\bar G_i(\cdot)$ corresponds to the true negative gradient $-\nabla f_i(\cdot)$ for each $i\in[M]$. At time step $t$, each agent $i\in[M]$ has access only to an estimator $G_i(\cdot)$ of the true operator, which is corrupted by additive Markovian noise: 
\begin{equation} 
G_i(\theta^{(t)}) = \bar G_i(\theta^{(t)}) + \xi_i^{(t)},\quad\quad \forall i\in[M] 
\end{equation} 
where the noise samples $\{\xi_i^{(t)}\}_{t\geq 0}$ are drawn from a discrete-time continous-state Markov chain.

We now explain how the Markovian noise is generated. For each agent $i$, we maintain a state vector $z_i \in \R^d$, initialized to zero, which evolves as 
\begin{equation} 
z_i^{(t+1)} = Q_i z_i^{(t)} + \epsilon_i^{(t)}. 
\end{equation} 
Here, $Q_i \in \R^{d \times d}$ is a \textit{Schur-stable matrix}, ensuring that all eigenvalues of $Q_i$ lie within the unit ball. This guarantees that $z_i^{(t)}$ does not diverge. The noise $\epsilon_i^{(t)}$ is a zero-mean Gaussian noise with variance $\sigma_\epsilon^2$ and  covariance matrix given by $\sigma^2_{\epsilon} I$, i.e.,  $\epsilon_i^{(t)} \sim \cN(0, \sigma_\epsilon^2 I)$. The Markovian noise $\xi_i^{(t)}$ is directly taken from the state vector: 
\begin{equation} 
\xi_i^{(t)} = z_i^{(t)}. 
\end{equation} It can be shown that this noise is Markovian~\cite{tu2018least}, and mixes geometrically fast (as needed by our theory). 

To validate (i), we consider Problem~\eqref{problem::quad} and solve it using the conventional FL framework, where each agent performs local steps. We compare \texttt{FedHSA} with the existing local SGD approach, which does not include a correction term during local updates as shown in~\eqref{eqn::local_update_fedavg} (referred to as the ``\texttt{Local SA} approach'' in what follows). These two approaches are evaluated under both noiseless conditions and with the presence of additive Markovian noise. The experimental configurations are as follows: $M = 20$ agents, each agent performs $H = 10$ local steps, the learning rate is $\eta = 0.001$, the parameter dimension is $d = 10$, and the performance is measured by $E_t:=\norm{\bar\theta^{(t)} - \theta^\star}^2$.

\begin{figure}[htbp]
    \centering
    \subfigure[Noiseless setting]{
        \includegraphics[width=0.45\textwidth]{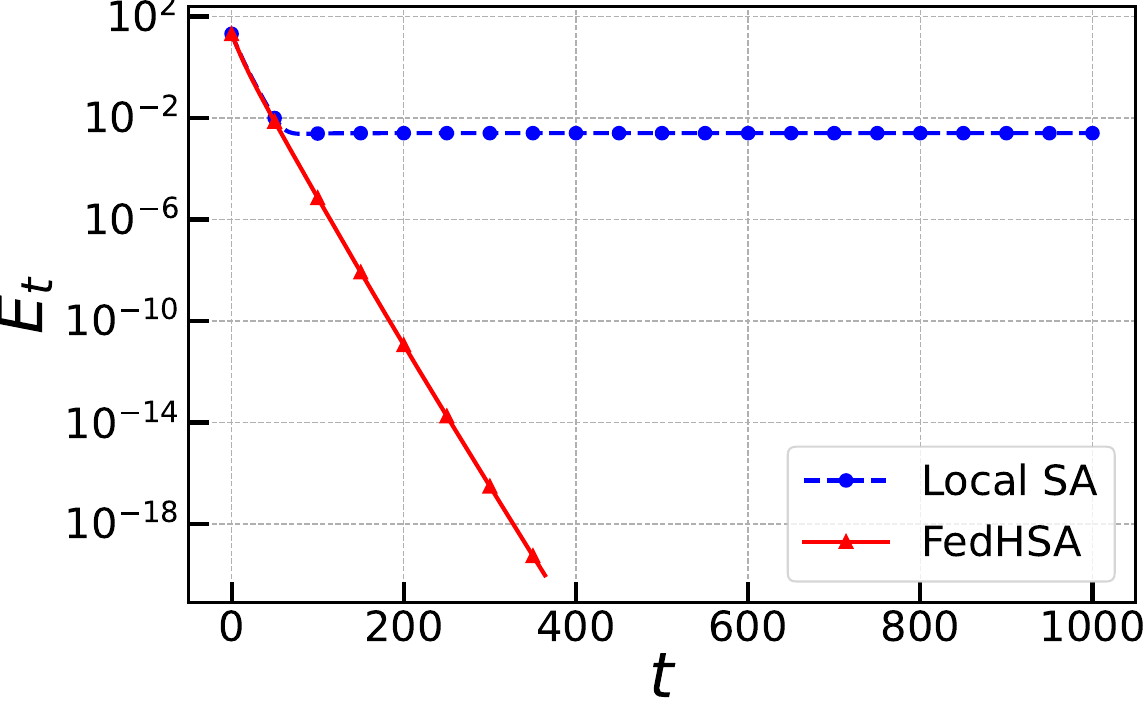}
        \label{fig:noiseless}
    }
    \hfill
    \subfigure[Additive Markovian noise setting]{
        \includegraphics[width=0.45\textwidth]{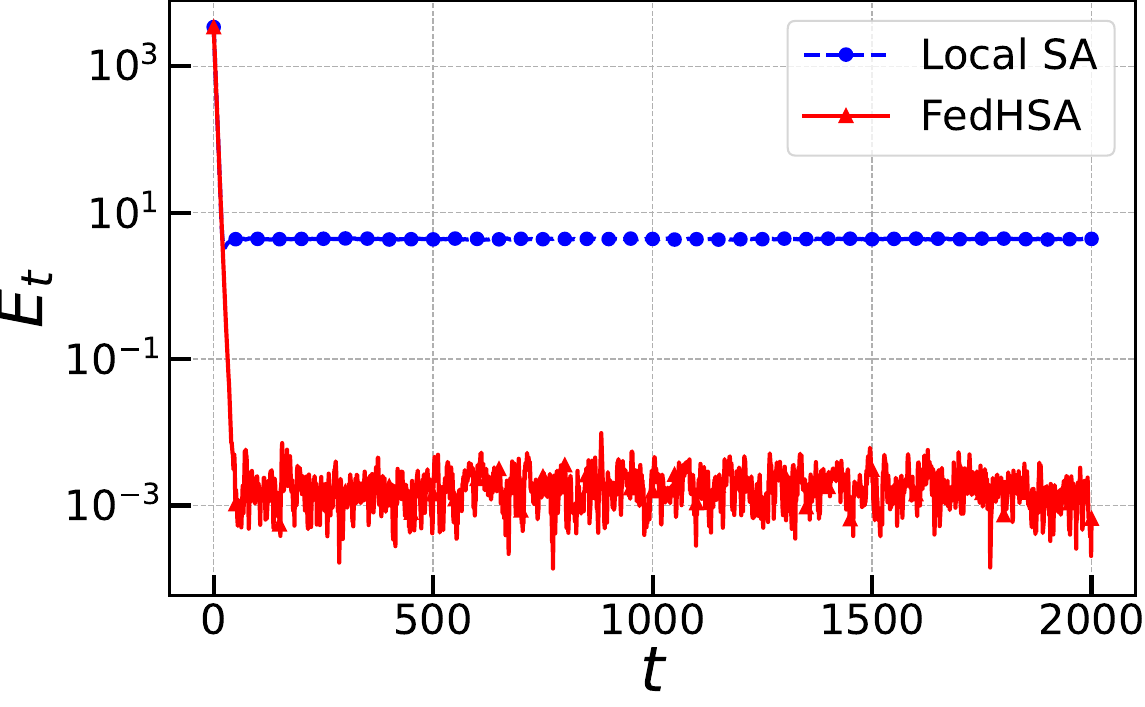}
        \label{fig:noisy}
    }
    \caption{Comparison between \texttt{Local SA} and \texttt{FedHSA}}
    \label{fig:comp_vanilla}
\end{figure}

As shown in Figure~\ref{fig:noiseless}, even in the noiseless case where each agent has access to the true operator $\bar G_i$, the \texttt{Local SA} approach fails to converge to the optimal point $\theta^\star$ due to the heterogeneity bias, as explained in Proposition~\ref{prop:hetbias}. In contrast, \texttt{FedHSA} demonstrates linear convergence towards $\theta^\star$. This result aligns with the theoretical prediction in~\eqref{eqn::markov}, where the algorithm converges exponentially fast to $\theta^\star$ in expectation when $\sigma^2 = 0$ (noiseless case).

In Figure~\ref{fig:noisy}, we introduce additive Markovian noise with $\sigma_\epsilon^2 = 0.01$ into the local operators $G_i$'s. In this noisy setting, we observe that the \texttt{FedHSA} method exhibits a lower error floor compared to \texttt{Local SA}. This is because \texttt{FedHSA} effectively eliminates the heterogeneity bias, with the resulting error being solely attributed to the Markovian noise. Additionally, the impact of this noise is mitigated by a factor of $M$, owing to the linear speedup of the \texttt{FedHSA} algorithm.

To further verify the linear speedup effect in (ii), we compare the results of \texttt{FedHSA} with different numbers of agents. We consider the same problem~\eqref{problem::quad} equipped with the \texttt{FedHSA} algorithm for $M=1,5,20,100$. The other configurations remain the same.

\begin{figure}[t!]
\centering
\includegraphics[width=0.45\textwidth]{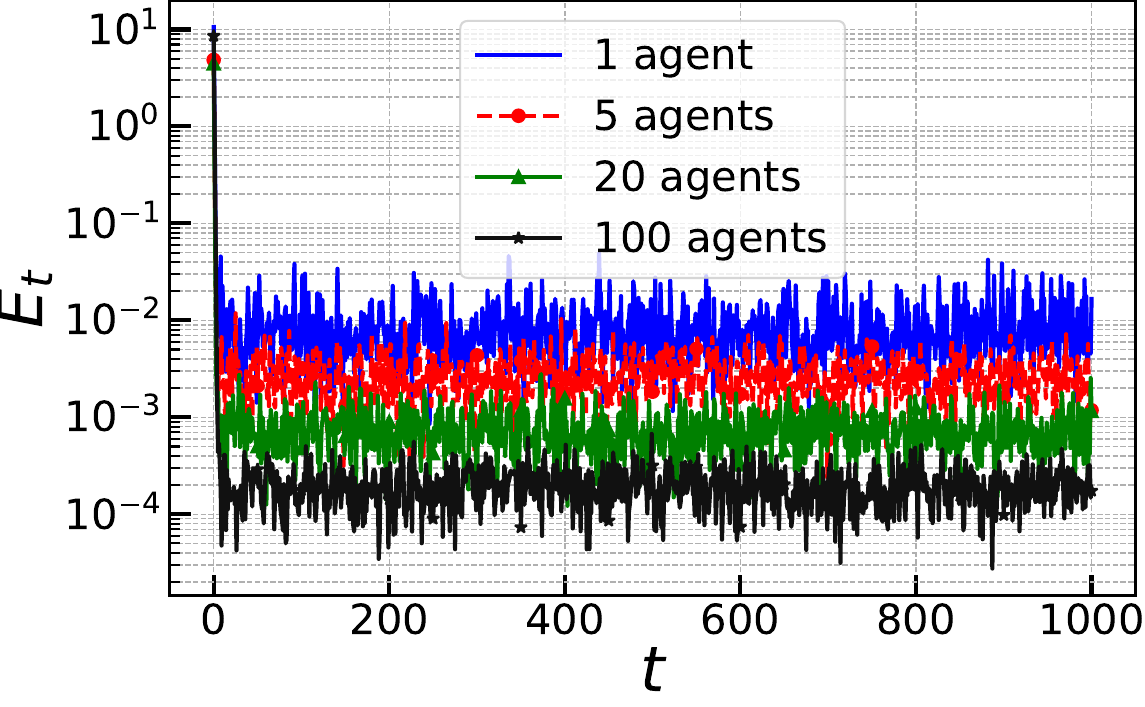}
\caption{Comparison between different numbers of agents for the \texttt{FedHSA} algorithm.}
\label{fig::comp_diff_num}
\end{figure}

Figure~\ref{fig::comp_diff_num} demonstrates a lower error floor with an increasing number of agents for our \texttt{FedHSA} algorithm. This is exactly what we expect, since Corollary~\ref{coro::markov} clearly states that with a proper choice of the step-size $\eta$, the expected error floor $d_T$ is upper-bounded by $\tilde{\cO}(1/(MHT))$, which is inversely proportional to the number of agents $M$.

We also present results in scenarios where Assumption~\ref{ass::strong_monotonicity} does not hold, meaning each $A_i$ for $i\in[M]$ is symmetric but not necessarily positive definite. Consequently, the objective function $f_i$ becomes nonconvex. Removing this assumption significantly expands the applicability of our algorithms. As demonstrated in Figures~\ref{fig:comp_vanilla_nonconvex} and~\ref{fig::comp_diff_num_nonconvex}, the observed results remain consistent with those obtained under strongly convex objectives.

\begin{figure}[htbp]
    \centering
    \subfigure[Noiseless setting]{
        \includegraphics[width=0.45\textwidth]{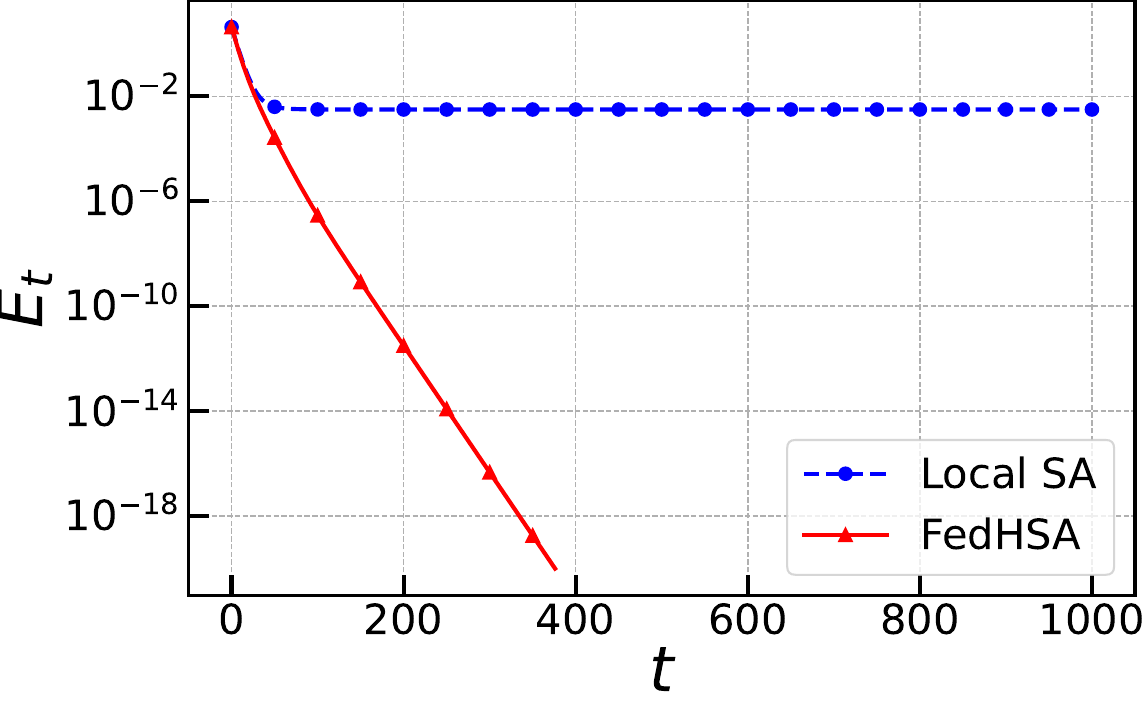}
        \label{fig:noiseless_nonconvex}
    }
    \hfill
    \subfigure[Additive Markovian noise setting]{
        \includegraphics[width=0.45\textwidth]{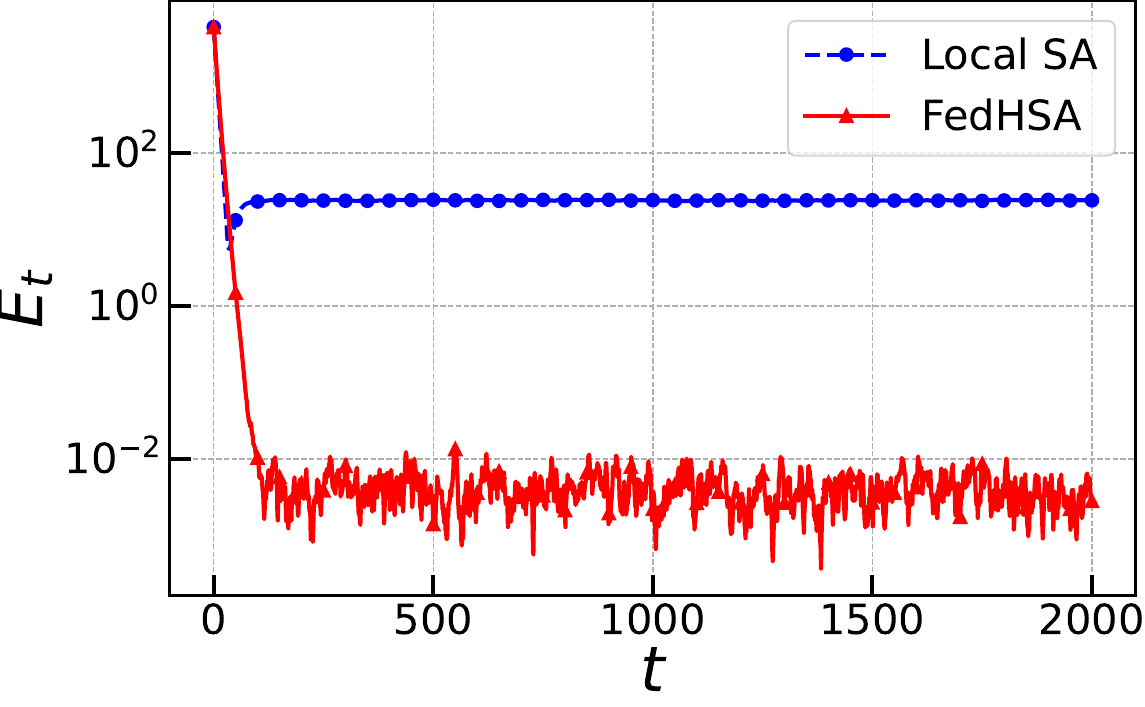}
        \label{fig:noisy_nonconvex}
    }
    \caption{Comparison between \texttt{Local SA} and \texttt{FedHSA} with nonconvex objectives}
    \label{fig:comp_vanilla_nonconvex}
\end{figure}

\begin{figure}[t!]
\centering
\includegraphics[width=0.45\textwidth]{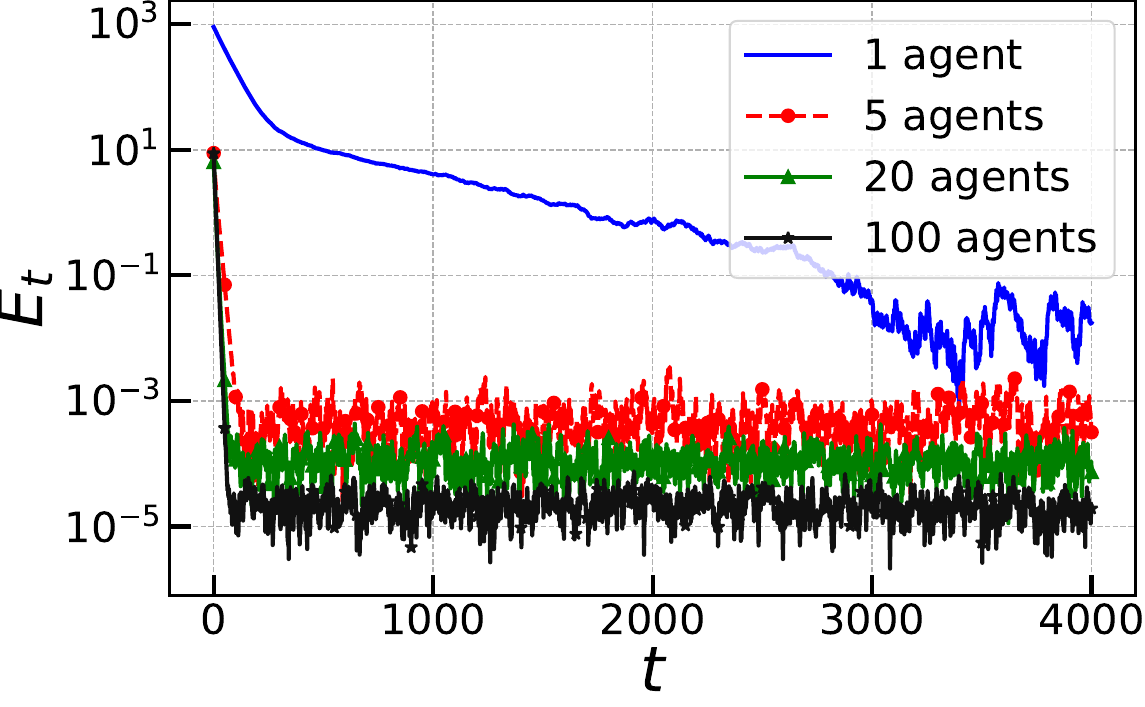}
\caption{Comparison between different numbers of agents for the \texttt{FedHSA} algorithm with nonconvex objectives.}
\label{fig::comp_diff_num_nonconvex}
\end{figure}

\subsection{Federated TD Learning with Linear Function Approximation}
We proceed to explore the application of \texttt{FedHSA} to FRL via focusing on the setting of federated TD learning with LFA. We begin by providing a detailed explanation of the problem formulation.

Consider a total of $M$ agents, each agent $i$ interacting with its individual environment equipped with a fixed policy $\mu_i$, which can be modeled as a Markov reward process (MRP). Suppose that all the MRPs have identical finite state and action spaces, though not necessarily the same transition matrices, reward functions, or discount factors. Specifically, the MRP of the $i$-th agent is denoted as $\cM_i=(\cS,\cA,\cP_i,R_i,\gamma_i)$, where $\cS$ is the state space with cardinality $S$; $\cA$ is the action space with cardinality $A$; $\cP_i$ is the transition kernel dictated by the local policy $\mu_i$; $R_i$ is the reward function; and $\gamma_i$ is the discount factor with $0<\gamma_i<1$. We denote $r_i^{(t)}$ as the reward observed by agent $i$ at time-step $t$.

For self-containedness, we reiterate some basic concepts. For agent $i$ and the underlying MRP $\cM_i$, the associated policy is $\mu_i$, and the value-function $V_i$ is defined as:
\begin{equation}
    V_i(s)=\E\left[\sum_{t=0}^\infty \gamma_i^t r_i^{(t)}\mid s_i^{(t)}=s, \mu_i\right],
\end{equation}
where $s_i^{(t)}$ is the state of agent $i$ at time-step $t$.

In many RL applications, the state and action spaces can be extremely large, making it impractical to store the value-function for each state $s$. To address this, feature matrices are often used to approximate the value-function. One common approach is the LFA framework:
\begin{equation}
    \Tilde V_i=\Phi_i\theta,
\end{equation}
where $\Tilde V_i$ is the approximated value-function for agent $i$ in vector form, $\Phi_i\in\R^{S\times d}$ is the feature matrix specific to agent $i$, consisting of $d$ linearly independent feature vectors $\{\phi_{i,k}\}_{k=1}^d$, and $\theta\in\R^d$ is the nominal parameter. Here, we make the general assumption that all $M$ agents do not necessarily use the same set of feature vectors~\cite{doan2023finite}.

The \textbf{goal} in this problem is for the agents to collectively find a parameter $\theta^\star$ such that it best approximates the value-functions across all agents, i.e., 
\begin{equation}
    \bar\Phi \theta^\star \approx \frac{1}{M}\sum_{i=1}^M V_i,
\end{equation}
where $\bar\Phi=\frac{1}{M}\sum_{i=1}^M \Phi_i$.

To this end, each agent updates its parameters by taking the direction of the negative gradient of the sample Bellman error ${BE}_i$ at observation $o_{i,t}:=\{s_{i}^{(t)}, r_{i}^{(t)}, s_{i}^{(t+1)}\}$ w.r.t. parameter $\theta^{(t)}$ at time-step $t$, obtained via interacting with its own environment $\cM_i$:
\begin{equation}
    {BE}_i(\theta^{(t)},o_{i,t}):=\frac{1}{2}\left(r_{i}^{(t)}+\gamma \phi_i^T(s_{i}^{(t+1)})\theta^{(t)}-\phi_i^T(s_{i}^{(t)})\theta^{(t)}\right)^2,\label{meanpatherror}
\end{equation}
where $\phi_i(s_i^{(t)})$ is the feature vector of agent $i$ for state $s_i^{(t)}$. The server then collects the local parameters for aggregation, exactly as in the FL framework.

The negative gradient step of \eqref{meanpatherror} is given by:
\begin{equation}
    g_{i}(\theta^{(t)},o_{i,t})=\left(r_{i}^{(t)}+\gamma \phi_i^T(s_{i}^{(t+1)})\theta^{(t)}-\phi_i^T(s_{i}^{(t)})\theta^{(t)}\right)\phi_i(s_{i}^{(t)}).\label{exp:markov}
\end{equation}
Note that in~\eqref{exp:markov}, the gradient $g_{i}(\theta^{(t)},o_{i,t})$ is \textit{implicitly integrated with Markovian noise} since the states are sampled from the underlying Markov chain.

\begin{figure}[htbp]
    \centering
    \subfigure[Noiseless setting]{
        \includegraphics[width=0.45\textwidth]{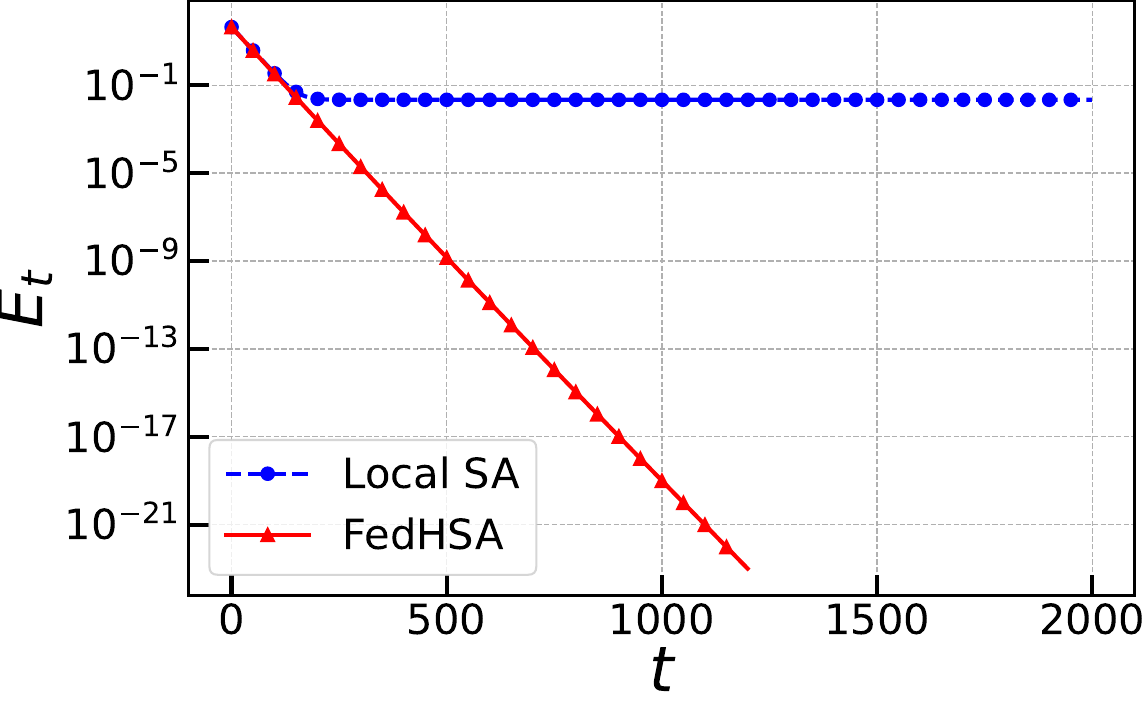}
        \label{fig:noiseless_TD}
    }
    \hfill
    \subfigure[Additive Markovian noise setting]{
        \includegraphics[width=0.45\textwidth]{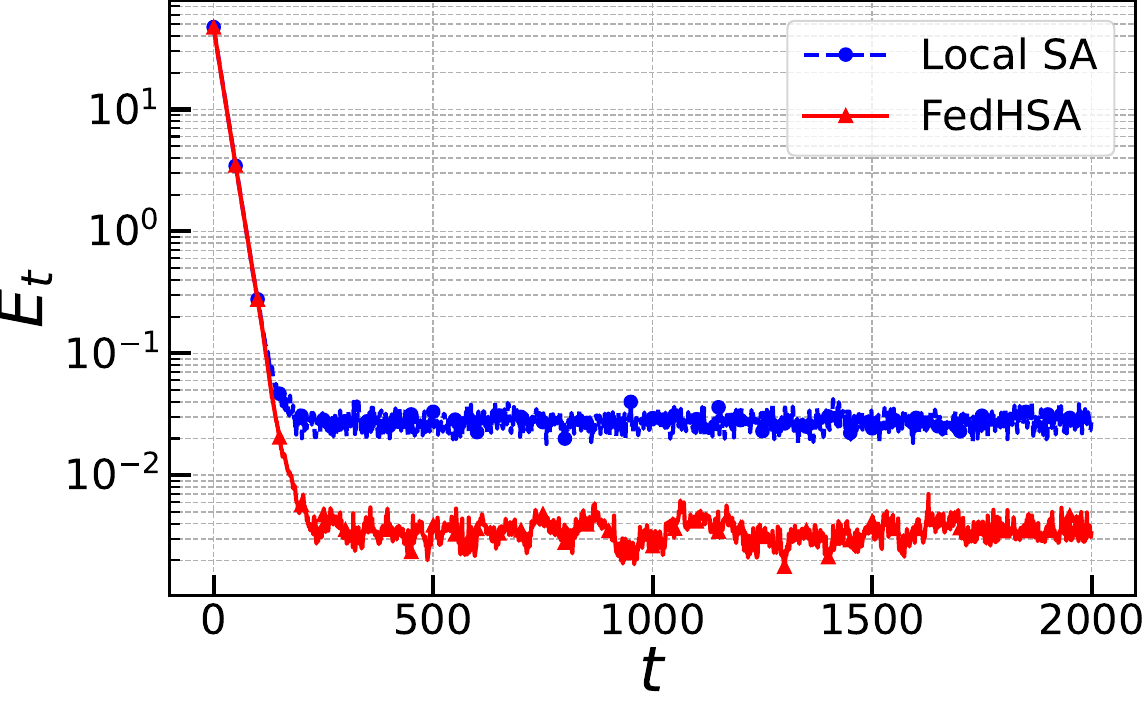}
        \label{fig:noisy_TD}
    }
    \caption{Comparison between \texttt{Local SA} and \texttt{FedHSA}}
    \label{fig:comp_vanilla_TD}
\end{figure}

\begin{figure}[hbtp]
\centering
\includegraphics[width=0.45\textwidth]{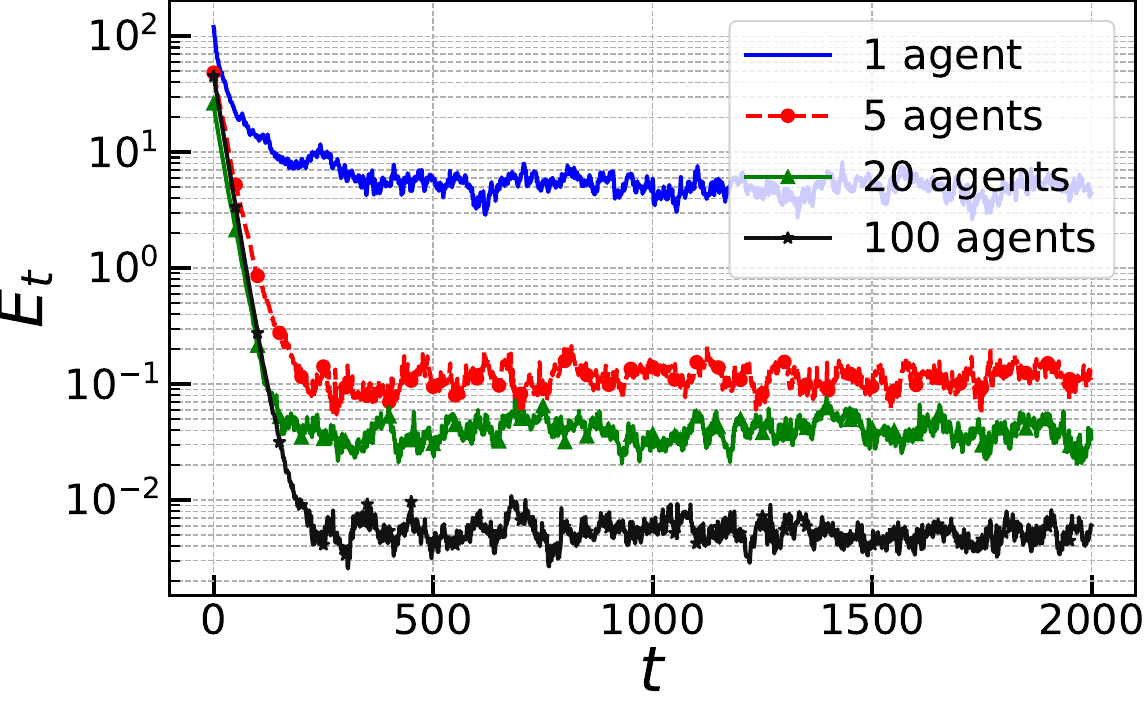}
\caption{Comparison between different numbers of agents for the \texttt{FedHSA} algorithm with the federated TD setting}
\label{fig::TD_agents}
\end{figure}

We also provide the expression for the expected negative gradient step, which one can interpret as the \textit{noiseless gradient}:
\begin{equation}
    \bar g_i(\theta^{(t)})=\Phi_i^T D_i(T_i\Phi_i\theta^{(t)}-\Phi_i\theta^{(t)}),
\end{equation}
where $D_i=\text{Diag}(\pi_i)$, $\pi_i$ is the stationary distribution of $\cP_i$, and $T_i$ is the Bellman operator for agent $i$. We refer the reader to~\cite{bhandari_finite} for more details.

Our federated SA setup~\eqref{prob::average_operator} precisely captures this setting. Specifically, $\bar G_i$ corresponds to the noiseless gradient operator 
$\bar g_i$ for agent $i\in[M]$, and $G_i$ corresponds to the noisy operator $g_i$ incorporated with Markovian noise.

To validate message (i), we compare our \texttt{FedHSA} algorithm with \texttt{Local SA} where there is no correction term in the local update of each agent. We consider a federated TD learning setting with LFA involving $M=200$ agents, with each agent performing $H=10$ local steps. Each MRP $\cM_i$ has $S=100$ states, and the rank of each feature matrix $\Phi_i$ is $d=50$.

Figure~\ref{fig:comp_vanilla_TD} clearly demonstrates that 	\texttt{FedHSA} converges exponentially fast to $\theta^\star$ in the noiseless case and consistently outperforms \texttt{Local SA} both in the noiseless case and the one with Markovian noise.

We validate message (ii) by comparing \texttt{FedHSA} for different numbers of agents $M=1,5,20,100$. Figure~\ref{fig::TD_agents} shows improved error bounds with an increase in the number of agents, substantiating the linear speedup effect.

\subsection{Federated Finite-Sum Minimization Problem with Quadratic Loss}

\begin{figure}[htbp]
    \centering
    \subfigure[Noiseless setting]{
        \includegraphics[width=0.45\textwidth]{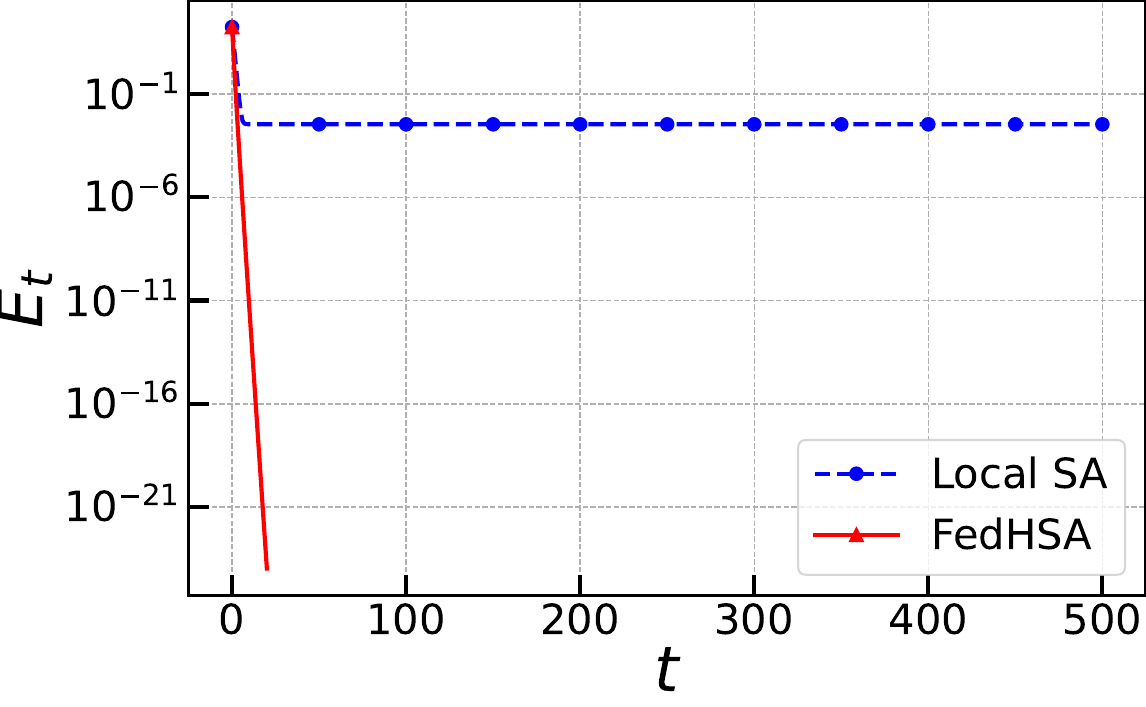}
        \label{fig:noiseless_finite}
    }
    \hfill
    \subfigure[Additive Markovian noise setting]{
        \includegraphics[width=0.45\textwidth]{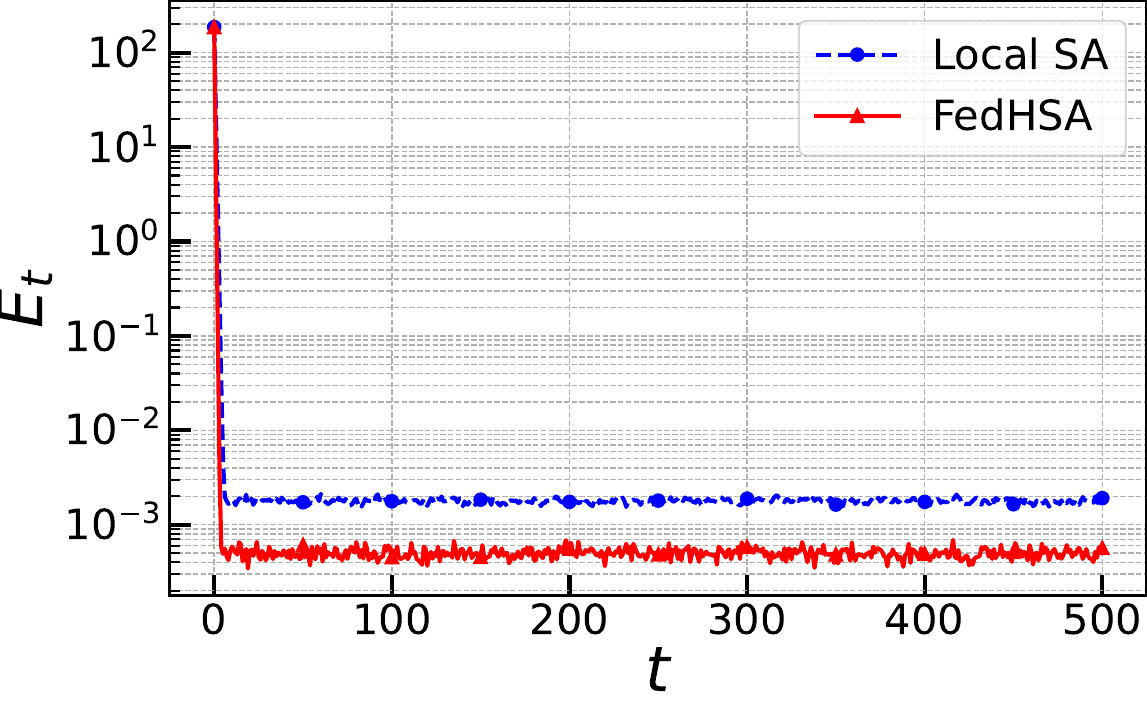}
        \label{fig:noisy_finite}
    }
    \caption{Comparison between \texttt{Local SA} and \texttt{FedHSA}}
    \label{fig:comp_vanilla_finite}
\end{figure}

\begin{figure}[hbtp]
\centering
\includegraphics[width=0.45\textwidth]{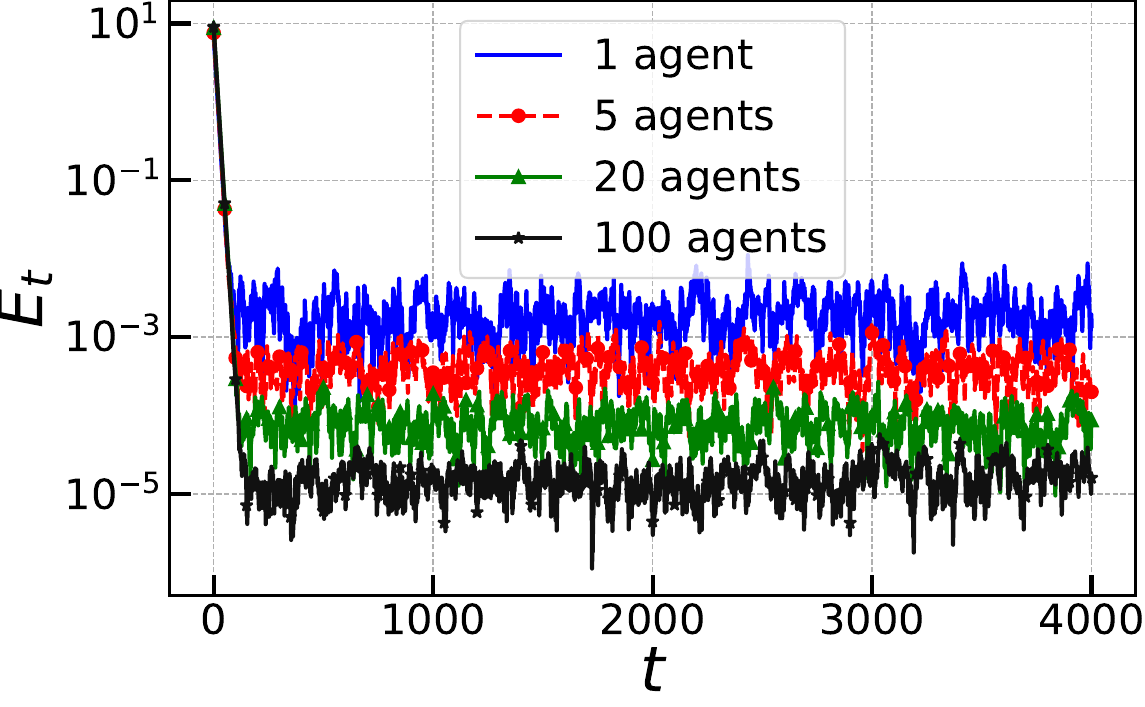}
\caption{Comparison between different numbers of agents for the \texttt{FedHSA} algorithm with the finite-sum setting}
\label{fig::comp_finite}
\end{figure}
In section~\ref{exp::FQLM}, Markovian noise was introduced by directly adding noise to the true local operator. Here, we adopt an alternative approach to incorporating Markovian noise: sampling data points from a Markov chain in a finite-sum setting. The detailed problem formulation is presented as follows.

We consider the same federated minimization problem as in~\eqref{problem::quad}, but with a different structure for the local loss functions $f_i$'s. Specifically, for each $i\in[M]$, we have
\begin{equation}
    f_i(\theta) = \frac{1}{N} \sum_{j=1}^N f_{i,j}(\theta) = \frac{1}{N} \sum_{j=1}^N \paren{\frac{1}{2}\theta^TA_{i,j}\theta - b_{i,j}^T\theta + c_{i,j}}.
\end{equation}
Here, $A_{i,j} \in \R^{d \times d}$ is a positive definite matrix, $b_{i,j} \in \R^d$ is a $d$-dimensional vector, and $c_{i,j} \in \R$ is a scalar for $j=1,\cdots, N$. In essence, each local loss function is the average of $N$ quadratic loss functions, giving rise to the name ``finite-sum setting."

We now describe the incorporation of Markovian noise. For intuition, consider the case where the noise is i.i.d., meaning each agent \( i \in [M] \) selects one quadratic loss function \( f_{i,j} \) \textit{uniformly at random} and computes its gradient to determine the descent direction. Under Markovian noise, however, instead of selecting data samples uniformly at random, agent \( i \) selects \( f_{i,j} \) based on a discrete Markov chain \( \cM_i \). The states of the Markov chain \( \cM_i \) correspond to the indices \( [N] \) of the loss functions \( \{f_{i,j}\}_{j=1}^N \), and its stationary distribution \( \mu_i \) is the uniform distribution over \( [N] \). This ensures that in the limit, \( G_i \) is an unbiased estimator of the true operator \( \bar G_i \), as \( \bar G_i(\cdot) = \E_{o \sim \mu_i} \bracket{G_i(\cdot, o)} \) holds only when \( \mu_i \) is uniform over \( [N] \).

We set up the experiment with $M=200$, $N=2$, $d=200$ and $\eta=0.01$. Figure~\ref{fig:comp_vanilla_finite} validates message (i) exactly as in the previous experiments. Figure~\ref{fig::comp_finite} compares the performance of \texttt{FedHSA} under varying numbers of agents, specifically $M=1, 5, 20, 100$, with $N=10$. Consistent with the observations from the previous experiments, the error floor decreases as the number of participating agents increases. This empirically validates our theoretical result that $d_T \leq \tilde{\cO}(1/(MHT))$, demonstrating the benefits of collaborative speedup in reducing error.

\section{Conclusion}
We studied a general federated stochastic approximation problem subject to heterogeneity in the agents' local operators, and temporally correlated Markovian data at each agent. For this setting, we first showed that standard local SA algorithms may fail to converge to the right point. Based on this observation, we developed a new class of heterogeneity-aware federated SA algorithms that simultaneously (i) guarantee convergence to the right point while matching centralized rates, and (ii) enjoy linear speedups in sample-complexity that are not degraded by the presence of an additive heterogeneity-induced bias term. Our results subsume the standard federated optimization setting and have implications for a broad class of SA-based RL algorithms. There are various interesting questions that we plan to explore in the future:

\begin{enumerate}
\item One natural question is to ask how much communication is necessarily needed to achieve optimal speedups in sample-complexity? 

\item At the moment, our bounds scale with the mixing time of the slowest mixing Markov chain among the agents. Can this bound be further refined? 

\item Can the techniques developed in this paper be employed for more complex SA-based RL algorithms with nonlinear function approximators? The extension to two-time-scale algorithms is also of interest. 
\end{enumerate}

\bibliographystyle{unsrt}
\bibliography{refs.bib}

\begin{thebibliography}{10}

\bibitem{robbins1951stochastic}
Herbert Robbins and Sutton Monro.
\newblock A stochastic approximation method.
\newblock {\em The Annals of Mathematical Statistics}, pages 400--407, 1951.

\bibitem{borkar}
Vivek~S Borkar.
\newblock {\em Stochastic approximation: a dynamical systems viewpoint}, volume~48.
\newblock Springer, 2009.

\bibitem{borkarode}
Vivek~S Borkar and Sean~P Meyn.
\newblock The ode method for convergence of stochastic approximation and reinforcement learning.
\newblock {\em SIAM Journal on Control and Optimization}, 38(2):447--469, 2000.

\bibitem{meyn2023}
Sean Meyn.
\newblock Stability of {Q}-learning through design and optimism.
\newblock {\em arXiv preprint arXiv:2307.02632}, 2023.

\bibitem{konevcny}
Jakub Kone{\v{c}}n{\`y}, H~Brendan McMahan, Daniel Ramage, and Peter Richt{\'a}rik.
\newblock Federated optimization: Distributed machine learning for on-device intelligence.
\newblock {\em arXiv preprint arXiv:1610.02527}, 2016.

\bibitem{bonawitz}
Keith Bonawitz, Hubert Eichner, Wolfgang Grieskamp, Dzmitry Huba, Alex Ingerman, Vladimir Ivanov, Chloe Kiddon, Jakub Kone{\v{c}}n{\`y}, Stefano Mazzocchi, H~Brendan McMahan, et~al.
\newblock Towards federated learning at scale: System design.
\newblock {\em arXiv preprint arXiv:1902.01046}, 2019.

\bibitem{mcmahan}
Brendan McMahan, Eider Moore, Daniel Ramage, Seth Hampson, and Blaise~Aguera y~Arcas.
\newblock Communication-efficient learning of deep networks from decentralized data.
\newblock In {\em Artificial Intelligence and Statistics}, pages 1273--1282. PMLR, 2017.

\bibitem{li}
Xiang Li, Kaixuan Huang, Wenhao Yang, Shusen Wang, and Zhihua Zhang.
\newblock On the convergence of fedavg on non-iid data.
\newblock {\em arXiv preprint arXiv:1907.02189}, 2019.

\bibitem{khaled2}
Ahmed Khaled, Konstantin Mishchenko, and Peter Richt{\'a}rik.
\newblock Tighter theory for local {SGD} on identical and heterogeneous data.
\newblock In {\em International Conference on Artificial Intelligence and Statistics}, pages 4519--4529. PMLR, 2020.

\bibitem{scaffold}
Sai~Praneeth Karimireddy, Satyen Kale, Mehryar Mohri, Sashank Reddi, Sebastian Stich, and Ananda~Theertha Suresh.
\newblock Scaffold: Stochastic controlled averaging for federated learning.
\newblock In {\em International Conference on Machine Learning}, pages 5132--5143. PMLR, 2020.

\bibitem{fedlin}
Aritra Mitra, Rayana Jaafar, George~J Pappas, and Hamed Hassani.
\newblock Linear convergence in federated learning: Tackling client heterogeneity and sparse gradients.
\newblock {\em Advances in Neural Information Processing Systems}, 34:14606--14619, 2021.

\bibitem{gorbunov}
Eduard Gorbunov, Filip Hanzely, and Peter Richt{\'a}rik.
\newblock Local sgd: Unified theory and new efficient methods.
\newblock In {\em International Conference on Artificial Intelligence and Statistics}, pages 3556--3564. PMLR, 2021.

\bibitem{mishchenko2022proxskip}
Konstantin Mishchenko, Grigory Malinovsky, Sebastian Stich, and Peter Richt{\'a}rik.
\newblock Proxskip: Yes! local gradient steps provably lead to communication acceleration! finally!
\newblock In {\em International Conference on Machine Learning}, pages 15750--15769. PMLR, 2022.

\bibitem{sun2018markov}
Tao Sun, Yuejiao Sun, and Wotao Yin.
\newblock On markov chain gradient descent.
\newblock In {\em Advances in Neural Information Processing Systems}, volume~31, 2018.

\bibitem{duchi}
John~C Duchi, Alekh Agarwal, Mikael Johansson, and Michael~I Jordan.
\newblock Ergodic mirror descent.
\newblock {\em SIAM Journal on Optimization}, 22(4):1549--1578, 2012.

\bibitem{doanMkv}
Thinh~T Doan.
\newblock Finite-time analysis of markov gradient descent.
\newblock {\em IEEE Transactions on Automatic Control}, 2022.

\bibitem{even2023stochastic}
Mathieu Even.
\newblock Stochastic gradient descent under markovian sampling schemes.
\newblock In {\em International Conference on Machine Learning}, pages 9412--9439. PMLR, 2023.

\bibitem{beznosikov2024first}
Aleksandr Beznosikov, Sergey Samsonov, Marina Sheshukova, Alexander Gasnikov, Alexey Naumov, and Eric Moulines.
\newblock First order methods with markovian noise: from acceleration to variational inequalities.
\newblock In {\em Advances in Neural Information Processing Systems}, volume~36, 2024.

\bibitem{malinovsky}
Grigory Malinovskiy, Dmitry Kovalev, Elnur Gasanov, Laurent Condat, and Peter Richtarik.
\newblock From local {SGD} to local fixed-point methods for federated learning.
\newblock In {\em International Conference on Machine Learning}, pages 6692--6701. PMLR, 2020.

\bibitem{bhandari_finite}
Jalaj Bhandari, Daniel Russo, and Raghav Singal.
\newblock A finite time analysis of temporal difference learning with linear function approximation.
\newblock In {\em Conference on learning theory}, pages 1691--1692. PMLR, 2018.

\bibitem{srikant}
Rayadurgam Srikant and Lei Ying.
\newblock Finite-time error bounds for linear stochastic approximation and {TD} learning.
\newblock In {\em Conference on Learning Theory}, pages 2803--2830. PMLR, 2019.

\bibitem{khamaruTD}
Koulik Khamaru, Ashwin Pananjady, Feng Ruan, Martin~J Wainwright, and Michael~I Jordan.
\newblock Is temporal difference learning optimal? an instance-dependent analysis.
\newblock {\em arXiv preprint arXiv:2003.07337}, 2020.

\bibitem{Waiwright}
Martin~J Wainwright.
\newblock Stochastic approximation with cone-contractive operators: Sharp $\ell_{\infty}$-bounds for {Q} -learning.
\newblock {\em arXiv preprint arXiv:1905.06265}, 2019.

\bibitem{Qu}
Adam~Wierman Guannan~Qu.
\newblock Finite-time analysis of asynchronous stochastic approximation and {Q}-learning.
\newblock In {\em Proceedings of Machine Learning Research}, volume 125, pages 1--21. Kluwer Academic Publisher, 2020.

\bibitem{li2024q}
Gen Li, Changxiao Cai, Yuxin Chen, Yuting Wei, and Yuejie Chi.
\newblock Is {Q}-learning minimax optimal? a tight sample complexity analysis.
\newblock {\em Operations Research}, 72(1):222--236, 2024.

\bibitem{mitra2024simple}
Aritra Mitra.
\newblock A simple finite-time analysis of td learning with linear function approximation.
\newblock {\em IEEE Transactions on Automatic Control}, 2024.

\bibitem{qiFRL}
Jiaju Qi, Qihao Zhou, Lei Lei, and Kan Zheng.
\newblock Federated reinforcement learning: techniques, applications, and open challenges.
\newblock {\em arXiv preprint arXiv:2108.11887}, 2021.

\bibitem{doan}
Thinh Doan, Siva Maguluri, and Justin Romberg.
\newblock Finite-time analysis of distributed {TD} (0) with linear function approximation on multi-agent reinforcement learning.
\newblock In {\em International Conference on Machine Learning}, pages 1626--1635. PMLR, 2019.

\bibitem{liuMARL}
Rui Liu and Alex Olshevsky.
\newblock Distributed {TD} $(0)$ with almost no communication.
\newblock {\em IEEE Control Systems Letters}, 7:2892--2897, 2023.

\bibitem{shen2023}
Han Shen, Kaiqing Zhang, Mingyi Hong, and Tianyi Chen.
\newblock Towards understanding asynchronous advantage actor-critic: Convergence and linear speedup.
\newblock {\em IEEE Transactions on Signal Processing}, 2023.

\bibitem{khodadadian}
Sajad Khodadadian, Pranay Sharma, Gauri Joshi, and Siva~Theja Maguluri.
\newblock Federated reinforcement learning: Linear speedup under markovian sampling.
\newblock In {\em International Conference on Machine Learning}, pages 10997--11057. PMLR, 2022.

\bibitem{woo2023blessing}
Jiin Woo, Gauri Joshi, and Yuejie Chi.
\newblock The blessing of heterogeneity in federated {Q}-learning: Linear speedup and beyond.
\newblock In {\em International Conference on Machine Learning}, pages 37157--37216. PMLR, 2023.

\bibitem{tian2024one}
Haoxing Tian, Ioannis~Ch Paschalidis, and Alex Olshevsky.
\newblock One-shot averaging for distributed {TD} ($\lambda$) under markov sampling.
\newblock {\em IEEE Control Systems Letters}, 2024.

\bibitem{salgia2024sample}
Sudeep Salgia and Yuejie Chi.
\newblock The sample-communication complexity trade-off in federated q-learning.
\newblock In {\em Advances in Neural Information Processing Systems}, 2024.

\bibitem{dal}
Nicol{\`o} Dal~Fabbro, Aritra Mitra, and George~J Pappas.
\newblock Federated {TD} learning over finite-rate erasure channels: Linear speedup under markovian sampling.
\newblock {\em IEEE Control Systems Letters}, 7:2461--2466, 2023.

\bibitem{mitraTDEF}
Aritra Mitra, George~J Pappas, and Hamed Hassani.
\newblock Temporal difference learning with compressed updates: Error-feedback meets reinforcement learning.
\newblock {\em Transactions on Machine Learning Research}, 2024.

\bibitem{beikmohammadi2024}
Ali Beikmohammadi, Sarit Khirirat, and Sindri Magn{\'u}sson.
\newblock Compressed federated reinforcement learning with a generative model.
\newblock In {\em Joint European Conference on Machine Learning and Knowledge Discovery in Databases}, pages 20--37. Springer, 2024.

\bibitem{dal2025finite}
Nicol{\`o} Dal~Fabbro, Aritra Mitra, Robert~W Heath, Luca Schenato, and George~J Pappas.
\newblock Finite-time analysis of over-the-air federated td learning.
\newblock {\em IEEE Transactions on Wireless Communications}, 2025.

\bibitem{jinFRL}
Hao Jin, Yang Peng, Wenhao Yang, Shusen Wang, and Zhihua Zhang.
\newblock Federated reinforcement learning with environment heterogeneity.
\newblock In {\em International Conference on Artificial Intelligence and Statistics}, pages 18--37. PMLR, 2022.

\bibitem{wang2023TMLR}
Han Wang, Aritra Mitra, Hamed Hassani, George~J Pappas, and James Anderson.
\newblock Federated temporal difference learning with linear function approximation under environmental heterogeneity.
\newblock {\em Transactions on Machine Learning Research}, 2024.

\bibitem{zhang2024finite}
Chenyu Zhang, Han Wang, Aritra Mitra, and James Anderson.
\newblock Finite-time analysis of on-policy heterogeneous federated reinforcement learning.
\newblock In {\em International Conference on Learning Representations}, 2024.

\bibitem{tsitsiklisroy}
John~N Tsitsiklis and Benjamin Van~Roy.
\newblock An analysis of temporal-difference learning with function approximation.
\newblock In {\em IEEE Transactions on Automatic Control}, 1997.

\bibitem{chen2022finite}
Zaiwei Chen, Sheng Zhang, Thinh~T Doan, John-Paul Clarke, and Siva~Theja Maguluri.
\newblock Finite-sample analysis of nonlinear stochastic approximation with applications in reinforcement learning.
\newblock {\em Automatica}, 146:110623, 2022.

\bibitem{xie}
Zhijie Xie and Shenghui Song.
\newblock Fedkl: Tackling data heterogeneity in federated reinforcement learning by penalizing kl divergence.
\newblock {\em IEEE Journal on Selected Areas in Communications}, 41(4):1227--1242, 2023.

\bibitem{fan2021fault}
Xiaofeng Fan, Yining Ma, Zhongxiang Dai, Wei Jing, Cheston Tan, and Bryan Kian~Hsiang Low.
\newblock Fault-tolerant federated reinforcement learning with theoretical guarantee.
\newblock In M.~Ranzato, A.~Beygelzimer, Y.~Dauphin, P.S. Liang, and J.~Wortman Vaughan, editors, {\em Advances in Neural Information Processing Systems}, volume~34, pages 1007--1021. Curran Associates, Inc., 2021.

\bibitem{lan2023improved}
Guangchen Lan, Han Wang, James Anderson, Christopher Brinton, and Vaneet Aggarwal.
\newblock Improved communication efficiency in federated natural policy gradient via {ADMM}-based gradient updates.
\newblock {\em arXiv preprint arXiv:2310.19807}, 2023.

\bibitem{wang2024momentum}
Han Wang, Sihong He, Zhili Zhang, Fei Miao, and James Anderson.
\newblock Momentum for the win: Collaborative federated reinforcement learning across heterogeneous environments.
\newblock In {\em International Conference on Machine Learning}. PMLR, 2024.

\bibitem{zhu2024towards}
Feng Zhu, Robert~W Heath, and Aritra Mitra.
\newblock Towards fast rates for federated and multi-task reinforcement learning.
\newblock In {\em 2024 IEEE 63rd Conference on Decision and Control (CDC)}, pages 2658--2663. IEEE, 2024.

\bibitem{mangold2024}
Paul Mangold, Sergey Samsonov, Safwan Labbi, Ilya Levin, Reda Alami, Alexey Naumov, and Eric Moulines.
\newblock Scafflsa: Taming heterogeneity in federated linear stochastic approximation and td learning.
\newblock {\em Advances in Neural Information Processing Systems}, 2024.

\bibitem{zengTAC}
Sihan Zeng, Thinh~T Doan, and Justin Romberg.
\newblock Finite-time convergence rates of decentralized stochastic approximation with applications in multi-agent and multi-task learning.
\newblock {\em IEEE Transactions on Automatic Control}, 2022.

\bibitem{adibi2024stochastic}
Arman Adibi, Nicol{\`o} Dal~Fabbro, Luca Schenato, Sanjeev Kulkarni, H~Vincent Poor, George~J Pappas, Hamed Hassani, and Aritra Mitra.
\newblock Stochastic approximation with delayed updates: Finite-time rates under markovian sampling.
\newblock In {\em International Conference on Artificial Intelligence and Statistics}, pages 2746--2754. PMLR, 2024.

\bibitem{levin2017markov}
David~A Levin and Yuval Peres.
\newblock {\em {Markov chains and mixing times}}, volume 107.
\newblock American Mathematical Soc., 2017.

\bibitem{nagaraj}
Dheeraj Nagaraj, Xian Wu, Guy Bresler, Prateek Jain, and Praneeth Netrapalli.
\newblock Least squares regression with markovian data: Fundamental limits and algorithms.
\newblock In {\em Advances in Neural Information Processing Systems}, volume~33, pages 16666--16676, 2020.

\bibitem{chenQ}
Zaiwei Chen, Sheng Zhang, Thinh~T Doan, Siva~Theja Maguluri, and John-Paul Clarke.
\newblock Performance of {Q}-learning with linear function approximation: Stability and finite-time analysis.
\newblock {\em arXiv preprint arXiv:1905.11425}, page~4, 2019.
\newblock Lemma 3.1.

\bibitem{tu2018least}
Stephen Tu and Benjamin Recht.
\newblock Least-squares temporal difference learning for the linear quadratic regulator.
\newblock In {\em International Conference on Machine Learning}, pages 5005--5014. PMLR, 2018.

\bibitem{doan2023finite}
Thinh~T Doan.
\newblock Finite-time convergence rates of distributed local stochastic approximation.
\newblock {\em Automatica}, 158:111294, 2023.

\end{thebibliography}

\newpage
\appendix
\onecolumn
\section{Basic Results}
\label{app:basic}
In this section, we will compile some known facts and derive certain preliminary results that will play a key role in our subsequent analysis of \texttt{FedHSA}. Before we proceed, for the reader's convenience, we assemble all relevant notation in Table~\ref{tb:notation}.  

\begin{table}[ht]
	\caption{~Notation}
	\label{tb:notation}
	\renewcommand{\arraystretch}{1.2}
	\centering\begin{tabular}{cl}
		\toprule
		Notation                                           & Definition                                                                  \\
		\midrule
		$[N]$                                   & The set of $N$ numbers from $1$ to $N$           \\
		$\mathcal{M}_i$                 & Markov chain of agent $i$                                                   \\
		$\mc{S}, \R^d$                                    & Common finite state space of agents' Markov chains,  and $d$-dimensional real space                              \\
        $s$ & Instance of state\\
		$o_{i,t},o_{i,\ell}^{(t)}$                                          & Observation of agent $i$ at time-step $t$, and observation of agent $i$ in local iteration $\ell$ of round $t$                              \\
		$\mu_i$                                      & stationary distribution of $\cM_i$                                \\
		$d_{TV}$                         & Total variation distance     \\
		$\bar\theta^{(t)},\theta_{i,\ell}^{(t)}$                                     & Global parameter at round $t$ and local parameter of agent $i$ in local iteration $\ell$ of round $t$                                \\

		$G_i,\bar G_i$                         & Noisy and true local operators of agent $i$                                              \\
        $G,\bar G$                         & Noisy and true global operators                                              \\
        $\theta^\star,\theta_i^\star$ & Root of global operator $\bar{G}(\theta)$, and root of local operator $\bar{G}_i(\theta)$ of agent $i$\\ 
		$\sigma_i,\sigma$                          & Parameters capturing the effect of noise, with $\sigma:=\max\{\sigma_1,\cdots,\sigma_{M}\}$                                     \\
		$\rho_i,\rho$ & Parameters capturing the mixing-time property, with $\rho:=\max\{\rho_1,\cdots,\rho_{M}\}$\\
        $L,\mu$ & Lipschitz constant and strong-monotonicity constant\\
        $\eta,\alpha_g,\alpha$ & Local step-size, global step-size, and effective step-size\\
        $\cO,\tilde{\cO}$ & Big-O notation hiding universal constants, and big-O notation hiding poly-logarithmic terms\\
        $\tau,\bar\tau$ & Functions capturing the mixing time and the slowest mixing time defined as $\bar\tau:=\tau(\alpha^2)$\\
		$\mathcal{F}_\ell^{(t)}$                                    & Filtration containing all randomness prior to round $t$ and local iteration $\ell$ across all agents               \\
		\bottomrule
	\end{tabular}
\end{table}

\begin{itemize}
    \item Given $m$ vectors $(x_1, \cdots, x_m)\in\R^d\times\cdots\times\R^d$, the following holds true by a simple application of Jensen's inequality:
    \begin{equation}
        \norm{\sum_{i=1}^mx_i}^2\leq m\sum_{i=1}^m\norm{x_i}^2.
    \end{equation}
    \item Given $m$ vectors $(x_1, \cdots, x_m)\in\R^d\times\cdots\times\R^d$, the following is a generalization of the triangle inequality:
    \begin{equation}
        \norm{\sum_{i=1}^mx_i}\leq\sum_{i=1}^m\norm{x_i}.
    \end{equation}
    \item Given any two vectors $(x,y)\in\R^d\times\R^d$, the following holds for any $\xi>0$:
    \begin{equation}
    \begin{aligned}
        \iprod{x}{y}&\leq\frac{\xi}{2}\norm{x}^2 +\frac{1}{2\xi}\norm{y}^2.\label{eqn::cs}
    \end{aligned}
    \end{equation}
\end{itemize}
For simplicity, we also omit the observation $o_{i,\ell}^{(t)}$ in the parameters of the noisy operator $G_i$ in most cases.

\textbf{Proof of Proposition~\ref{prop:hetbias}}
We first provide the proof of Proposition~\ref{prop:hetbias}. Before we start, recall from Section~\ref{sec::problem_formulation} that the local update formula is
\begin{equation}
\theta_{i,\ell+1}^{(t)} =  \theta_{i,\ell}^{(t)} + \eta \bar{G}_i(\theta_{i,\ell}^{(t)}), \ell = 0, 1, \ldots, H-1,
\end{equation}
where $\bar G_i(\theta)=\bar A_i\theta-\bar b_i$ and $H=2$. Also recall that the aggregation formula at the server is 
\begin{equation}
    \bar{\theta}^{(t+1)} = \frac{1}{M} \sum_{i\in [M]} \theta^{(t)}_{i, H}.\label{eqn::aggr_prop1}
\end{equation}
\begin{proof}
By definition of $\bar G_i$, we can write
\begin{equation}
\begin{aligned}
    \theta_{i,\ell+1}^{(t)} &=  \theta_{i,\ell}^{(t)} + \eta \paren{\bar A_i\theta_{i,\ell}^{(t)}-\bar b_i}\\
    &=\paren{I+\eta\bar A_i}\theta_{i,\ell}^{(t)}-\eta \bar b_i\\
    &= \paren{I+\eta\bar A_i}\paren{\theta_{i,\ell}^{(t)}-\theta^\star}+\theta^\star+\eta \paren{\bar A_i\theta^\star-\bar b_i},
\end{aligned}
\end{equation}
which yields
\begin{equation}
    \theta_{i,\ell+1}^{(t)}-\theta^\star=\paren{I+\eta\bar A_i}\paren{\theta_{i,\ell}^{(t)}-\theta^\star}+\eta \paren{\bar A_i\theta^\star-\bar b_i}.\label{eqn::onestep_prop1}
\end{equation}

Iterating Eq.~\eqref{eqn::onestep_prop1} for $H=2$ steps, we obtain
\begin{equation}
    \theta_{i,2}^{(t)}-\theta^\star=\paren{I+\eta\bar A_i}^2\paren{\theta_{i,0}^{(t)}-\theta^\star}+\eta \paren{I+\paren{I+\eta\bar A_i}}\paren{\bar A_i\theta^\star-\bar b_i}.
\end{equation}

Therefore, by Eq.~\eqref{eqn::aggr_prop1} and the fact that $\theta_{i,0}^{(t)}=\bar\theta^{(t)}$ we have
\begin{equation}
    \bar\theta^{(t+1)}-\theta^\star=\paren{\frac{1}{M}\sum_{i=1}^M\paren{I+\eta\bar A_i}^2}\paren{\bar\theta^{(t)}-\theta^\star}+\eta\paren{\frac{1}{M}\sum_{i=1}^M\paren{2I+\eta\bar A_i}\paren{\bar A_i\theta^\star-\bar b_i}}.\label{eqn::propeq1}
\end{equation}
Using the fact that $\bar A\theta^\star -\bar b=0$ and $\bar b_i=\bar A_i\theta_i^\star$, we have
\begin{equation}
    \frac{1}{M}\sum_{i=1}^M\paren{2I+\eta\bar A_i}\paren{\bar A_i\theta^\star-\bar b_i}=\frac{\eta}{M}\sum_{i=1}^M\paren{\bar {A}_i^2\paren{\theta^\star-\theta_i^\star}}.\label{eqn::propeq2}
\end{equation}
Also, we have
\begin{equation}
\begin{aligned}
    \frac{1}{M}\sum_{i=1}^M\paren{I+\eta\bar A_i}^2&=\frac{1}{M}\sum_{i=1}^M\paren{I+2\eta\bar A_i+\eta^2\bar A_i^2}\\
    &=I+2\eta\bar A+\eta^2\bar A',\label{eqn::propeq3}
\end{aligned}
\end{equation}
where recall that $\bar{A} = (1/M) \sum_{i \in [M]} \bar{A}_i$ and $\bar{A}'=(1/M) \sum_{i \in [M]} \bar{A}^2_i$. 

Now, defining $e_t:=\bar\theta^{(t)}-\theta^\star$, and combining Eq.~\eqref{eqn::propeq1} - Eq.~\eqref{eqn::propeq3}, we obtain
\begin{equation}
    e_{t+1}=Fe_t+\bar v, \quad\forall t\geq 0,\label{eqn::LS}
\end{equation}
where $F=I+2\eta\bar A+\eta^2\bar A'$ and $\bar v=({\eta^2}/{M})\sum_{i=1}^M\paren{\bar {A}_i^2\paren{\theta^\star-\theta_i^\star}}$. We then conclude that Eq.~\eqref{eqn::LS} represents a linear dynamical system in discrete time with state transition matrix $F$. 

For stability, we need $F$ to be Schur-stable. Suppose $\eta$ is chosen to ensure that $F$ is Schur-stable. Now, applying Eq.\eqref{eqn::LS} iteratively yields
\begin{equation}
    e_t=F^te_0+\paren{\sum_{k=0}^{t-1}F^k}\bar v.
\end{equation}
Taking limits on both sides and using the fact that $F$ is Schur-stable, we obtain
\begin{equation}
    \lim_{t\to\infty} F^t e_0 = 0, \quad \lim_{t\to\infty}\paren{\sum_{k=0}^{t-1}F^k}=\paren{I-F}^{-1}=\frac{-1}{\eta}\paren{2\bar A+\eta\bar A'}^{-1}.
\end{equation}
Therefore, 
\begin{equation}
    \lim_{t\to\infty}e_t=\frac{\eta}{M}\paren{2\bar A+\eta\bar A'}^{-1}\sum_{i=1}^M\paren{\bar {A}_i^2\paren{\theta_i^\star-\theta^\star}}.
\end{equation}
We have thus proved that the error $e_t$ is non-vanishing.
\end{proof}

The following corollary will come in handy at several points in our analysis. 
\begin{corollary}\label{corr:diff}
    Given Assumption~\ref{ass::smoothness}, for any given $\theta,o$, the following holds for all $i\in[M]$:
    \begin{equation}
    \begin{aligned}
        \norm{G_i(\theta, o)-\bar G_i(\theta)}\leq \bigo{L}\paren{\norm{\theta-\theta^\star}+\sigma}\\
        \norm{G_i(\theta, o)}\leq \bigo{L}\paren{\norm{\theta-\theta^\star}+\sigma}\\
        \norm{\bar G_i(\theta)}\leq \bigo{L}\paren{\norm{\theta-\theta^\star}+\sigma}
    \end{aligned}
    \end{equation}
\end{corollary}
\begin{proof}
\begin{equation}
\begin{aligned}
    \norm{G_i(\theta, o)-\bar G_i(\theta)}&\leq\norm{G_i(\theta, o)}+\norm{\bar G_i(\theta)}\\
    &\leq 2L\paren{\norm{\theta}+\sigma}\\
    &=2L\paren{\norm{\theta-\theta^\star}+\norm{\theta^\star}+\sigma}\\
    &\leq 2L\paren{\norm{\theta-\theta^\star}+2\sigma}.
\end{aligned}
\end{equation}
Here, the second inequality follows from~\eqref{eqn::uniform_bound}, and the last inequality uses the definition of $\sigma$. The rest two bounds can be obtained similarly.
\end{proof}

We also state upfront the following lemma that bounds the drift term $\norm{\theta_{i,\ell}^{(t)}-\bar \theta^{(t)}}^2$. Notably, this result holds deterministically, i.e., it applies to both i.i.d. and Markovian sampling.  
\begin{lemma}\label{lem::drift}
    Suppose Assumptions~\ref{ass::smoothness} and~\ref{ass::irreducible} hold. By selecting $\eta\leq 1/(LH)$, the following is true for \texttt{FedHSA}:
    \begin{equation}
        \norm{\theta_{i,\ell}^{(t)}-\bar \theta^{(t)}}^2\leq \cO(\eta^2 L^2H^2)\left(\norm{\bar\theta^{(t)}-\theta^\star}^2+\sigma^2\right).
    \end{equation}
\end{lemma}
\textbf{Proof of Lemma~\ref{lem::drift}.}
\begin{proof}
A high-level intuition for proving this lemma is to use the update rule of \texttt{FedHSA} in tandem with Corollary~\ref{corr:diff} to obtain a recursion for the term $\norm{\theta_{i,\ell}^{(t)}-\bar \theta^{(t)}}$. With that aim in mind, from~\eqref{eqn::FedHSA_update} we can write
\begin{equation}
\begin{aligned}
    \norm{\theta_{i,\ell+1}^{(t)}-\bar \theta^{(t)}}&=\norm{\theta_{i,\ell}^{(t)}-\bar \theta^{(t)}+\eta\left(G_i(\theta_{i,\ell}^{(t)})+ G(\bar\theta^{(t)})- G_i(\bar\theta^{(t)})\right)}\\
    &=\norm{\theta_{i,\ell}^{(t)}-\bar \theta^{(t)}+\eta\left(G_i(\theta_{i,\ell}^{(t)},o_{i,\ell}^{(t)})- G_i(\bar\theta^{(t)},o_{i,\ell}^{(t)})\right)+ \eta G(\bar\theta^{(t)})+\eta G_i(\bar\theta^{(t)},o_{i,\ell}^{(t)})-\eta G_i(\bar\theta^{(t)})}\\
    &\leq \norm{\theta_{i,\ell}^{(t)}-\bar \theta^{(t)}}+\norm{\eta\left(G_i(\theta_{i,\ell}^{(t)},o_{i,\ell}^{(t)})- G_i(\bar\theta^{(t)},o_{i,\ell}^{(t)})\right)}+ \norm{\eta G(\bar\theta^{(t)})+\eta G_i(\bar\theta^{(t)},o_{i,\ell}^{(t)})-\eta G_i(\bar\theta^{(t)})}\\
    &\leq \norm{\theta_{i,\ell}^{(t)}-\bar \theta^{(t)}}+\eta L\norm{\theta_{i,\ell}^{(t)}-\bar \theta^{(t)}}+ \norm{\eta G(\bar\theta^{(t)})}+\norm{\eta G_i(\bar\theta^{(t)},o_{i,\ell}^{(t)})}+\norm{\eta G_i(\bar\theta^{(t)})}\\
    &\leq (1+\eta L)\norm{\theta_{i,\ell}^{(t)}-\bar \theta^{(t)}}+\cO(\eta L)\left(\norm{\bar\theta^{(t)}-\theta^\star}+\sigma\right),\label{bound::drift_iteration}
\end{aligned}
\end{equation}
where the in the second last inequality we used Assumption~\ref{ass::smoothness}, and the last inequality follows from Corollary~\ref{corr:diff}.

Applying~\eqref{bound::drift_iteration} iteratively, we obtain
\begin{equation}
\begin{aligned}
    \norm{\theta_{i,\ell}^{(t)}-\bar \theta^{(t)}}&\leq(1+\eta L)^\ell\norm{\theta_{i,0}^{(t)}-\bar \theta^{(t)}}+\sum_{j=0}^{\ell-1}(1+\eta L)^j\cO(\eta L)\left(\norm{\bar\theta^{(t)}-\theta^\star}+\sigma\right)\\
    &\leq H(1+\eta L)^H\cO(\eta L)\left(\norm{\bar\theta^{(t)}-\theta^\star}+\sigma\right)\\
    &\leq \cO(\eta LH)\left(\norm{\bar\theta^{(t)}-\theta^\star}+\sigma\right). \label{eqn::drift_last}
\end{aligned}
\end{equation}
Here, the second inequality holds because $\theta_{i,0}^{(t)}=\bar \theta^{(t)}$ and $\ell\leq H$. The last one follows by choosing $\eta\leq 1/(LH)$, and thus $H(1+\eta L)^H\cO(\eta L)\leq (1+1/H)^H\bigo{\eta LH}\leq \bigo{\eta LH}$, where in the last inequality we used $(1+1/x)^x\leq e, \forall x>0$. Squaring both sides of~\eqref{eqn::drift_last} yields the final form of Lemma~\ref{lem::drift}.
\end{proof}
Lemma~\ref{lem::drift} states that if we select $\alpha_g=1$ and $\alpha=H\alpha_g\eta$, we can then bound the drift term $\norm{\theta_{i,\ell}^{(t)}-\bar \theta^{(t)}}^2$ in the form of the squared norm of the distance of the global iterate to the global root plus the effect of noise, damped by the factor $\alpha^2$. 

\newpage

\section{Analysis of \texttt{FedHSA} under I.I.D. Sampling}
\label{app:iidproof}
We begin our analysis of the \texttt{FedHSA} algorithm's convergence by examining its performance in a simplified i.i.d. scenario. Specifically, we assume that for each agent $i \in [M]$, its observation $o_{i,t}$ at time step $t$ is drawn in an i.i.d. manner from its stationary distribution $\mu_i$. Our analysis of the i.i.d. setting will provide a foundational understanding of the finite-time behavior of \texttt{FedHSA}. In turn, it will offer a smoother transition to the more complex case of Markovian sampling.

We start by introducing the following lemma, which provides a recursion for the progress towards $\theta^\star$ made by \texttt{FedHSA} in each communication round. 
 
\begin{lemma}\label{lem::onestep_iid}
    Suppose Assumption~\ref{ass::smoothness} -~\ref{ass::irreducible} and~\ref{ass::independence} hold. Also, suppose the data samples of each agent $i\in[M]$ are drawn i.i.d. from the stationary distribution $\mu_i$ of its underlying Markov chain $\cM_i$. Then,  the following holds for \texttt{FedHSA} $\forall t \geq 0$:
    \begin{equation}
    \begin{aligned}
        \E\bracket{\norm{\bar\theta^{(t+1)}-\theta^\star}^2}&\leq \paren{1-\alpha\mu+\bigo{\alpha^2L^2}}\E\left[\norm{\bar\theta^{(t)}-\theta^\star}^2\right]+\bigo{\frac{\alpha^2L^2\sigma^2}{MH}}\\
        &\quad+\bigo{\frac{\alpha L^2}{\mu MH}+\frac{\alpha^2L^2}{MH}}\sum_{i=1}^M\sum_{\ell=0}^{H-1}\E\bracket{\norm{\theta_{i,\ell}^{(t)}-\bar \theta^{(t)}}^2}.
    \end{aligned}
    \end{equation}
\end{lemma}
\begin{proof}
From the update rule of \texttt{FedHSA} in Eq.~\eqref{eqn::simple_average} and the definition of $G(\bar\theta^{(t)})$, we obtain
\begin{equation}
\begin{aligned}
    \bar{\theta}^{(t+1)}-\Bar{\theta}^{(t)} & = \frac{\alpha_g\eta}{M}\sum_{i=1}^M\sum_{\ell=0}^{H-1}\paren{G_i(\theta_{i,\ell}^{(t)})-G_i(\Bar{\theta}^{(t)})+G(\Bar{\theta}^{(t)})}\\
    & =\frac{\alpha_g\eta}{M}\sum_{i=1}^M\sum_{\ell=0}^{H-1} G_i(\theta_{i,\ell}^{(t)}). 
\end{aligned}
\end{equation}
Using the definition of the effective step-size $\alpha=H\eta\alpha_g$, we can then write the squared norm of the distance of the global iterate to the optimum $\theta^*$ as
\begin{equation}
\begin{aligned}
    \norm{\bar\theta^{(t+1)}-\theta^\star}^2
   &=\norm{\bar\theta^{(t)}-\theta^\star+\frac{\alpha}{MH}\sum_{i=1}^M\sum_{\ell=0}^{H-1} G_i(\theta_{i,\ell}^{(t)})}^2\\
    &=\norm{\bar\theta^{(t)}-\theta^\star}^2+\overbrace{\left\langle\bar\theta^{(t)}-\theta^\star, \frac{2\alpha}{MH}\sum_{i=1}^M\sum_{\ell=0}^{H-1} G_i(\theta_{i,\ell}^{(t)})\right\rangle}^{T_1}+\overbrace{\norm{\frac{\alpha}{MH}\sum_{i=1}^M\sum_{\ell=0}^{H-1} G_i(\theta_{i,\ell}^{(t)})}^2}^{T_2}.\label{bound::onestep_iid}
\end{aligned}
\end{equation}
To bound this equation in expectation, it then boils down to bounding the expectations of the terms $T_1$ and $T_2$ separately. We proceed to do this next. 
\begin{equation}
\begin{aligned}
    \E[T_1]&\overset{(a)}{=}2\alpha\E\left[\left\langle\bar\theta^{(t)}-\theta^\star, \frac{1}{MH}\sum_{i=1}^M\sum_{\ell=0}^{H-1} \bar G_{i}(\theta_{i,\ell}^{(t)})\right\rangle\right]\\
    &=2\alpha\E\left[\left\langle\bar\theta^{(t)}-\theta^\star, \frac{1}{MH}\sum_{i=1}^M\sum_{\ell=0}^{H-1} \left(\bar G_{i}(\theta_{i,\ell}^{(t)})-\bar G_{i}(\bar \theta^{(t)})\right)\right\rangle\right]+2\alpha\E\left[\left\langle\bar\theta^{(t)}-\theta^\star, \frac{1}{MH}\sum_{i=1}^M\sum_{\ell=0}^{H-1} \bar G_{i}(\bar\theta^{(t)})\right\rangle\right]\\
    &=2\alpha\E\left[\left\langle\bar\theta^{(t)}-\theta^\star, \frac{1}{MH}\sum_{i=1}^M\sum_{\ell=0}^{H-1} \left(\bar G_{i}(\theta_{i,\ell}^{(t)})-\bar G_{i}(\bar \theta^{(t)})\right)\right\rangle\right]+2\alpha \E\left[\left\langle\bar\theta^{(t)}-\theta^\star, \bar G(\bar\theta^{(t)})\right\rangle\right]\\
    &\overset{(b)}{\leq}\frac{2\alpha}{MH}\sum_{i=1}^M\sum_{\ell=0}^{H-1}\E\left[\left\langle\bar\theta^{(t)}-\theta^\star,  \bar G_{i}(\theta_{i,\ell}^{(t)})-\bar G_{i}(\bar \theta^{(t)})\right\rangle\right]-2\alpha\mu\E\left[\norm{\bar \theta^{(t)}-\theta^\star}^2\right]\\
    &\leq\frac{\alpha}{MH}\sum_{i=1}^M\sum_{\ell=0}^{H-1}\E\left[\mu\norm{\bar \theta^{(t)}-\theta^\star}^2+\frac{1}{\mu}\norm{\bar G_{i}(\theta_{i,\ell}^{(t)})-\bar G_{i}(\bar \theta^{(t)})}^2\right]-2\alpha\mu\E\left[\norm{\bar \theta^{(t)}-\theta^\star}^2\right]\\
    &\overset{(c)}{\leq}\frac{\alpha}{MH}\sum_{i=1}^M\sum_{\ell=0}^{H-1}\E\left[\mu\norm{\bar \theta^{(t)}-\theta^\star}^2+\frac{L^2}{\mu}\norm{\theta_{i,\ell}^{(t)}-\bar \theta^{(t)}}^2\right]-2\alpha\mu\E\left[\norm{\bar \theta^{(t)}-\theta^\star}^2\right]\\
    &=-\alpha\mu\E\left[\norm{\bar \theta^{(t)}-\theta^\star}^2\right]+\frac{\alpha L^2}{\mu MH}\sum_{i=1}^M\sum_{\ell=0}^{H-1}\E\bracket{\norm{\theta_{i,\ell}^{(t)}-\bar \theta^{(t)}}^2},\label{bound::T1_iid}
\end{aligned}
\end{equation}
where $(b)$ follows from Assumption~\ref{ass::strong_monotonicity}; $(c)$ is a result of Assumption~\ref{ass::smoothness}. Before we proceed, we define $\cF_\ell^{(t)}$ as the $\sigma$-algebra capturing all the randomness up to the $\ell$-th local iteration of round $t$. Building on this definition, we reason about $(a)$ separately as follows:
\begin{equation}\label{eqn::T1_iid}
\begin{aligned}
    \E[T_1]&=2\alpha\E\left[\left\langle\bar\theta^{(t)}-\theta^\star, \frac{1}{MH}\sum_{i=1}^M\sum_{\ell=0}^{H-1} G_i(\theta_{i,\ell}^{(t)})\right\rangle\right]\\
    &=2\alpha\E\left[\left\langle\bar\theta^{(t)}-\theta^\star, \frac{1}{MH}\sum_{i=1}^M\sum_{\ell=0}^{H-2} G_i(\theta_{i,\ell}^{(t)})\right\rangle\right]+2\alpha\E\left[\left\langle\bar\theta^{(t)}-\theta^\star, \frac{1}{MH}\sum_{i=1}^M G_i(\theta_{i,H-1}^{(t)})\right\rangle\right]\\
    &\overset{(a1)}{=}2\alpha\E\left[\left\langle\bar\theta^{(t)}-\theta^\star, \frac{1}{MH}\sum_{i=1}^M\sum_{\ell=0}^{H-2} G_i(\theta_{i,\ell}^{(t)})\right\rangle\right]+2\alpha\E\left[\E\left[\left\langle\bar\theta^{(t)}-\theta^\star, \frac{1}{MH}\sum_{i=1}^M G_i(\theta_{i,H-1}^{(t)})\right\rangle\Bigg|\cF_{H-2}^{(t)}\right]\right]\\
    &\overset{(a2)}{=}2\alpha\E\left[\left\langle\bar\theta^{(t)}-\theta^\star, \frac{1}{MH}\sum_{i=1}^M\sum_{\ell=0}^{H-2} G_i(\theta_{i,\ell}^{(t)})\right\rangle\right]+2\alpha\E\left[\left\langle\bar\theta^{(t)}-\theta^\star, \frac{1}{MH}\sum_{i=1}^M \E\left[G_i(\theta_{i,H-1}^{(t)})\Big|\cF_{H-2}^{(t)}\right]\right\rangle\right]\\
    &\overset{(a3)}{=}2\alpha\E\left[\left\langle\bar\theta^{(t)}-\theta^\star, \frac{1}{MH}\sum_{i=1}^M\sum_{\ell=0}^{H-2} G_i(\theta_{i,\ell}^{(t)})\right\rangle\right]+2\alpha\left\langle\bar\theta^{(t)}-\theta^\star, \frac{1}{MH}\sum_{i=1}^M \bar G_{i}(\theta_{i,H-1}^{(t)})\right\rangle.
\end{aligned}
\end{equation}
Here, $(a1)$ follows from the tower property of expectations; $(a2)$ holds from the fact that $\bar\theta^{(t)}$ is deterministic conditioned on $\cF_{H-2}^{(t)}$, and $(a3)$ is a result of the i.i.d. assumption in tandem with the fact that $\theta_{i,H-1}^{(t)}$ is deterministic conditioned on $\cF_{H-2}^{(t)}$. We can repeat this process by iteratively conditioning on $\cF_{H-3}^{(t)},...,\cF_{-1}^{(t)}$ and arrive at $(a)$, where $\cF_{-1}^{(t)}$ is all the randomness up to round $t$ before any local steps are taken.

We now move on to bound the expectation of $T_2$.
\begin{equation}
\begin{aligned}
    \E[T_2]&=\alpha^2\E\left[\norm{\frac{1}{MH}\sum_{i=1}^M\sum_{\ell=0}^{H-1} \left(G_i(\theta_{i,\ell}^{(t)})-\bar G_i(\theta_{i,\ell}^{(t)})\right)+\frac{1}{MH}\sum_{i=1}^M\sum_{\ell=0}^{H-1} \bar G_i(\theta_{i,\ell}^{(t)})}^2\right]\\
    &\leq2\alpha^2\E\left[\norm{\frac{1}{MH}\sum_{i=1}^M\sum_{\ell=0}^{H-1} \left(G_i(\theta_{i,\ell}^{(t)})-\bar G_i(\theta_{i,\ell}^{(t)})\right)}^2\right]+2\alpha^2\E\left[\norm{\frac{1}{MH}\sum_{i=1}^M\sum_{\ell=0}^{H-1} \bar G_i(\theta_{i,\ell}^{(t)})}^2\right]\\
    &\overset{(a)}{\leq}2\alpha^2\E\left[\frac{1}{M^2H^2}\sum_{i=1}^M\sum_{\ell=0}^{H-1} \norm{G_i(\theta_{i,\ell}^{(t)})-\bar G_i(\theta_{i,\ell}^{(t)})}^2\right]\\
    &\quad + 2\alpha^2\E\bracket{\norm{\frac{1}{MH}\sum_{i=1}^M\sum_{\ell=0}^{H-1}\left(\bar G_i(\theta_{i,\ell}^{(t)})-\bar G_i(\bar\theta^{(t)})\right)+\frac{1}{MH}\sum_{i=1}^M\sum_{\ell=0}^{H-1}\bar G_i(\bar\theta^{(t)})}^2}\\
    &\overset{(b)}{\leq}2\alpha^2\E\left[\frac{1}{M^2H^2}\sum_{i=1}^M\sum_{\ell=0}^{H-1} \cO\left(L^2\right)\paren{\norm{\theta_{i,\ell}^{(t)}-\theta^\star}^2+\sigma^2}\right]+ 4\alpha^2\E\bracket{\norm{\frac{1}{MH}\sum_{i=1}^M\sum_{\ell=0}^{H-1}\left(\bar G_i(\theta_{i,\ell}^{(t)})-\bar G_i(\bar\theta^{(t)})\right)}^2}\\
    &\quad+ 4\alpha^2\E\bracket{\norm{\frac{1}{MH}\sum_{i=1}^M\sum_{\ell=0}^{H-1}\bar G_i(\bar\theta^{(t)})}^2}\\
    &\leq2\alpha^2\E\left[\frac{1}{M^2H^2}\sum_{i=1}^M\sum_{\ell=0}^{H-1} \cO(L^2)\left(\norm{\theta_{i,\ell}^{(t)}-\bar\theta^{(t)}}^2+\norm{\bar\theta^{(t)}-\theta^\star}^2+\sigma^2\right)\right]\\
    &\quad +\frac{4\alpha^2}{MH}\sum_{i=1}^M\sum_{\ell=0}^{H-1}\E\bracket{\norm{\bar G_i(\theta_{i,\ell}^{(t)})-\bar G_i(\bar\theta^{(t)})}^2} +4\alpha^2\E\bracket{\norm{\bar G(\bar\theta^{(t)})}^2}\\
    &\overset{(c)}{\leq}\frac{2\alpha^2}{M^2H^2}\sum_{i=1}^M\sum_{\ell=0}^{H-1}\cO(L^2)\E\left[ \left(\norm{\theta_{i,\ell}^{(t)}-\bar\theta^{(t)}}^2+\norm{\bar\theta^{(t)}-\theta^\star}^2+\sigma^2\right)\right]+\frac{4\alpha^2L^2}{MH}\sum_{i=1}^M\sum_{\ell=0}^{H-1}\E\left[\norm{\theta_{i,\ell}^{(t)}-\bar\theta^{(t)}}^2\right]\\
    &\quad+4\alpha^2L^2\E\left[\norm{\bar\theta^{(t)}-\theta^\star}^2\right]\\
    &\leq \bigo{\frac{\alpha^2L^2}{M^2H^2}+\frac{\alpha^2L^2}{MH}}\sum_{i=1}^M\sum_{\ell=0}^{H-1}\E\left[\norm{\theta_{i,\ell}^{(t)}-\bar\theta^{(t)}}^2\right]+\bigo{\alpha^2L^2+\frac{\alpha^2L^2}{MH}}\E\bracket{\norm{\bar\theta^{(t)}-\theta^\star}^2}+\bigo{\frac{\alpha^2L^2\sigma^2}{MH}}\\
    &\leq \bigo{\frac{\alpha^2L^2}{MH}}\sum_{i=1}^M\sum_{\ell=0}^{H-1}\E\left[\norm{\theta_{i,\ell}^{(t)}-\bar\theta^{(t)}}^2\right]+\bigo{\alpha^2L^2}\E\bracket{\norm{\bar\theta^{(t)}-\theta^\star}^2}+\bigo{\frac{\alpha^2L^2\sigma^2}{MH}},\label{bound::T2_iid}
\end{aligned}
\end{equation}
where $(b)$ uses Corollary~\ref{corr:diff}, and $(c)$ is a result of Assumption~\ref{ass::smoothness}. We provide the rationale behind $(a)$ as follows.

The second term in inequality $(a)$ is a direct result of add-and-subtract. To inspect the first term, note that
\begin{equation}
\begin{aligned}
    &\quad\,\,\E\left[\norm{\sum_{i=1}^M\sum_{\ell=0}^{H-1} \paren{G_i(\theta_{i,\ell}^{(t)})-\bar G_i(\theta_{i,\ell}^{(t)})}}^2\right]\\
    &=\E\left[\sum_{i=1}^M\sum_{\ell=0}^{H-1} \norm{G_i(\theta_{i,\ell}^{(t)})-\bar G_i(\theta_{i,\ell}^{(t)})}^2\right]+2\sum_{j=1}^M\sum_{m< n}\E\left[\left\langle G_j(\theta_{j,m}^{(t)})-\bar G_j(\theta_{j,m}^{(t)}),G_j(\theta_{j,n}^{(t)})-\bar G_j(\theta_{j,n}^{(t)})\right\rangle\right]\\
    &+2\sum_{j< k}\sum_{m=0}^{H-1}\E\left[\left\langle G_j(\theta_{j,m}^{(t)})-\bar G_j(\theta_{j,m}^{(t)}),G_k(\theta_{k,m}^{(t)})-\bar G_k(\theta_{k,m}^{(t)})\right\rangle\right]\\
    &+2\sum_{j< k}\sum_{m<n}\E\left[\left\langle G_j(\theta_{j,m}^{(t)})-\bar G_j(\theta_{j,m}^{(t)}),G_k(\theta_{k,n}^{(t)})-\bar G_k(\theta_{k,n}^{(t)})\right\rangle\right]\\
    &=\E\left[\sum_{i=1}^M\sum_{\ell=0}^{H-1} \norm{G_i(\theta_{i,\ell}^{(t)})-\bar G_i(\theta_{i,\ell}^{(t)})}^2\right]+2\sum_{j=1}^M\sum_{m< n}\E\left[\E\left[\left\langle G_j(\theta_{j,m}^{(t)})-\bar G_j(\theta_{j,m}^{(t)}),G_j(\theta_{j,n}^{(t)})-\bar G_j(\theta_{j,n}^{(t)})\right\rangle\Big|\cF_{n-1}^{(t)}\right]\right]\\
    &+2\sum_{j< k}\sum_{m=0}^{H-1}\E\left[\E\left[\left\langle G_j(\theta_{j,m}^{(t)})-\bar G_j(\theta_{j,m}^{(t)}),G_k(\theta_{k,m}^{(t)})-\bar G_k(\theta_{k,m}^{(t)})\right\rangle\Big|\cF_{m-1}^{(t)}\right]\right]\\
    &+2\sum_{j< k}\sum_{m<n}\E\left[\E\left[\left\langle G_j(\theta_{j,m}^{(t)})-\bar G_j(\theta_{j,m}^{(t)}),G_k(\theta_{k,n}^{(t)})-\bar G_k(\theta_{k,n}^{(t)})\right\rangle\Big|\cF_{n-1}^{(t)}\right]\right]\\
    &=\E\left[\sum_{i=1}^M\sum_{\ell=0}^{H-1} \norm{G_i(\theta_{i,\ell}^{(t)})-\bar G_i(\theta_{i,\ell}^{(t)})}^2\right],
\end{aligned}
\end{equation}
where the first equality holds from unrolling the squared norm, and the second equality follows from the tower property of expectation. We analyze the terms separately in the last equality.

For the first cross-term, we can write
\begin{equation}
\begin{aligned}
&\quad2\sum_{j=1}^M\sum_{m< n}\E\left[\E\left[\left\langle G_j(\theta_{j,m}^{(t)})-\bar G_j(\theta_{j,m}^{(t)}),G_j(\theta_{j,n}^{(t)})-\bar G_j(\theta_{j,n}^{(t)})\right\rangle\Big|\cF_{n-1}^{(t)}\right]\right]\\
&=2\sum_{j=1}^M\sum_{m< n}\E\left[\left\langle G_j(\theta_{j,m}^{(t)})-\bar G_j(\theta_{j,m}^{(t)}),\E\left[G_j(\theta_{j,n}^{(t)})-\bar G_j(\theta_{j,n}^{(t)})\Big|\cF_{n-1}^{(t)}\right]\right\rangle\right]\\
&=0.
\end{aligned}
\end{equation}
The first inequality holds from the definition of $\cF_{n-1}^{(t)}$. The second one holds because $\theta_{j,n}^{(t)}$ is deterministic conditioned on $\cF_{n-1}^{(t)}$, and the only randomness comes from $o_{j,n}^{(t)}$ in $G_j(\theta_{j,n}^{(t)})$. Due to the i.i.d. sampling assumption, we readily know that $\E\bracket{G_j(\theta_{j,n}^{(t)})}=\bar G_j(\theta_{j,n}^{(t)})$ for a fixed $\theta_{j,n}^{(t)}$, and the cross-term then equals 0.

For the second cross-term, we can write
\begin{equation}
\begin{aligned}
&\quad2\sum_{j< k}\sum_{m=0}^{H-1}\E\left[\E\left[\left\langle G_j(\theta_{j,m}^{(t)})-\bar G_j(\theta_{j,m}^{(t)}),G_k(\theta_{k,m}^{(t)})-\bar G_k(\theta_{k,m}^{(t)})\right\rangle\Big|\cF_{m-1}^{(t)}\right]\right]\\
&\overset{(i)}{=}2\sum_{j< k}\sum_{m=0}^{H-1}\E\left[\left\langle\E\left[e_{j,m}^{(t)}\big|\cF_{m-1}^{(t)}\right],\E\bracket{e_{k,m}^{(t)}\big|\cF_{m-1}^{(t)]}}\right\rangle\right]\\
&\overset{(ii)}{=}0,
\end{aligned}
\end{equation}
where we define $e_{i,\ell}^{(t)}:=G_i(\theta_{i,\ell}^{(t)})-\bar G_i(\theta_{i,\ell}^{(t)})$. For $(i)$, we used the fact that $\theta_{j,m}^{(t)}$ and $\theta_{k,m}^{(t)}$ are deterministic conditioned on $\cF_{m-1}^{(t)}$. So, due to Assumption~\ref{ass::independence}, $e_{j,m}^{(t)}$ and $e_{k,m}^{(t)}$ are conditionally independent. For $(ii)$, we used the i.i.d. sampling assumption.

The third cross-term can be proved to be 0 similarly as for the second one.

Taking expectation on both sides of~\eqref{bound::onestep_iid}, and plugging in the bounds from~\eqref{bound::T1_iid} and~\eqref{bound::T2_iid}, we obtain:  
\begin{equation}
\begin{aligned}
    \E\bracket{\norm{\bar\theta^{(t+1)}-\theta^\star}^2}&\leq \paren{1-\alpha\mu+\bigo{\alpha^2L^2}}\E\left[\norm{\bar\theta^{(t)}-\theta^\star}^2\right]\\
    &\quad+\overbrace{\paren{\frac{\alpha L^2}{\mu MH}+\frac{\alpha^2L^2}{MH}}\sum_{i=1}^M\sum_{\ell=0}^{H-1}\E\bracket{\norm{\theta_{i,\ell}^{(t)}-\bar \theta^{(t)}}^2}}^{DRIFT} +\bigo{\frac{\alpha^2L^2\sigma^2}{MH}}.
\end{aligned}
\end{equation}
The proof of Lemma~\ref{lem::onestep_iid} is then complete.
\end{proof}
Lemma~\ref{lem::onestep_iid} breaks the bounding of the recursion $\E\bracket{\norm{\bar\theta^{(t+1)}-\theta^\star}^2}$ into three parts: (i) the ``good term'' $\paren{1-\alpha\mu+\bigo{\alpha^2L^2}}\E\left[\norm{\bar\theta^{(t)}-\theta^\star}^2\right]$, where we can select $\alpha$ small enough such that $\bigo{\alpha^2L^2}$ can be absorbed by the negative term $-\alpha\mu$, ensuring that the distance to the optimum $\theta^\star$ decreases; (ii) the effect of noise $\bigo{{\alpha^2L^2\sigma^2}/{(MH)}}$ which demonstrates the benefit of collaboration (since it gets scaled down by $M$), and (iii) the drift term $DRIFT$ that accounts for the client drift caused by environmental heterogeneity and local steps. We can further bound this term by plugging in Lemma~\ref{lem::drift}, and arrive at the following result of Theorem~\ref{thm::iid_mse}.

\textbf{Proof of Theorem~\ref{thm::iid_mse}.}
\begin{proof}
Plugging the bound from Lemma~\ref{lem::drift} into the bound from Lemma~\ref{lem::onestep_iid}, we obtain:
\begin{equation}
\begin{aligned}
    \E\bracket{\norm{\bar\theta^{(t+1)}-\theta^\star}^2}
    &\leq \paren{1-\alpha\mu+\bigo{\alpha^2L^2}}\E\left[\norm{\bar\theta^{(t)}-\theta^\star}^2\right]+\bigo{\frac{\alpha^2L^2\sigma^2}{MH}}\\
    &\quad+\bigo{\frac{\alpha L^2}{\mu MH}+\frac{\alpha^2L^2}{MH}}MH\alpha^2L^2\paren{\E\bracket{\norm{\bar\theta^{(t)}-\theta^\star}^2}+\sigma^2}\\
    &\leq \paren{1-\alpha\mu+\bigo{\alpha^2L^2}+\bigo{\alpha^4L^4}+\bigo{\frac{\alpha^3L^4}{\mu}}}\E\left[\norm{\bar\theta^{(t)}-\theta^\star}^2\right]\\
    &\quad+\bigo{\frac{\alpha^2L^2\sigma^2}{MH}+\frac{\alpha^3L^4\sigma^2}{\mu}+\alpha^4L^4\sigma^2}\\
    &\leq \paren{1-\alpha\mu+\bigo{\alpha^2L^2}}\E\left[\norm{\bar\theta^{(t)}-\theta^\star}^2\right]+\bigo{\frac{\alpha^2L^2}{MH}+\frac{\alpha^3L^4}{\mu}}\sigma^2.\label{eqn::iid_onestep}
\end{aligned}
\end{equation}
Here, the first inequality follows from the definition of $\alpha=H\eta\alpha_g$. The last one is a result of the fact that $\alpha\leq \mu/L^2$, $L\geq 1$, $\mu\leq 1$, such that $\alpha^4L^4\leq\alpha^2L^2$, $\alpha^3L^4/\mu\leq\alpha^2L^2$, and $\alpha^4L^4\sigma^2\leq \alpha^3L^4\sigma^2/\mu$. 

Applying Eq.~\eqref{eqn::iid_onestep} iteratively $T$ times yields
\begin{equation}
\begin{aligned}
    \E\bracket{\norm{\bar\theta^{(T)}-\theta^\star}^2}
    &\leq \paren{1-\frac{\alpha\mu}{2}}^T\norm{\bar\theta^{(0)}-\theta^\star}^2+\bigo{\frac{\alpha^2L^2}{MH}+\frac{\alpha^3L^4}{\mu}}\sigma^2\sum_{t=0}^{T-1}\paren{1-\frac{\alpha\mu}{2}}^t\\
    &\leq \exp\paren{-\frac{\alpha\mu}{2}T}\norm{\bar\theta^{(0)}-\theta^\star}^2+\bigo{\frac{\alpha^2L^2}{MH}+\frac{\alpha^3L^4}{\mu}}\frac{2\sigma^2}{\alpha\mu}\\
    &\leq \exp\paren{-\frac{\mu}{2}\alpha T}\norm{\bar\theta^{(0)}-\theta^\star}^2+\bigo{\frac{\alpha L^2}{\mu MH}+\frac{\alpha^2L^4}{\mu^2}}\sigma^2.
\end{aligned}
\end{equation}
Here, the first inequality holds since $\alpha\leq \mu/(2CL^2)$ where $C$ is the dominant constant in $\bigo{\alpha^2L^2}$, and thus $C\alpha^2L^2\leq\alpha\mu/2$. The second inequality holds from the fact that $1-x\leq \exp{(-x)}$. The proof is then complete.
\end{proof}
Thus far, we have demonstrated that by selecting $\alpha$ small enough so that the term $\alpha^2L^4/\mu^2$ is subsumed by $\alpha L^2/(\mu MH)$, Theorem~\ref{thm::iid_mse} mirrors the result of the single-agent case as seen in Eq.~\eqref{eqn::mse_bound} under i.i.d. sampling. Additionally, it is important to highlight that a linear speedup is achieved, i.e., the variance term, $\alpha L^2\sigma^2/(\mu MH)$, is reduced in proportion to the number of agents $M$, further underscoring the benefits of collaboration in a multi-agent setting.

\newpage
\section{Analysis of \texttt{FedHSA} under Markovian Sampling}
\label{app:Markovproof}
Now, we shift our focus to the central result of our analysis: the convergence of the \texttt{FedHSA} algorithm under Markovian sampling. In this scenario, the observations $\{o_{i,t}\}_{t \geq 1}$ for each agent $i \in [M]$ are temporally correlated, as opposed to being statistically independent. This introduces additional complexity to the analysis, as we must account for the intricate interactions between observations across different time steps and among various agents.

Building upon the geometric mixing time property, we present the following corollary to deal with time-correlation in observations.

\begin{corollary}\label{corr::markov}
For each agent $i\in[M]$, any given $\theta\in\R^d$, and any non-negative integers $\tau,k$ satisfying $\tau\leq k$, the following holds:
\begin{equation}
    \norm{\E\bracket{G_i(\theta, o_{i,k})|o_{i,k-\tau}}-\bar G_i(\theta)} \leq \bigo{L\rho^\tau}\paren{\norm{\theta-\theta^\star}+\sigma},\label{eqn::mixing_time}
\end{equation}
where $o_{i,k}$ is the observation made by agent $i$ at the $k$-th time-step.
\end{corollary}
We refer the reader to Lemma~3.1 of~\cite{chenQ} for a proof.

Corollary~\ref{corr::markov} states that given a fixed parameter $\theta$, the difference in the Euclidean norm $\norm{\cdot}$ between the true operator and the expectation of the noisy operator, conditioned on the observation from $\tau$ time-steps before, decays exponentially fast, where the $\norm{\theta-\theta^\star}$ term captures the influence of $\theta$. Corollary~\ref{corr::markov} is particularly useful for addressing the temporal correlation between observations. Specifically, when observations are sampled i.i.d. from the distribution $\mu_i$, the left-hand side of~\eqref{eqn::mixing_time} equals zero, recovering the relationship between $G_i$ and $\bar G_i$. 

Before delving into the details of the proofs, we first introduce the following lemma that bounds the variance reduction term in the Markovian sampling scenario.
\begin{lemma}\label{lem::vr_markov}
Suppose Assumptions~\ref{ass::smoothness} to~\ref{ass::independence} hold. Then the following holds for \texttt{FedHSA} $\forall t\geq \bar\tau$, where $\bar\tau=\tau(\alpha^2)$.
\begin{equation}
\begin{aligned}
\overbrace{2\alpha^2\E\left[\norm{\frac{1}{MH}\sum_{i=1}^M\sum_{\ell=0}^{H-1} \left(G_i(\theta_{i,\ell}^{(t)})-\bar G_i(\theta_{i,\ell}^{(t)})\right)}^2\right]}^{VR_{Markov}}&\leq \bigo{\frac{L^2\alpha^2}{MH}}\sum_{i=1}^M\sum_{\ell=0}^{H-1}\E\bracket{\norm{\theta_{i,\ell}^{(t)}-\bar\theta^{(t)}}^2}\\
&\quad +\bigo{L^2\alpha^2}\E\bracket{\norm{\bar\theta^{(t)}-\theta^\star}^2} +\bigo{L^2\alpha^6\sigma^2}\\
&\quad+\bigo{\frac{L^2\alpha^2\sigma^2}{MH(1-\rho)}}.
\end{aligned}
\end{equation}
\end{lemma}
\begin{proof}
Note here that we can no longer bound it as we did in the i.i.d. case due to the presence of Markovian sampling, where data samples for each agent are temporally correlated rather than independent. To handle this, we take advantage of Corollary~\ref{corr::markov} to address the time correlation. We thus bound \( VR_{Markov} \) as follows.

\begin{equation}\label{eqn:T21}
\begin{aligned}
    VR_{Markov}&=2\alpha^2\E\left[\norm{\frac{1}{MH}\sum_{i=1}^M\sum_{\ell=0}^{H-1} \paren{G_i(\theta^\star, o_{i,\ell}^{(t)})-\bar G_i(\theta^\star) +\bar G_i(\theta^\star) -\bar G_i(\theta_{i,\ell}^{(t)}) +  G_i(\theta_{i,\ell}^{(t)},o_{i,\ell}^{(t)}) - G_i(\theta^\star, o_{i,\ell}^{(t)})}}^2\right]\\
    &\leq \overbrace{6\alpha^2\E\left[\norm{\frac{1}{MH}\sum_{i=1}^M\sum_{\ell=0}^{H-1} \left(G_i(\theta^\star,o_{i,\ell}^{(t)})-\bar G_i(\theta^\star)\right)}^2\right]}^{A_1}+\overbrace{6\alpha^2\E\left[\norm{\frac{1}{MH}\sum_{i=1}^M\sum_{\ell=0}^{H-1} \left(\bar G_i(\theta_{i,\ell}^{(t)})-\bar G_i(\theta^\star)\right)}^2\right]}^{A_2}\\
    &\ +\overbrace{6\alpha^2\E\left[\norm{\frac{1}{MH}\sum_{i=1}^M\sum_{\ell=0}^{H-1} \left(G_i(\theta_{i,\ell}^{(t)},o_{i,\ell}^{(t)})-G_i(\theta^\star,o_{i,\ell}^{(t)})\right)}^2\right]}^{A_3}.
\end{aligned}
\end{equation}

We proceed to bound $A_2$ as:
\begin{equation}\label{eqn:A2}
\begin{aligned}
    A_2&\leq \bigo{\frac{L^2\alpha^2}{MH}}\sum_{i=1}^M\sum_{\ell=0}^{H-1}\E\bracket{\norm{\theta_{i,\ell}^{(t)}-\theta^\star}^2}\\
    &\leq \bigo{\frac{L^2\alpha^2}{MH}}\sum_{i=1}^M\sum_{\ell=0}^{H-1}\E\bracket{\norm{\theta_{i,\ell}^{(t)}-\bar\theta^{(t)}}^2} + \bigo{\frac{L^2\alpha^2}{MH}}\sum_{i=1}^M\sum_{\ell=0}^{H-1}\E\bracket{\norm{\bar\theta^{(t)}-\theta^\star}^2}\\
    &= \bigo{\frac{L^2\alpha^2}{MH}}\sum_{i=1}^M\sum_{\ell=0}^{H-1}\E\bracket{\norm{\theta_{i,\ell}^{(t)}-\bar\theta^{(t)}}^2} + \bigo{L^2\alpha^2}\E\bracket{\norm{\bar\theta^{(t)}-\theta^\star}^2},
\end{aligned}
\end{equation}
where we used Assumption \ref{ass::smoothness}. We can obtain the same bound for $A_3$ with the same reasoning.

Now we continue to bound the term $A_1$. With a slight overload of notation, denote $e_{i,\ell}=G_i(\theta^\star,o_{i,\ell}^{(t)})-\bar G_i(\theta^\star)$, and we obtain
\begin{equation}\label{eqn:A1}
\begin{aligned}
    A_1&=\bigo{\frac{\alpha^2}{M^2H^2}}\E\left[\norm{\sum_{i=1}^M\sum_{\ell=0}^{H-1} e_{i,\ell}}^2\right]\\
    &=\overbrace{\bigo{\frac{\alpha^2}{M^2H^2}}\sum_{i=1}^M\sum_{\ell=0}^{H-1}\E\left[\norm{ e_{i,\ell}}^2\right]}^{B_1} + \overbrace{\bigo{\frac{\alpha^2}{M^2H^2}}\sum_{i=1}^M\sum_{\ell_1\neq\ell_2}\E\bracket{\iprod{e_{i,\ell_1}}{e_{i,\ell_2}}}}^{B_2}\\
    &\ + \overbrace{\bigo{\frac{\alpha^2}{M^2H^2}}\sum_{i\neq j}\sum_{\ell=0}^{H-1}\E\bracket{\iprod{e_{i,\ell}}{e_{j,\ell}}}}^{B_3} + \overbrace{\bigo{\frac{\alpha^2}{M^2H^2}}\sum_{i\neq j}\sum_{\ell_1\neq\ell_2}\E\bracket{\iprod{e_{i,\ell_1}}{e_{j,\ell_2}}}}^{B_4}.
\end{aligned}
\end{equation}
We then bound the four terms separately.

For the term $B_1$, we have
\begin{equation}\label{eqn:B1}
\begin{aligned}
    B_1&\leq \bigo{\frac{\alpha^2}{M^2H^2}}\sum_{i=1}^M\sum_{\ell=0}^{H-1}\E\left[\norm{G_i(\theta^\star,o_{i,\ell}^{(t)})}^2+\norm{\bar G_i(\theta^\star)}^2\right]\\
    &\leq \bigo{\frac{L^2\alpha^2\sigma^2}{MH}},
\end{aligned}
\end{equation}
where we used Corollary~\ref{corr:diff}.

For the term $B_2$, we have
\begin{equation}\label{eqn:B2}
\begin{aligned}
    {B_2}&\overset{(a)}{=}\bigo{\frac{\alpha^2}{M^2H^2}}\sum_{i=1}^M\sum_{\ell_1\neq\ell_2}\E\bracket{\E\bracket{\iprod{e_{i,\ell_1}}{e_{i,\ell_2}}|o_{i,\ell_1}^{(t)}}} \\
    & = \bigo{\frac{\alpha^2}{M^2H^2}}\sum_{i=1}^M\sum_{\ell_1\neq\ell_2}\E\bracket{\iprod{e_{i,\ell_1}}{\E\bracket{e_{i,\ell_2}|o_{i,\ell_1}^{(t)}}}} \\
    & \leq \bigo{\frac{\alpha^2}{M^2H^2}}\sum_{i=1}^M\sum_{\ell_1\neq\ell_2}\E\bracket{\norm{e_{i,\ell_1}}\norm{\E\bracket{e_{i,\ell_2}|o_{i,\ell_1}^{(t)}}}} \\
    & \overset{(b)}{\leq} \bigo{\frac{L\alpha^2}{M^2H^2}}\sum_{i=1}^M\sum_{\ell_1=0}^{H-1}\sum_{\ell_2>\ell_1}\rho^{\ell_2-\ell_1}\sigma\E\bracket{\norm{e_{i,\ell_1}}} \\
    & \leq \bigo{\frac{L\alpha^2}{M^2H^2}}\sum_{i=1}^M\sum_{\ell_1=0}^{H-1}\sum_{\ell_2>\ell_1}\rho^{\ell_2-\ell_1}\sigma\E\bracket{\norm{G_i(\theta^\star,o_{i,\ell_1}^{(t)})}+\norm{\bar G_i(\theta^\star)}} \\
    & \overset{(c)}{\leq} \bigo{\frac{L^2\alpha^2}{M^2H^2}}\sum_{i=1}^M\sum_{\ell_1=0}^{H-1}\sum_{\ell_2>\ell_1}\rho^{\ell_2-\ell_1}\sigma^2\\
    & \leq \bigo{\frac{L^2\alpha^2\sigma^2}{M^2H^2}}\sum_{i=1}^M\sum_{\ell_1=0}^{H-1}(1+\rho+\rho^2+\cdots)\\
    &\leq \bigo{\frac{L^2\alpha^2\sigma^2}{MH(1-\rho)}},
\end{aligned}
\end{equation}
where we used the tower property of expectation in $(a)$, Corollary~\ref{corr::markov} in $(b)$, and Corollary~\ref{corr:diff} in $(c)$.

For the term $B_3$, we have
\begin{equation}\label{eqn:B3}
\begin{aligned}
    B_3 &\overset{(a)}{=}\bigo{\frac{\alpha^2}{M^2H^2}}\sum_{i\neq j}\sum_{\ell=0}^{H-1}\iprod{\E\bracket{e_{i,\ell}}}{\E\bracket{e_{j,\ell}}}\\
    &\leq \bigo{\frac{\alpha^2}{M^2H^2}}\sum_{i\neq j}\sum_{\ell=0}^{H-1}\norm{\E\bracket{e_{i,\ell}}}\norm{\E\bracket{e_{j,\ell}}}\\
    &\overset{(b)}{=} \bigo{\frac{\alpha^2}{M^2H^2}}\sum_{i\neq j}\sum_{\ell=0}^{H-1}\norm{\E\bracket{\E\bracket{e_{i,\ell}|o_{i,tH+\ell-\bar\tau}}}}\norm{\E\bracket{\E\bracket{e_{j,\ell}|o_{j,tH+\ell-\bar\tau}}}}\\
    &\overset{(c)}{\leq} \bigo{\frac{\alpha^2}{M^2H^2}}\sum_{i\neq j}\sum_{\ell=0}^{H-1}\E\bracket{\norm{\E\bracket{e_{i,\ell}|o_{i,tH+\ell-\bar\tau}}}}\E\bracket{\norm{\E\bracket{e_{j,\ell}|o_{j,tH+\ell-\bar\tau}}}}\\
    &\overset{(d)}{\leq} \bigo{\frac{L^2\alpha^2}{M^2H^2}}\sum_{i\neq j}\sum_{\ell=0}^{H-1}\rho^{2\bar\tau}\sigma^2\\
    &\overset{(e)}{\leq}\bigo{\frac{L^2\alpha^6\sigma^2}{H}},
\end{aligned}
\end{equation}
where $(a)$ is a result of Assumption~\ref{ass::independence}, $(b)$ uses the tower property of expectation, $(c)$ uses the fact that $\norm{\E\bracket{x}}\leq \E\bracket{\norm{x}}$, $(d)$ is the application of Corollary~\ref{corr::markov}, and $(e)$ uses the definition that $\rho^{\bar\tau}\leq \alpha^2$. Note here that we are allowed to condition on $\bar\tau$ time-steps before because $t\geq \bar\tau$.

Similarly, we can achieve the bound for $B_4$ as
\begin{equation}\label{eqn:B4}
    B_4\leq \bigo{L^2\alpha^6\sigma^2}.
\end{equation}

Further plugging the bounds for $B_1$ to $B_4$ into~\eqref{eqn:A1} yields
\begin{equation}\label{eqn:A1_bound}
\begin{aligned}
    A_1&\leq \bigo{\frac{L^2\alpha^2\sigma^2}{MH}}+\bigo{\frac{L^2\alpha^2\sigma^2}{MH(1-\rho)}}+\bigo{\frac{L^2\alpha^6\sigma^2}{H}}+\bigo{L^2\alpha^6\sigma^2}\\
    &\leq\bigo{L^2\alpha^6\sigma^2}+\bigo{\frac{L^2\alpha^2\sigma^2}{MH(1-\rho)}}.
\end{aligned}
\end{equation}

Finally, plugging the bounds for $A_1$ to $A_3$ into~\eqref{eqn:T21} yields
\begin{equation}\label{eqn:T21_bound}
    VR_{Markov}\leq \bigo{\frac{L^2\alpha^2}{MH}}\sum_{i=1}^M\sum_{\ell=0}^{H-1}\E\bracket{\norm{\theta_{i,\ell}^{(t)}-\bar\theta^{(t)}}^2}+\bigo{L^2\alpha^2}\E\bracket{\norm{\bar\theta^{(t)}-\theta^\star}^2}+\bigo{L^2\alpha^6\sigma^2}+\bigo{\frac{L^2\alpha^2\sigma^2}{MH(1-\rho)}}.
\end{equation}
The proof is then complete.
\end{proof}

With the upper-bound for the variance reduction term under Markovian sampling, we now introduce the following lemma, which bounds the one-step recursion of the distance to the optimal parameter $\theta^\star$ in the Markovian sampling setting. This lemma lays the foundation for understanding the impact of time correlations on the convergence behavior of \texttt{FedHSA}.

\begin{lemma}\label{lem::onestep_markov}
Suppose Assumptions~\ref{ass::smoothness} to~\ref{ass::independence} hold. Then the following holds for \texttt{FedHSA} $\forall t\geq \bar\tau$:
\begin{equation}
\begin{aligned}
    \E\bracket{\norm{\bar\theta^{(t+1)}-\theta^\star}^2}&\leq\paren{1-\alpha\mu+\bigo{L^2\alpha^2}}\E\bracket{\norm{\bar\theta^{(t)}-\theta^\star}^2} +\bigo{\frac{L^2\alpha^2}{MH}+\frac{L^2\alpha}{\mu MH}}\sum_{i=1}^M\sum_{\ell=0}^{H-1}\E\bracket{\norm{\theta_{i,\ell}^{(t)}-\bar \theta^{(t)}}^2}\\
    &\ +\bigo{L^2\alpha^6\sigma^2}+\bigo{\frac{L^2\alpha^2\sigma^2}{MH(1-\rho)}}  +\E\bracket{\iprod{\bar\theta^{(t)}-\theta^\star}{\frac{2\alpha}{MH}\sum_{i=1}^M\sum_{\ell=0}^{H-1} \paren{G_{i}(\theta_{i,\ell}^{(t)})-\bar G_{i}(\theta_{i,\ell}^{(t)})}}}.
\end{aligned}
\end{equation}
\end{lemma}
\begin{proof}
As in the i.i.d. case, we first write the error update rule as follows:
\begin{equation}\label{eqn::error_update}
    \norm{\bar\theta^{(t+1)}-\theta^\star}^2=\norm{\bar\theta^{(t)}-\theta^\star}^2+\overbrace{\left\langle\bar\theta^{(t)}-\theta^\star, \frac{2\alpha}{MH}\sum_{i=1}^M\sum_{\ell=0}^{H-1} G_{i}(\theta_{i,\ell}^{(t)})\right\rangle}^{U_1}+\overbrace{\alpha^2\norm{\frac{1}{MH}\sum_{i=1}^M\sum_{\ell=0}^{H-1} G_{i}(\theta_{i,\ell}^{(t)})}^2}^{U_2}.
\end{equation}

We proceed to bound the term $U_2$ using Jensen's inequality. By taking expectation, we obtain:
\begin{equation}\label{eqn:T2}
    \E[U_2]\leq \overbrace{2\alpha^2\E\left[\norm{\frac{1}{MH}\sum_{i=1}^M\sum_{\ell=0}^{H-1} \left(G_i(\theta_{i,\ell}^{(t)})-\bar G_i(\theta_{i,\ell}^{(t)})\right)}^2\right]}^{U_{21}}+\overbrace{2\alpha^2\E\left[\norm{\frac{1}{MH}\sum_{i=1}^M\sum_{\ell=0}^{H-1} \bar G_i(\theta_{i,\ell}^{(t)})}^2\right]}^{U_{22}}.
\end{equation}

Note here that the term $U_{22}$ only involves the true operators, and thus we do not need to consider the effect of Markovian sampling. Therefore, it can be upper-bounded exactly as the term $T_2$ in the i.i.d. case, i.e.:
\begin{equation}\label{eqn:T22}
\begin{aligned}
    U_{22}&\leq \bigo{\frac{L^2\alpha^2}{MH}}\sum_{i=1}^M\sum_{\ell=0}^{H-1}\E\bracket{\norm{\theta_{i,\ell}^{(t)}-\bar\theta^{(t)}}^2}+\bigo{L^2\alpha^2}\E\bracket{\norm{\bar\theta^{(t)}-\theta^\star}^2}.
\end{aligned}
\end{equation}

The term \( U_{21} \) is exactly the variance reduction term we bounded in Lemma~\ref{lem::vr_markov}.
Plugging the bounds for $U_{21}$ and $U_{22}$ into~\eqref{eqn:T2} yields
\begin{equation}\label{eqn::U2}
    \E[U_{2}]\leq \bigo{\frac{L^2\alpha^2}{MH}}\sum_{i=1}^M\sum_{\ell=0}^{H-1}\E\bracket{\norm{\theta_{i,\ell}^{(t)}-\bar\theta^{(t)}}^2}+\bigo{L^2\alpha^2}\E\bracket{\norm{\bar\theta^{(t)}-\theta^\star}^2}+\bigo{L^2\alpha^6\sigma^2}+\bigo{\frac{L^2\alpha^2\sigma^2}{MH(1-\rho)}}.
\end{equation}

For the term $\E[U_1]$, we can decompose it into two parts:
\begin{equation}\label{eqn::C1}
\begin{aligned}
    \E[U_1]=\overbrace{\E\bracket{\iprod{\bar\theta^{(t)}-\theta^\star}{\frac{2\alpha}{MH}\sum_{i=1}^M\sum_{\ell=0}^{H-1} \paren{G_{i}(\theta_{i,\ell}^{(t)})-\bar G_{i}(\theta_{i,\ell}^{(t)})}}}}^{C_1}+\overbrace{\E\bracket{\iprod{\bar\theta^{(t)}-\theta^\star}{\frac{2\alpha}{MH}\sum_{i=1}^M\sum_{\ell=0}^{H-1} \bar G_{i}(\theta_{i,\ell}^{(t)})}}}^{C_2}.
\end{aligned}
\end{equation}
Note that the second part $C_2$ can be bounded similarly as the term $T_1$ in the i.i.d. case:
\begin{equation}\label{eqn::C2}
\begin{aligned}
    C_2&\leq-\alpha\mu\E\left[\norm{\bar \theta^{(t)}-\theta^\star}^2\right]+\frac{\alpha L^2}{\mu MH}\sum_{i=1}^M\sum_{\ell=0}^{H-1}\E\bracket{\norm{\theta_{i,\ell}^{(t)}-\bar \theta^{(t)}}^2}.
\end{aligned}
\end{equation}

Plugging~\eqref{eqn::C2} into~\eqref{eqn::C1}, together with~\eqref{eqn::U2} into~\eqref{eqn::error_update} and taking expectation on both sides, we obtain
\begin{equation}
\begin{aligned}
    \E\bracket{\norm{\bar\theta^{(t+1)}-\theta^\star}^2}&\leq\paren{1-\alpha\mu+\bigo{L^2\alpha^2}}\E\bracket{\norm{\bar\theta^{(t)}-\theta^\star}^2} +\bigo{\frac{L^2\alpha^2}{MH}+\frac{L^2\alpha}{\mu MH}}\sum_{i=1}^M\sum_{\ell=0}^{H-1}\E\bracket{\norm{\theta_{i,\ell}^{(t)}-\bar \theta^{(t)}}^2}\\
    &\quad +\overbrace{\E\bracket{\iprod{\bar\theta^{(t)}-\theta^\star}{\frac{2\alpha}{MH}\sum_{i=1}^M\sum_{\ell=0}^{H-1} \paren{G_{i}(\theta_{i,\ell}^{(t)})-\bar G_{i}(\theta_{i,\ell}^{(t)})}}}}^{T_{bias}}+\bigo{L^2\alpha^6\sigma^2}+\bigo{\frac{L^2\alpha^2\sigma^2}{MH(1-\rho)}}.
\end{aligned}
\end{equation}
\end{proof}
Comparing Lemma~\ref{lem::onestep_markov} with Lemma~\ref{lem::onestep_iid}, we observe that the one-step bound for the i.i.d. case can be recovered in Lemma~\ref{lem::onestep_markov} if data are i.i.d., i.e., $\rho=0$. The only term that impedes further bounding in Lemma~\ref{lem::onestep_markov} is the one $T_{bias}$ capturing the bias caused by Markovian sampling, which equals zero in the i.i.d. case as proved in Eq.~\eqref{eqn::T1_iid}. The next lemma then focuses on bounding this distinct bias term in the Markovian case.

\subsection{Bounding of the Markovian Bias}
It is worth noting that the \textbf{goal} of bounding the bias term is to ensure that the term $\E\bracket{\norm{\bar\theta^{(t)}-\theta^\star}^2}$ on the R.H.S. has a scaling factor no larger than $\alpha\mu$, along with a constant term scaled by an order of $\alpha$ that is either higher than or equal to 3, or an order of 2 but scaled down by the number of agents $M$, thus demonstrating the benefits of collaboration. Unfortunately, directly applying Eq.~\eqref{eqn::cs} cannot yield the desired result. Instead, by applying Corollary~\ref{corr::markov} to account for the time correlation in data samples, we condition on a parameter at least $\bar\tau$ time-steps earlier in time. Specifically, since in the bias term, the parameter $\theta_{i,\ell}^{(t)}$ is from the $t$-th round and $\ell$-th local step with $t\geq \bar\tau$, we condition on $\bar\theta^{(t-\bar\tau)}$, which is at least $\bar\tau$ time-steps earlier. For this reason, we first provide an upper-bound for the term \(\E\bracket{\norm{\bar\theta^{(t)}-\bar\theta^{(t-\tau)}}^2}\), which comes in handy in subsequent proofs. 

\begin{lemma}\label{lem::diff_tau}
Suppose all the conditions in Lemma~\ref{lem::onestep_markov} hold. Then there exist a universal constant $\bar C\geq 1$ such that by selecting $\eta\leq 1/(\bar\tau \Tilde{C}HL^2)$ and $\alpha_g=1$, the following is true for \texttt{FedHSA} for all $t\geq 2\bar\tau$:
\begin{equation}
    \E\bracket{\norm{\bar\theta^{(t)}-\bar\theta^{(t-\bar\tau)}}^2}\leq \bigo{\bar\tau^2L^2\alpha^2}\E\bracket{\norm{\bar\theta^{(t)}-\theta^\star}^2}+\bigo{\bar\tau^2L^4\alpha^4\sigma^2+\frac{\bar\tau^2L^2\alpha^2\sigma^2}{MH(1-\rho)}}.
\end{equation}
\end{lemma}
\begin{proof}
Observe that
\begin{equation}
\begin{aligned}
    \norm{\bar\theta^{(t+1)}-\theta^\star}^2&=\left\|\bar\theta^{(t)}-\theta^\star\right\|_2^2+2\alpha{\left\langle\bar\theta^{(t)}-\theta^\star, \frac{1}{MH}\sum_{i=1}^M\sum_{\ell=0}^{H-1} G_{i}(\theta_{i,\ell}^{(t)})\right\rangle}+\alpha^2{\left\|\frac{1}{MH}\sum_{i=1}^M\sum_{\ell=0}^{H-1} G_{i}(\theta_{i,\ell}^{(t)})\right\|_2^2}\\
    &\leq\left\|\bar\theta^{(t)}-\theta^\star\right\|_2^2+\alpha{\left\|\frac{1}{MH}\sum_{i=1}^M\sum_{\ell=0}^{H-1} G_{i}(\theta_{i,\ell}^{(t)})\right\|_2^2}+\alpha\left\|\bar\theta^{(t)}-\theta^\star\right\|_2^2+\alpha^2\norm{\frac{1}{MH}\sum_{i=1}^M\sum_{\ell=0}^{H-1} G_{i}(\theta_{i,\ell}^{(t)})}^2\\
    &\leq (1+\alpha)\left\|\bar\theta^{(t)}-\theta^\star\right\|_2^2+2\alpha\norm{\frac{1}{MH}\sum_{i=1}^M\sum_{\ell=0}^{H-1} G_{i}(\theta_{i,\ell}^{(t)})}^2,
\end{aligned}
\end{equation}
where we used the fact that $\alpha\leq 1$, and $2\iprod{a}{b}\leq \norm{a}^2+\norm{b}^2$, for all $(a,b)\in\R^d\times\R^d$. Taking expectation of both sides and plugging in the bound for $2\E[U_2]/\alpha$ in Eq.~\eqref{eqn::U2} yields

\begin{equation}
\begin{aligned}
    \E\bracket{\norm{\bar\theta^{(t+1)}-\bar\theta^\star}^2}&\leq (1+\alpha)\E\bracket{\left\|\bar\theta^{(t)}-\theta^\star\right\|_2^2} + \bigo{\frac{L^2\alpha}{MH}}\sum_{i=1}^M\sum_{\ell=0}^{H-1}\E\bracket{\norm{\theta_{i,\ell}^{(t)}-\bar\theta^{(t)}}^2}+\bigo{L^2\alpha}\E\bracket{\norm{\bar\theta^{(t)}-\theta^\star}^2}\\
    &\quad +\bigo{L^2\alpha^5\sigma^2}+\bigo{\frac{L^2\alpha\sigma^2}{MH(1-\rho)}}\\
    &\leq (1+\alpha)\E\bracket{\left\|\bar\theta^{(t)}-\theta^\star\right\|_2^2} + \bigo{L^4\alpha^3}\paren{\E\bracket{\norm{\bar\theta^{(t)}-\theta^\star}^2}+\sigma^2}+\bigo{L^2\alpha}\E\bracket{\norm{\bar\theta^{(t)}-\theta^\star}^2}\\
    &\quad +\bigo{L^2\alpha^5\sigma^2}+\bigo{\frac{L^2\alpha\sigma^2}{MH(1-\rho)}}\\
    &\leq \paren{1+\alpha+\bigo{L^4\alpha^3+L^2\alpha}}\E\bracket{\left\|\bar\theta^{(t)}-\theta^\star\right\|_2^2}+\bigo{L^2\alpha^5\sigma^2}+\bigo{\frac{L^2\alpha\sigma^2}{MH(1-\rho)}}\\
    &\quad+\bigo{L^4\alpha^3\sigma^2}\\
    &\leq \paren{1+\bigo{L^2\alpha}}\E\bracket{\left\|\bar\theta^{(t)}-\theta^\star\right\|_2^2}+\bigo{L^4\alpha^3\sigma^2}+\bigo{\frac{L^2\alpha\sigma^2}{MH(1-\rho)}},
\end{aligned}
\end{equation}
where in the second inequality we used Lemma~\ref{lem::drift} and we selected $\bar\tau$ large enough such that $\rho^{\bar\tau}\leq \alpha^2$, and in the last inequlaity we used the fact that $\alpha L\leq 1$.

Therefore, for any $t'\in[t-\bar\tau, t]$, we can write
\begin{equation}
\begin{aligned}
    \E\bracket{\left\|\bar\theta^{(t')}-\theta^\star\right\|_2^2}&\leq \paren{1+\bigo{L^2\alpha}}^{\bar\tau}\E\bracket{\left\|\bar\theta^{(t-\tau)}-\theta^\star\right\|_2^2} + \sum_{\ell=0}^{\tau-1}\paren{1+\bigo{L^2\alpha}}^\ell\bigo{L^4\alpha^3\sigma^2+\frac{L^2\alpha\sigma^2}{MH(1-\rho)}}\\
    &\leq \bigo{\E\bracket{\left\|\bar\theta^{(t-\tau)}-\theta^\star\right\|_2^2}}+\tau\bigo{L^4\alpha^3\sigma^2+\frac{L^2\alpha\sigma^2}{MH(1-\rho)}},\label{eqn::t'}
\end{aligned}
\end{equation}
where we selected $\alpha\leq 1/(\Tilde{C}L^2\bar\tau)$, where $\Tilde{C}\geq 1$ is the dominant constant in $\bigo{L^2\alpha}$, such that $\paren{1+\bigo{L^2\alpha}}^{\bar\tau}\leq(1+1/\bar\tau)^{\bar\tau}\leq e$.

Next, observe that
\begin{equation}
\begin{aligned}
    \E\bracket{\left\|\bar\theta^{(t)}-\bar\theta^{(t-\bar\tau)}\right\|_2^2}&\leq \bar\tau\sum_{m=t-\bar\tau}^{t-1}\E\bracket{\left\|\bar\theta^{(m+1)}-\bar\theta^{(m)}\right\|_2^2}\\
    &=\bar\tau\sum_{m=t-\bar\tau}^{t-1}\E\bracket{\alpha^2\norm{\frac{1}{MH}\sum_{i=1}^M\sum_{\ell=0}^{H-1} G_{i}(\theta_{i,\ell}^{(m)})}^2}\\
    &\overset{(a)}{\leq} \bar\tau\sum_{m=t-\bar\tau}^{t-1}\bigo{L^2\alpha^2\E\bracket{\norm{\bar\theta^{(m)}-\theta^\star}^2}+L^4\alpha^4\sigma^2+\frac{L^2\alpha^2\sigma^2}{MH(1-\rho)}}\\
    &\overset{(b)}{\leq} \bar\tau\sum_{m=t-\bar\tau}^{t-1}\bigo{L^2\alpha^2\E\bracket{\norm{\bar\theta^{(t-\bar\tau)}-\theta^\star}^2}+\bar\tau L^6\alpha^5\sigma^2+\frac{\bar\tau L^4\alpha^3\sigma^2}{MH(1-\rho)}+L^4\alpha^4\sigma^2+\frac{L^2\alpha^2\sigma^2}{MH(1-\rho)}}\\
    &\overset{(c)}{\leq} \bar\tau\sum_{m=t-\bar\tau}^{t-1}\bigo{L^2\alpha^2\E\bracket{\norm{\bar\theta^{(t-\bar\tau)}-\theta^\star}^2}+L^4\alpha^4\sigma^2+\frac{L^2\alpha^2\sigma^2}{MH(1-\rho)}}\\
    &\leq \bigo{\bar\tau^2L^2\alpha^2}\E\bracket{\norm{\bar\theta^{(t-\bar\tau)}-\theta^\star}^2}+\bigo{\bar\tau^2L^4\alpha^4\sigma^2+\frac{\bar\tau^2L^2\alpha^2\sigma^2}{MH(1-\rho)}}\\
    &\leq \bigo{\bar\tau^2L^2\alpha^2}\E\bracket{\norm{\bar\theta^{(t-\bar\tau)}-\bar\theta^{(t)}}^2}+\bigo{\bar\tau^2L^2\alpha^2}\E\bracket{\norm{\bar\theta^{(t)}-\theta^\star}^2}\\
    &\, +\bigo{\bar\tau^2L^4\alpha^4\sigma^2 +\frac{\bar\tau^2L^2\alpha^2\sigma^2}{MH(1-\rho)}},
\end{aligned}
\end{equation}
where in $(a)$ we used the bound for $\E[U_2]$ and plugged in Lemma~\ref{lem::drift}, i.e.,
\begin{equation}
    \E[U_{2}]\leq \bigo{L^2\alpha^2}\E\bracket{\norm{\bar\theta^{(t)}-\theta^\star}^2}+\bigo{L^4\alpha^4\sigma^2}+\bigo{\frac{L^2\alpha^2\sigma^2}{MH(1-\rho)}}.\label{eqn::U2_plugged}
\end{equation}
Note here that we are allowed to use the bound for $\E[U_2]$ because we require $t\geq 2\bar\tau$, and thus $m\geq t-\bar\tau\geq \bar\tau$; in $(b)$ we used the result in Eq.~\eqref{eqn::t'}, and in $(c)$ we selected $\alpha\leq 1/(L^2\bar\tau)$ such that $\bar\tau L^6\alpha^5\sigma^2\leq L^4\alpha^4\sigma^2$ and $\bar\tau L^4\alpha^3\sigma^2\leq L^2\alpha^2\sigma^2$.

Suppose that the dominating constant in $\bigo{\bar\tau^2L^2\alpha^2}\E\bracket{\norm{\bar\theta^{(t-1)}-\bar\theta^{(t)}}^2}$ is $\bar C\geq 1$. We can then write
\begin{equation}
    (1-\bar C\bar\tau^2L^2\alpha^2)\E\bracket{\norm{\bar\theta^{(t)}-\bar\theta^{(t-1)}}^2}\leq \bigo{\bar\tau^2L^2\alpha^2}\E\bracket{\norm{\bar\theta^{(t)}-\theta^\star}^2}+\bigo{\bar\tau^2L^4\alpha^4\sigma^2+\frac{\bar\tau^2L^2\alpha^2\sigma^2}{MH(1-\rho)}}.
\end{equation}
By selecting $\alpha$ such that $\bar C\bar\tau^2L^2\alpha^2\leq 1/2$, i.e., $\alpha\leq 1/(\sqrt{2\bar C} \bar\tau L)$ we obtain
\begin{equation}
    \E\bracket{\norm{\bar\theta^{(t)}-\bar\theta^{(t-\bar\tau)}}^2}\leq \bigo{\bar\tau^2L^2\alpha^2}\E\bracket{\norm{\bar\theta^{(t)}-\theta^\star}^2}+\bigo{\bar\tau^2L^4\alpha^4\sigma^2+\frac{\bar\tau^2L^2\alpha^2\sigma^2}{MH(1-\rho)}}.
\end{equation} 
The proof is then complete.
\end{proof}

With Lemma~\ref{lem::diff_tau} at hand, we can arrive at the following lemma that bounds the Markovian bias:
\begin{lemma}\label{lem::markov_bias}
Suppose all the conditions in Lemma~\ref{lem::diff_tau} hold. Then the following is true for \texttt{FedHSA} for all $t\geq 2\bar\tau$:
\begin{equation}
\begin{aligned}
    \overbrace{\E\bracket{\iprod{\bar\theta^{(t)}-\theta^\star}{\frac{2\alpha}{MH}\sum_{i=1}^M\sum_{\ell=0}^{H-1} \paren{G_{i}(\theta_{i,\ell}^{(t)})-\bar G_{i}(\theta_{i,\ell}^{(t)})}}}}^{C_1}&\leq \paren{\frac{\alpha\mu}{2}+\bigo{\bar\tau L^2\alpha^2+\frac{L^4\alpha^3}{\mu}}}\E\bracket{\norm{\bar\theta^{(t)}-\theta^\star}^2}\\
    &\quad+\bigo{\frac{  L^4 \sigma^2\alpha^3}{\mu}+\frac{\tau L^2\sigma^2\alpha^2}{MH(1-\rho)}}.
\end{aligned}
\end{equation}
\end{lemma}
\begin{proof}
Observe that 
\begin{equation}
\begin{aligned}
    \overbrace{\E\bracket{\iprod{\bar\theta^{(t)}-\theta^\star}{\frac{2\alpha}{MH}\sum_{i=1}^M\sum_{\ell=0}^{H-1} \paren{G_{i}(\theta_{i,\ell}^{(t)})-\bar G_{i}(\theta_{i,\ell}^{(t)})}}}}^{C_1} &= \overbrace{\E\bracket{\iprod{\bar\theta^{(t)}-\bar\theta^{(t-\bar\tau)}}{\frac{2\alpha}{MH}\sum_{i=1}^M\sum_{\ell=0}^{H-1} \paren{G_{i}(\theta_{i,\ell}^{(t)})-\bar G_{i}(\theta_{i,\ell}^{(t)})}}}}^{D_1}\\
    &\ +\overbrace{\E\bracket{\iprod{\bar\theta^{(t-\bar\tau)}-\theta^\star}{\frac{2\alpha}{MH}\sum_{i=1}^M\sum_{\ell=0}^{H-1} \paren{G_{i}(\theta_{i,\ell}^{(t)})-\bar G_{i}(\theta_{i,\ell}^{(t)})}}}}^{D_2}.
\end{aligned}
\end{equation}
We then continue to bound the terms $D_1$ and $D_2$ separately. For the term $D_1$, we have
\begin{equation}
\begin{aligned}
    D_1&\leq \frac{1}{\bar\tau}\E\bracket{\norm{\bar\theta^{(t)}-\bar\theta^{(t-\bar\tau)}}^2}  + \bar\tau\alpha^2\E\bracket{\norm{\frac{1}{MH}\sum_{i=1}^M\sum_{\ell=0}^{H-1} \paren{G_{i}(\theta_{i,\ell}^{(t)})-\bar G_{i}(\theta_{i,\ell}^{(t)})}}^2}\\
    &\leq \bigo{\bar\tau L^2\alpha^2}\E\bracket{\norm{\bar\theta^{(t)}-\theta^\star}^2}+\bigo{\bar\tau L^4\alpha^4\sigma^2+\frac{\bar\tau L^2\alpha^2\sigma^2}{MH(1-\rho)}} \\
    &\quad +\bigo{\bar\tau L^2\alpha^2}\E\bracket{\norm{\bar\theta^{(t)}-\theta^\star}^2}+\bigo{\bar\tau L^4\alpha^4\sigma^2}+\bigo{\frac{\bar\tau L^2\alpha^2\sigma^2}{MH(1-\rho)}}\\
    &\leq \bigo{\bar\tau L^2\alpha^2}\E\bracket{\norm{\bar\theta^{(t)}-\theta^\star}^2}+\bigo{\bar\tau L^4\alpha^4\sigma^2+\frac{\bar\tau L^2\alpha^2\sigma^2}{MH(1-\rho)}},
\end{aligned}
\end{equation}
where we use Lemma~\ref{lem::diff_tau} and Lemma~\ref{lem::vr_markov} with the bound from Lemma~\ref{lem::drift} plugged in. Clearly the bound for $D_1$ is eligible for our goal.

For the term $D_2$, again, directly bounding will not suffice. Therefore, we introduce intermediate terms $\bar G_i(\bar\theta^{(t)})$ and $G_i(\bar\theta^{(t)}, o_{i,\ell}^{(t)})$.

\begin{equation}
\begin{aligned}
    D_2&=\overbrace{\E\bracket{\iprod{\bar\theta^{(t-\bar\tau)}-\theta^\star}{\frac{2\alpha}{MH}\sum_{i=1}^M\sum_{\ell=0}^{H-1} \paren{G_{i}(\bar\theta^{(t)},o_{i,\ell}^{(t)})-\bar G_{i}(\bar\theta^{(t)})}}}}^{E_1}\\
    &\ +\overbrace{\E\bracket{\iprod{\bar\theta^{(t-\bar\tau)}-\theta^\star}{\frac{2\alpha}{MH}\sum_{i=1}^M\sum_{\ell=0}^{H-1} \paren{G_{i}(\theta_{i,\ell}^{(t)},o_{i,\ell}^{(t)})- G_{i}(\bar\theta^{(t)},o_{i,\ell}^{(t)})}}}}^{E_2}\\
    &\ +\overbrace{\E\bracket{\iprod{\bar\theta^{(t-\bar\tau)}-\theta^\star}{\frac{2\alpha}{MH}\sum_{i=1}^M\sum_{\ell=0}^{H-1} \paren{\bar G_{i}(\bar\theta^{(t)})-\bar G_{i}(\theta_{i,\ell}^{(t)})}}}}^{E_3}.
\end{aligned}
\end{equation}
We then bound these three terms separately.

For the term $E_2$, we have
\begin{equation}
\begin{aligned}
    E_2&\leq \frac{\alpha}{\beta}\E\bracket{\norm{\bar\theta^{(t-\bar\tau)}-\theta^\star}^2}+\alpha\beta\E\bracket{\norm{\frac{1}{MH}\sum_{i=1}^M\sum_{\ell=0}^{H-1} \paren{G_{i}(\theta_{i,\ell}^{(t)},o_{i,\ell}^{(t)})- G_{i}(\bar\theta^{(t)},o_{i,\ell}^{(t)})}}^2}\\
    &\overset{(a)}{\leq} \frac{2\alpha}{\beta}\E\bracket{\norm{\bar\theta^{(t-\bar\tau)}-\bar\theta^{(t)}}^2}+\frac{2\alpha}{\beta}\E\bracket{\norm{\bar\theta^{(t)}-\theta^\star}^2}+\alpha\beta\frac{1}{MH}\sum_{i=1}^M\sum_{\ell=0}^{H-1}L^2\E\bracket{\norm{\theta_{i,\ell}^{(t)}-\bar\theta^{(t)}}^2}\\
    &\overset{(b)}{\leq}\bigo{\bar\tau^2L^2\alpha^3}\frac{1}{\beta}\E\bracket{\norm{\bar\theta^{(t)}-\theta^\star}^2}+\frac{1}{\beta}\bigo{\bar\tau^2L^4\alpha^5\sigma^2+\frac{\bar\tau^2L^2\alpha^3\sigma^2}{MH(1-\rho)}}+\frac{2\alpha}{\beta}\E\bracket{\norm{\bar\theta^{(t)}-\theta^\star}^2}\\
    &\quad +\bigo{\beta L^4\alpha^3}\paren{\E\bracket{\norm{\bar\theta^{(t)}-\theta^\star}^2}+\sigma^2}\\
    &\leq\paren{\frac{2\alpha}{\beta}+\bigo{\frac{\bar\tau^2L^2\alpha^3}{\beta}+\beta L^4\alpha^3}}\E\bracket{\norm{\bar\theta^{(t)}-\theta^\star}^2}+\bigo{\frac{\bar\tau^2L^4\sigma^2\alpha^5}{\beta}+\frac{\bar\tau^2L^2\sigma^2\alpha^3}{MH\beta(1-\rho)}+\beta L^4 \sigma^2\alpha^3},
\end{aligned}
\end{equation}
where we use Assumption~\ref{ass::smoothness} in $(a)$, Corollary~\ref{corr:diff} and the bound from Lemma~\ref{lem::diff_tau} in $(b)$. We can achieve identical bound for $E_3$ with the same reasoning. With a proper choice of the parameter $\beta$ (which will be made later), the bounds for $E_2$ and $E_3$ also comply with the requirement.

For the term $E_1$, we further decompose it into three terms:
\begin{equation}
\begin{aligned}
    E_1 &= \overbrace{\E\bracket{\iprod{\bar\theta^{(t-\bar\tau)}-\theta^\star}{\frac{2\alpha}{MH}\sum_{i=1}^M\sum_{\ell=0}^{H-1} \paren{G_{i}(\bar\theta^{(t-\bar\tau)},o_{i,\ell}^{(t)})-\bar G_{i}(\bar\theta^{(t-\bar\tau)})}}}}^{F_1}\\
    &\ +\overbrace{\E\bracket{\iprod{\bar\theta^{(t-\bar\tau)}-\theta^\star}{\frac{2\alpha}{MH}\sum_{i=1}^M\sum_{\ell=0}^{H-1} \paren{G_{i}(\bar\theta^{(t)},o_{i,\ell}^{(t)})- G_{i}(\bar\theta^{(t-\bar\tau)},o_{i,\ell}^{(t)})}}}}^{F_2}\\
    &\ +\overbrace{\E\bracket{\iprod{\bar\theta^{(t-\bar\tau)}-\theta^\star}{\frac{2\alpha}{MH}\sum_{i=1}^M\sum_{\ell=0}^{H-1} \paren{\bar G_{i}(\bar\theta^{(t-\bar\tau)})-\bar G_{i}(\bar\theta^{(t)})}}}}^{F_3}.
\end{aligned}
\end{equation}

For the term $F_2$, we have
\begin{equation}
\begin{aligned}
    F_2&\leq \E\bracket{2\alpha L\norm{\bar\theta^{(t-\bar\tau)}-\theta^\star}\norm{\bar\theta^{(t-\bar\tau)}-\bar\theta^{(t)}}}\\
    &\leq \alpha L\E\bracket{\bar\tau L\alpha\norm{\bar\theta^{(t-\bar\tau)}-\theta^\star}^2+\frac{1}{ \bar\tau L\alpha}\norm{\bar\theta^{(t-\bar\tau)}-\bar\theta^{(t)}}^2}\\
    &\leq\alpha L\bigo{\E\bracket{ \bar\tau L\alpha\norm{\bar\theta^{(t-\bar\tau)}-\bar\theta^{(t)}}^2+\bar\tau L\alpha\norm{\bar\theta^{(t)}-\theta^\star}^2+\frac{1}{\bar\tau L\alpha}\norm{\bar\theta^{(t-\bar\tau)}-\bar\theta^{(t)}}^2}}\\
    & \leq\alpha L\bigo{\E\bracket{ \bar\tau L\alpha\norm{\bar\theta^{(t)}-\theta^\star}^2+\frac{1}{\bar\tau L\alpha}\norm{\bar\theta^{(t-\bar\tau)}-\bar\theta^{(t)}}^2}}\\
    &\leq \bigo{\bar\tau L^2\alpha^2}\paren{\E\bracket{\norm{\bar\theta^{(t)}-\theta^\star}^2}}+\bigo{\bar\tau L^4\alpha^4\sigma^2+\frac{\bar\tau L^2\alpha^2\sigma^2}{MH(1-\rho)}},
\end{aligned}
\end{equation}
where we use Assumption 1, the fact that $\alpha\bar\tau L\leq 1$, and the bound from Lemma~\ref{lem::vr_markov}. Similarly, we can obtain the same bound for $F_3$. The bounds for $F_2,F_3$ satisfy the requirement as well.

Finally, for the term $F_1$, we have
\begin{equation}
\begin{aligned}
    F_1 &= \E\bracket{\E\bracket{\iprod{\bar\theta^{(t-\bar\tau)}-\theta^\star}{\frac{2\alpha}{MH}\sum_{i=1}^M\sum_{\ell=0}^{H-1} \paren{G_{i}(\bar\theta^{(t-\bar\tau)},o_{i,\ell}^{(t)})-\bar G_{i}(\bar\theta^{(t-\bar\tau)})}}\Bigg|\cF_{-1}^{(t-\bar\tau)}}}\\
    &\overset{(a)}{=}\E\bracket{\iprod{\bar\theta^{(t-\bar\tau)}-\theta^\star}{\frac{2\alpha}{MH}\sum_{i=1}^M\sum_{\ell=0}^{H-1} \paren{\E\bracket{G_{i}(\bar\theta^{(t-\bar\tau)},o_{i,\ell}^{(t)})\Bigg|\cF_{-1}^{(t-\bar\tau)}}-\bar G_{i}(\bar\theta^{(t-\bar\tau)})}}}\\
    &\leq \E\bracket{\frac{2\alpha}{MH}\sum_{i=1}^M\sum_{\ell=0}^{H-1} \norm{\bar\theta^{(t-\bar\tau)}-\theta^\star}\norm{\E\bracket{G_{i}(\bar\theta^{(t-\bar\tau)},o_{i,\ell}^{(t)})\Bigg|\cF_{-1}^{(t-\bar\tau)}}-\bar G_{i}(\bar\theta^{(t-\bar\tau)})}}\\
    &\overset{(b)}{\leq} \bigo{L\rho^{\bar\tau}\alpha}\E\bracket{\norm{\bar\theta^{(t-\bar\tau)}-\theta^\star}\paren{\norm{\bar\theta^{(t-\bar\tau)}-\theta^\star}+\sigma}}\\
    &=\bigo{L\rho^{\bar\tau}\alpha}\E\bracket{\norm{\bar\theta^{(t-\bar\tau)}-\theta^\star}^2}+\bigo{L\rho^{\bar\tau}\alpha\sigma}\E\bracket{\norm{\bar\theta^{(t-\bar\tau)}-\theta^\star}}\\
    &\leq \bigo{L\rho^{\bar\tau}\alpha}\E\bracket{\frac{1}{\alpha}\norm{\bar\theta^{(t-\bar\tau)}-\theta^\star}^2+2\sigma\norm{\bar\theta^{(t-\bar\tau)}-\theta^\star}+\alpha\sigma^2}\\
    &\overset{(c)}{\leq} \bigo{L\alpha^3}\E\bracket{\paren{\frac{1}{\sqrt{\alpha}}\norm{\bar\theta^{(t-\bar\tau)}-\theta^\star}+\sqrt{\alpha}\sigma}^2}\\
    &\leq \bigo{L\alpha^3}\E\bracket{\frac{1}{\alpha}\norm{\bar\theta^{(t-\bar\tau)}-\theta^\star}^2+\alpha\sigma^2}\\
    &=\bigo{L\alpha^2\E\bracket{\norm{\bar\theta^{(t-\bar\tau)}-\theta^\star}^2}+L\alpha^4\sigma^2}\\
    &\leq \bigo{L\alpha^2\E\bracket{\norm{\bar\theta^{(t-\bar\tau)}-\theta^{(t)}}^2}+L\alpha^2\E\bracket{\norm{\bar\theta^{(t)}-\theta^\star}^2}+L\alpha^4\sigma^2}\\
    &\overset{(d)}{\leq} \bigo{\bar\tau^2L^3\alpha^4\E\bracket{\norm{\bar\theta^{(t)}-\theta^\star}^2}+\bar\tau^2L^5\alpha^6\sigma^2+\frac{\bar\tau^2L^3\alpha^4\sigma^2}{MH(1-\rho)}+L\alpha^2\E\bracket{\norm{\bar\theta^{(t)}-\theta^\star}^2}+L\alpha^4\sigma^2}\\
    &\overset{(e)}{\leq} \bigo{L\alpha^2}\E\bracket{\norm{\bar\theta^{(t)}-\theta^\star}^2}+\bigo{L\alpha^4\sigma^2+\frac{\bar\tau^2L^3\alpha^4\sigma^2}{MH(1-\rho)}}.
\end{aligned}
\end{equation}
Here, in $(a)$ we used the fact that $\bar\theta^{(t-\bar\tau)}$ is deterministic conditioned on $\cF_{-1}^{(t-\bar\tau)}$, $(b)$ is a result of Corollary~\ref{corr::markov} and Assumption~\ref{ass::independence}, $(c)$ used the fact that $\rho^{\bar\tau}\leq\alpha^2$, $(d)$ used previous bounds for $\E\bracket{\norm{\bar\theta^{(t)}-\bar\theta^{(t-\bar\tau)}}^2}$, and $(e)$ used the fact that $\alpha \leq 1/(\bar\tau L^2)$, and hence $\bar\tau^2L^5\alpha^6\leq L \alpha^4$.

We now plug in the bounds recursively to arrive at the final bound for the Markovian bias term. First, plugging in the bounds for $E_1$ with $E_1=F_1+F_2+F_3$ yields
\begin{equation}
\begin{aligned}
    E_1&\leq \bigo{L \alpha^2}\E\bracket{\norm{\bar\theta^{(t)}-\theta^\star}^2}+\bigo{L\alpha^4\sigma^2}+\bigo{\frac{\bar\tau^2L^3\alpha^4\sigma^2}{MH(1-\rho)}}\\
    &\quad+\bigo{\bar\tau L^2\alpha^2}\paren{\E\bracket{\norm{\bar\theta^{(t)}-\theta^\star}^2}}+\bigo{\bar\tau L^4\alpha^4\sigma^2+\frac{\bar\tau L^2\alpha^2\sigma^2}{MH(1-\rho)}}\\
    &\leq \bigo{\bar\tau L^2\alpha^2}\E\bracket{\norm{\bar\theta^{(t)}-\theta^\star}^2}+\bigo{\bar\tau L^4\sigma^2\alpha^4}+\bigo{\frac{\bar\tau L^2\alpha^2\sigma^2}{MH(1-\rho)}},
\end{aligned}
\end{equation}
where we use the fact that $L\geq 1$, $\bar\tau\geq 1$ and $\alpha^2\leq 1/(\bar\tau L)$.

Second, with $D_2=E_1+E_2+E_3$, we obtain
\begin{equation}
\begin{aligned}
    D_2 &\leq \bigo{\bar\tau L^2\alpha^2}\E\bracket{\norm{\bar\theta^{(t)}-\theta^\star}^2}+\bigo{\bar\tau L^4\sigma^2\alpha^4}+\bigo{\frac{\bar\tau L^2\alpha^2\sigma^2}{MH(1-\rho)}}\\
    &\quad +\paren{\frac{4\alpha}{\beta}+\bigo{\frac{\bar\tau^2L^2\alpha^3}{\beta}+\beta L^4\alpha^3}}\E\bracket{\norm{\bar\theta^{(t)}-\theta^\star}^2}+\bigo{\frac{\bar\tau^2L^4\sigma^2\alpha^5}{\beta}+\frac{\bar\tau^2L^2\sigma^2\alpha^3}{MH\beta(1-\rho)}+\beta L^4 \sigma^2\alpha^3}\\
    &\leq \paren{\frac{4\alpha}{\beta}+\bigo{\bar\tau L^2\alpha^2+\frac{\bar\tau^2L^2\alpha^3}{\beta}+\beta L^4\alpha^3}}\E\bracket{\norm{\bar\theta^{(t)}-\theta^\star}^2}\\
    &\quad +\bigo{\frac{\bar\tau^2L^4\sigma^2\alpha^5}{\beta}+\frac{\bar\tau^2L^2\sigma^2\alpha^3}{MH\beta(1-\rho)}+\beta L^4 \sigma^2\alpha^3+\bar\tau L^4\sigma^2\alpha^4+\frac{\bar\tau L^2\sigma^2\alpha^2}{MH(1-\rho)}}.
\end{aligned}
\end{equation}

Finally, using $C_1=D_1+D_2$ yields
\begin{equation}
\begin{aligned}
    C_1&\leq \bigo{\bar\tau L^2\alpha^2}\E\bracket{\norm{\bar\theta^{(t)}-\theta^\star}^2}+\bigo{\bar\tau L^4\alpha^4\sigma^2+\frac{\bar\tau L^2\alpha^2\sigma^2}{MH(1-\rho)}}\\
    &\quad+\paren{\frac{4\alpha}{\beta}+\bigo{\bar\tau L^2\alpha^2+\frac{\bar\tau^2L^2\alpha^3}{\beta}+\beta L^4\alpha^3}}\E\bracket{\norm{\bar\theta^{(t)}-\theta^\star}^2}\\
    &\quad +\bigo{\frac{\bar\tau^2L^4\sigma^2\alpha^5}{\beta}+\frac{\bar\tau^2L^2\sigma^2\alpha^3}{MH\beta(1-\rho)}+\beta L^4 \sigma^2\alpha^3+\bar\tau L^4\sigma^2\alpha^4+\frac{\bar\tau L^2\sigma^2\alpha^2}{MH(1-\rho)}}\\
    &\leq \paren{\frac{4\alpha}{\beta}+\bigo{\bar\tau L^2\alpha^2+\frac{\bar\tau^2L^2\alpha^3}{\beta}+\beta L^4\alpha^3}}\E\bracket{\norm{\bar\theta^{(t)}-\theta^\star}^2}\\
    &\quad +\bigo{\frac{\bar\tau^2L^4\sigma^2\alpha^5}{\beta}+\frac{\bar\tau^2L^4\sigma^2\alpha^3}{MH\beta(1-\rho)}+\beta L^4 \sigma^2\alpha^3+\bar\tau L^4\sigma^2\alpha^4+\frac{\bar\tau L^2\sigma^2\alpha^2}{MH(1-\rho)}}.
\end{aligned}
\end{equation}

By selecting $4\alpha/\beta=\alpha\mu/2$, i.e., $\beta=8/\mu$, we obtain
\begin{equation}
\begin{aligned}
    C_1&\leq \paren{\frac{\alpha\mu}{2}+\bigo{\bar\tau L^2\alpha^2+\bar\tau^2L^2\mu\alpha^3+\frac{L^4\alpha^3}{\mu}}}\E\bracket{\norm{\bar\theta^{(t)}-\theta^\star}^2}\\
    &\quad +\bigo{\bar\tau^2L^4\mu\sigma^2\alpha^5+\frac{\bar\tau^2L^4\mu\sigma^2\alpha^3}{MH(1-\rho)}+\frac{ L^4 \sigma^2\alpha^3}{\mu}+\bar\tau L^4\sigma^2\alpha^4+\frac{\bar\tau L^2\sigma^2\alpha^2}{MH(1-\rho)}}\\
    &\leq \paren{\frac{\alpha\mu}{2}+\bigo{\bar\tau L^2\alpha^2+\frac{L^4\alpha^3}{\mu}}}\E\bracket{\norm{\bar\theta^{(t)}-\theta^\star}^2} +\bigo{\frac{  L^4 \sigma^2\alpha^3}{\mu}+\frac{\bar\tau L^2\sigma^2\alpha^2}{MH(1-\rho)}},
\end{aligned}
\end{equation}
where we use the fact that $\mu\leq 1$ and $\alpha\leq 1/\bar\tau$.
\end{proof}

\begin{lemma}\label{theorem::markov}
Suppose all the conditions in Lemma~\ref{lem::diff_tau} hold. Then the following holds for \texttt{FedHSA} for any {$t\geq 2\bar\tau$:}
\begin{equation}
\begin{aligned}
    \E\bracket{\norm{\bar\theta^{(t+1)}-\theta^\star}^2}\leq&\paren{1-\frac{\alpha\mu}{2}+\bigo{\bar\tau L^2\alpha^2}}\E\bracket{\norm{\bar\theta^{(t)}-\theta^\star}^2}+\bigo{\frac{\bar\tau \alpha^2L^2}{MH(1-\rho)}+\frac{\alpha^3L^4}{\mu}}\sigma^2.
\end{aligned}
\end{equation}
\end{lemma}
\newpage 
\textbf{Proof of Lemma~\ref{theorem::markov}.}
\begin{proof}
Plugging Lemma~\ref{lem::markov_bias} and Lemma~\ref{lem::drift} into Lemma~\ref{lem::onestep_markov}, we obtain
\begin{equation}
\begin{aligned}
    \E\bracket{\norm{\bar\theta^{(t+1)}-\theta^\star}^2}&\leq\paren{1-\alpha\mu+\bigo{L^2\alpha^2}+\bigo{\frac{L^4}{\mu}\alpha^3}}\E\bracket{\norm{\bar\theta^{(t)}-\theta^\star}^2}+\bigo{\frac{L^4\sigma^2\alpha^3}{\mu}}+\bigo{\frac{L^2\alpha^2\sigma^2}{MH(1-\rho)}}\\
    &\, +\paren{\frac{\alpha\mu}{2}+\bigo{\bar\tau L^2\alpha^2+\frac{L^4\alpha^3}{\mu}}}\E\bracket{\norm{\bar\theta^{(t)}-\theta^\star}^2} +\bigo{\frac{ L^4 \sigma^2\alpha^3}{\mu}+\frac{\bar\tau L^2\sigma^2\alpha^2}{MH(1-\rho)}}+\bigo{L^2\alpha^6\sigma^2}\\
    &\leq \paren{1-\frac{\alpha\mu}{2}+\bigo{\bar\tau 
 L^2\alpha^2+\frac{L^4}{\mu}\alpha^3}}\E\bracket{\norm{\bar\theta^{(t)}-\theta^\star}^2}+\bigo{\frac{ L^4\sigma^2\alpha^3}{\mu}+\frac{\bar\tau L^2\alpha^2\sigma^2}{MH(1-\rho)}}\\
    &\leq \paren{1-\frac{\alpha\mu}{2}+\bigo{\bar\tau L^2\alpha^2}}\E\bracket{\norm{\bar\theta^{(t)}-\theta^\star}^2}+\bigo{\frac{L^4\sigma^2\alpha^3}{\mu}+\frac{\bar\tau L^2\alpha^2\sigma^2}{MH(1-\rho)}},
\end{aligned}
\end{equation}
where we used the fact that $\alpha\leq \bar\tau\mu/L^2$.
\end{proof}

\textbf{Proof of Theorem~\ref{theorem::markov_mse}}.
\begin{proof}
Applying Lemma~\ref{theorem::markov} recursively, we obtain
\begin{equation}
\begin{aligned}
    \E\bracket{\norm{\bar\theta^{(T)}-\theta^\star}^2}&\leq \paren{1-\frac{\alpha\mu}{4}}^{T-2\bar\tau}\E\bracket{\norm{\bar\theta^{(2\bar\tau)}-\theta^\star}^2}+\bigo{\frac{\bar\tau \alpha^2L^2}{MH(1-\rho)}+\frac{\alpha^3L^4}{\mu}}\sigma^2\sum_{t=0}^{T-2\bar\tau-1}\paren{1-\frac{\alpha\mu}{4}}^{t}\\
    &\leq \paren{1-\frac{\alpha\mu}{4}}^{T-2\bar\tau}\E\bracket{\norm{\bar\theta^{(2\bar\tau)}-\theta^\star}^2}+\bigo{\frac{\bar\tau \alpha L^2}{\mu MH(1-\rho)}+\frac{\alpha^2L^4}{\mu^2}}\sigma^2,\label{eqn::thm4}
\end{aligned}
\end{equation}
where the first inequality holds because we select $\alpha\leq\mu/(4\bar\tau L^2C')$, where $C'$ is greater than or equal to the dominant constant in $\bigo{\bar\tau L^2\alpha^2}$ such that $\bigo{\bar\tau L^2\alpha^2}\leq \alpha\mu/4$.

We proceed to bound the term $\E\bracket{\norm{\bar\theta^{(2\bar\tau)}-\theta^\star}^2}$. From the update rule of \texttt{FedHSA} in Eq.~\eqref{eqn::FedHSA_update} and Eq.~\eqref{eqn::simple_average}, we obtain
\begin{equation}
\begin{aligned}
    \bar\theta^{(t+1)}-\theta^\star=\bar\theta^{(t)}-\theta^\star+\frac{\alpha}{MH}\sum_{i=1}^M\sum_{\ell=0}^{H-1}G_i(\theta_{i,\ell}^{(t)}),
\end{aligned}
\end{equation}
where we subtracted $\theta^\star$ on both sides. Taking norm on both sides and applying triangle inequality yields
\begin{equation}
\begin{aligned}
    \norm{\bar\theta^{(t+1)}-\theta^\star}&\leq\norm{\bar\theta^{(t)}-\theta^\star}+\norm{\frac{\alpha}{MH}\sum_{i=1}^M\sum_{\ell=0}^{H-1}G_i(\theta_{i,\ell}^{(t)})}\\
    &\leq \norm{\bar\theta^{(t)}-\theta^\star}+\frac{\alpha}{MH}\sum_{i=1}^M\sum_{\ell=0}^{H-1}\norm{G_i(\theta_{i,\ell}^{(t)})}\\
    &\overset{(a)}{\leq} \norm{\bar\theta^{(t)}-\theta^\star}+\frac{\alpha L}{MH}\sum_{i=1}^M\sum_{\ell=0}^{H-1}\paren{\norm{\theta_{i,\ell}^{(t)}}+\sigma}\\
    &= \norm{\bar\theta^{(t)}-\theta^\star}+\frac{\alpha L}{MH}\sum_{i=1}^M\sum_{\ell=0}^{H-1}\paren{\norm{\theta_{i,\ell}^{(t)}-\bar\theta^{(t)}+\bar\theta^{(t)}-\theta^\star+\theta^\star}+\sigma}\\
    &\leq \norm{\bar\theta^{(t)}-\theta^\star}+\frac{\alpha L}{MH}\sum_{i=1}^M\sum_{\ell=0}^{H-1}\paren{\norm{\theta_{i,\ell}^{(t)}-\bar\theta^{(t)}}+\norm{\bar\theta^{(t)}-\theta^\star}+\norm{\theta^\star}+\sigma}\\
    &\overset{(b)}{\leq} \norm{\bar\theta^{(t)}-\theta^\star}+\frac{\alpha L}{MH}\sum_{i=1}^M\sum_{\ell=0}^{H-1}\paren{L\alpha\norm{\bar\theta^{(t)}-\theta^\star}+L\alpha\sigma+\norm{\bar\theta^{(t)}-\theta^\star}+2\sigma}\\
    &\leq \norm{\bar\theta^{(t)}-\theta^\star}+\alpha L\paren{2\norm{\bar\theta^{(t)}-\theta^\star}+3\sigma}\\
    &=\paren{1+2\alpha L}\norm{\bar\theta^{(t)}-\theta^\star}+3\alpha L\sigma,\label{eqn::initial}
\end{aligned}
\end{equation}
where $(a)$ holds due to Assumption~\ref{ass::smoothness}, and $(b)$ is a result of Lemma~\ref{lem::drift}. Therefore, applying Eq.~\eqref{eqn::initial} recursively yields
\begin{equation}
\begin{aligned}
    \norm{\bar\theta^{(2\bar\tau)}-\theta^\star}&\leq \paren{1+2\alpha L}^{2\bar\tau}\norm{\bar\theta^{(0)}-\theta^\star}+3\alpha L\sigma \sum_{t=0}^{2\bar\tau-1}\paren{1+2\alpha L}^{t}\\
    &\leq \paren{1+\frac{1}{2\bar\tau}}^{2\bar\tau}\norm{\bar\theta^{(0)}-\theta^\star}+3\alpha L\sigma\sum_{t=0}^{2\bar\tau-1}\paren{1+\frac{1}{2\bar\tau}}^{2\bar\tau}\\
    &\leq e\norm{\bar\theta^{(0)}-\theta^\star}+2e\sigma\\
    &=\bigo{\norm{\bar\theta^{(0)}-\theta^\star}+\sigma},\label{eqn::thm4_last}
\end{aligned}
\end{equation}
where we selected $\alpha\leq 1/(4\bar\tau L)$, and used the fact that $(1+1/x)^x\leq e, \forall x>0$. Plugging Eq.~\eqref{eqn::thm4_last} into Eq.~\eqref{eqn::thm4} yields
\begin{equation}
\begin{aligned}
    \E\bracket{\norm{\bar\theta^{(T)}-\theta^\star}^2}&\leq \paren{1-\frac{\alpha\mu}{4}}^{T-2\bar\tau}\bigo{\norm{\bar\theta^{(0)}-\theta^\star}^2+\sigma}+\bigo{\frac{\bar\tau \alpha L^2}{\mu MH(1-\rho)}+\frac{\alpha^2L^4}{\mu^2}}\sigma^2\\
    &\leq \exp\paren{-\frac{\mu}{4}\alpha (T-2\bar\tau)}\bigo{\norm{\bar\theta^{(0)}-\theta^\star}^2+\sigma}+\bigo{\frac{\bar\tau \alpha L^2}{\mu MH(1-\rho)}+\frac{\alpha^2L^4}{\mu^2}}\sigma^2\\
    &= \exp\paren{-\frac{\mu}{4}\alpha T}\exp{\paren{\frac{\mu\bar\tau\alpha}{2}}}\bigo{\norm{\bar\theta^{(0)}-\theta^\star}^2+\sigma}+\bigo{\frac{\bar\tau \alpha L^2}{\mu MH(1-\rho)}+\frac{\alpha^2L^4}{\mu^2}}\sigma^2\\
    &\leq \exp\paren{-\frac{\mu}{4}\alpha T}\bigo{\norm{\bar\theta^{(0)}-\theta^\star}^2+\sigma}+\bigo{\frac{\bar\tau \alpha L^2}{\mu MH(1-\rho)}+\frac{\alpha^2L^4}{\mu^2}}\sigma^2
\end{aligned}
\end{equation}
where we used the fact that $1-x\leq \exp\paren{-x}$, $\alpha\leq 1/\bar\tau$, and $\mu\leq 1$ such that $\exp(\mu\tau\alpha/2)\leq\bigo{1}$.
\end{proof}
\end{document}